\documentclass{article}
\usepackage{iclr2027_conference,times}
\iclrfinalcopy
\usepackage{microtype}
\usepackage{amsmath,amssymb,amsthm,mathtools}
\usepackage{booktabs,multirow,array,tabularx,longtable}
\usepackage{graphicx}
\usepackage{placeins}
\usepackage{needspace}
\usepackage{algorithm}
\usepackage[noend]{algpseudocode}
\usepackage{enumitem}
\usepackage{xcolor}
\usepackage{tikz}
\usetikzlibrary{arrows.meta,positioning,fit,shapes.geometric}
\usepackage{url}
\usepackage{hyperref}
\usepackage[capitalize,noabbrev]{cleveref}
\hypersetup{colorlinks=true,allcolors=black,pdfauthor={Aravind Reddy; Deeraj S K; Mohd Wasif Raza; Saurav Prakash},pdftitle={Machine Unlearning for Federated Socially Fair k-Means Clustering}}
\makeatletter
\providecommand{\theHALG@line}{}
\renewcommand{\theHALG@line}{\thealgorithm.\arabic{ALG@line}}
\makeatother
\newcommand{\R}{\mathbb{R}}
\newcommand{\E}{\mathbb{E}}
\newcommand{\dist}{\operatorname{dist}}
\newcommand{\Train}{\mathsf{Train}}
\newcommand{\Unlearn}{\mathsf{Unlearn}}
\newcommand{\SG}{\textnormal{\scshape SplitGroup-2k}}

\newcommand{\GFU}{\textnormal{\scshape GuardedFairUpdate}}
\newcommand{\Merged}{\textnormal{\scshape Merged-1k}}

\newtheorem{theorem}{Theorem}
\newtheorem{lemma}[theorem]{Lemma}
\newtheorem{proposition}[theorem]{Proposition}

\theoremstyle{definition}

\theoremstyle{remark}

\title{Machine Unlearning for Federated\\Socially Fair $k$-Means Clustering}
\author{\hspace*{-\tabcolsep}\begin{minipage}[t]{\textwidth}
\centering
\textbf{Aravind Reddy}$^{1,*}$ \quad
\textbf{Deeraj S K}$^{2,*}$ \quad
\textbf{Mohd Wasif Raza}$^{3,}$\thanks{Equal contribution.} \quad
\textbf{Saurav Prakash}$^{1,3}$\\[4pt]
\normalfont\small
$^1$Centre for Responsible AI, Wadhwani School of Data Science \& AI, Indian Institute of Technology Madras\\[2pt]
$^2$Department of Artificial Intelligence, Sardar Vallabhbhai National Institute of Technology, Surat, India\\[2pt]
$^3$Department of Electrical Engineering, Indian Institute of Technology Madras\\[4pt]
{\ttfamily aravind@cerai.in \quad u23ai050@coed.svnit.ac.in}\\[1pt]
{\ttfamily ee25s074@smail.iitm.ac.in \quad saurav@ee.iitm.ac.in}
\end{minipage}\hspace*{-\tabcolsep}}
\begin{document}
\maketitle
\lhead{Preprint}
\begin{abstract}
Federated clustering can support applications such as patient stratification across hospitals and customer segmentation across distributed financial datasets. Responsible AI calls for fair representation of demographic groups and support for data deletion, motivated in part by rights such as the GDPR's right to erasure. These goals interact: deleting records can change the balance between groups and alter the clustering of the remaining population. We study exact machine unlearning for federated socially fair $k$-means, whose objective is to minimize the largest group-average squared distance to the nearest center. Our approach keeps separate local summaries for each group, so summaries at unaffected institutions remain valid as global fairness weights change. For two groups, we prove an initialization approximation guarantee independent of client group proportions under ideal sampling, and we show that same-seed replay reproduces retraining exactly. Reuse pays off most when a whole client withdraws: across three datasets, replay is $35.8\times$ faster than retraining for the initialization-only model and remains $2.4\times$ and $1.7\times$ faster with one and two fair-refinement rounds, which improve worst-group cost. For individual record deletions, gains are modest (about $1.06\times$), and certifying cached assignments can cost more than it saves. These results map where exact unlearning can profitably reuse computation. Code and reproducibility materials are available at \url{https://github.com/arareddy/mufsc}.
\end{abstract}

\section{Introduction}
\label{sec:introduction}
Hospitals can cluster patients with similar clinical profiles and banks can segment customers with similar activity \citep{aliyev2020segmenting}. Prior works already cluster patients across hospitals without pooling their records \citep{huang2019patient,elhussein2023patient}, though neither addresses group fairness or deletion. When records cannot be pooled, federated clustering learns shared centers from local summaries. Such a system should represent demographic groups fairly: centers that fit a majority well can describe a smaller group poorly. It should also honor deletion requests, motivated in part by the GDPR's conditional right to erasure \citep{gdpr2016}. For deletion requests, we adopt the framework of exact unlearning: after a deletion, the released model should be the one that retraining on the remaining records with the same random seed would produce.

\paragraph{Groups, clusters, and weights.}
Demographic groups are fixed attributes of records, such as census SEX categories. Clusters are learned: each record joins its nearest center, so a cluster usually contains both groups. Socially fair $k$-means \citep{ghadiri2021socially,abbasi2021equitable} computes each group's average squared distance to the nearest center and minimizes the worse of the group averages. It asks whether the centers describe every group comparably well, not whether each cluster mirrors the population. Because each group's cost is an average, each record of group $g$ carries weight $1/n_g$, where $n_g$ is the group's global size.

\paragraph{Why deletion is harder under this objective.}
Deletion has global effects even in ordinary $k$-means: removing one record can move a center and so reassign records held anywhere. Socially fair $k$-means adds a structural effect. Deleting a group-$g$ record changes $n_g$, and with it the weight of every remaining group-$g$ record, including records at institutions that received no request. The data change at one site, but the normalization that determines how every site's records count changes globally. This mismatch shapes how clients should summarize their data. Suppose, hypothetically, that a hospital holds 30\% of all group-A patients but only 3\% of all group-B patients. In the global objective each group's records carry total weight one, so this hospital's group-A records carry 0.30 and its group-B records 0.03. If the hospital instead balances its two groups locally, its group-B patients receive ten times more influence, relative to its group-A patients, than the global objective gives them. For a natural merged summary, we show that the expected initialization loss can grow linearly with this ratio. Using global group sizes in local sampling avoids that mismatch but can make a deletion invalidate summaries at otherwise untouched hospitals. Changed weights can also alter the initial centers and, through them, later assignments and fair updates, so subtracting a deleted record from old cluster statistics does not in general reproduce retraining.

This raises the question we study: \emph{which parts of a federated fair-clustering computation can be kept after a deletion while the released model remains exactly the one retraining would produce, and when does keeping more of them save work?}

\paragraph{Our approach.}
We separate local summary construction from work that depends on global group sizes. Each (institution, group) pair selects up to $k$ representatives by $D^2$-sampling ($k$-means++ initialization) \citep{arthur2007kmeanspp,wei2016constant} and reports how many records each represents. These summaries never use global group sizes, so a deletion elsewhere changes only how the server weights them, not the summaries themselves. The server divides the counts by current group sizes and selects initial centers by weighted $D^2$-sampling from at most $2Ck$ representatives for $C$ institutions. After a deletion, only the touched summaries are rebuilt before the server reweights and resamples. \emph{Fair refinement} then alternates nearest-center assignment with center updates targeting the largest group-average squared distance, following Fair-Lloyd \citep{ghadiri2021socially}. Our guarded version accepts only updates that do not increase this cost (\cref{sec:fair-update}), using exact per-cluster, per-group statistics. \emph{Direct replay} reassigns every retained record each round; \emph{certified replay} first uses cached distances to test whether its nearest center is unchanged.

\paragraph{Refinement budget.} Let $T$ count fair-refinement rounds after initialization. At a matched configuration and seed, replay reproduces fresh training exactly, including its quality. When a whole client withdraws, every surviving slice is unchanged, so an initialization-only ($T=0$) model can be replayed from summaries alone. This specializes our replay theorem and follows the server-only retraining of \citet[Section~5.3]{pan2023mufc} (\cref{app:client-summary}). For $T>0$, direct replay keeps this initialization saving but still recomputes surviving assignments every round.

\paragraph{Contributions and evidence.}
\begin{itemize}[leftmargin=*,itemsep=1pt,topsep=2pt]
\item \emph{Reusable summaries with quality control.} For two groups and ideal independent sampling (execution uses deterministic pseudorandom streams), group-separated initialization has an expected $O(\log^2k)$ approximation factor plus a quantization term, independent of client group proportions (\cref{thm:warmstart}). A locally balanced merged summary loses a factor linear in the local/global mismatch on a specific construction (\cref{thm:kappa}); another construction keeps a gap at a matched anchor-slot budget, which one fair refinement round removes (\cref{prop:s2-matched-budget}).
\item \emph{Exact replay.} With preprocessing and encoding fixed and both groups nonempty, same-seed replay reproduces retraining's entire indexed center trajectory (\cref{thm:exactness}). All 810 replay comparisons in the original local tabular and public-Census suites match retraining.
\item \emph{Large savings for whole-client deletion.} Deleting one preselected client per dataset (three datasets, five seeds), replay is $35.8\times$ faster than retraining for the initialization-only model (pooled request time, inputs already in memory) and remains $2.4\times$ and $1.7\times$ faster with one and two fair-refinement rounds, which improve worst-group cost; each comparison uses the same budget (\cref{tab:client-main}, \cref{sec:experiments}).
\item \emph{Modest gains for record deletion.} Direct replay, which keeps local summaries but reassigns every retained record, beats retraining in 173/175 and 80/95 trials of our two suites, with modest aggregate fresh/direct algorithm-time ratios of 1.061 and 1.056. Certifying cached assignments has mixed timing outcomes and increases aggregate end-to-end cost in both suites. A construction shows how reuse can fail: for the sampler, cache, and fallback policy studied, one deletion can force expected duplicated assignment work proportional to dataset size under ideal sampling (\cref{prop:s2-deletion-work}).
\end{itemize}

\Cref{sec:training,sec:theory} specify and analyze training; \cref{sec:unlearning,sec:experiments} develop replay and evaluate its costs.

\section{Related Work}
\label{sec:related}
Socially fair clustering \citep{ghadiri2021socially,abbasi2021equitable} supplies our objective; the two-group fair update follows \citet{ghadiri2021socially}. A union of per-group strong coresets preserves the socially fair objective for every center query \citep{gorantla2023coresets}; our per-group anchors serve a narrower purpose, supporting initialization and reuse after deletion. One-shot aggregation of local representatives \citep{dennis2021heterogeneity} and the two-level seeding, server clustering, and secure aggregation of \citet{pan2023mufc}, who study unlearning of federated clusters, are the closest architectural precedents. We add global demographic normalization of group-separated summaries and exact replay of fair refinement on retained records. For centralized $k$-means, Q-$k$-means and DC-$k$-means already exploit stability or recompute affected tree paths \citep{ginart2019forget}. Deletion is also known to interact with fairness interventions in supervised learning \citep{oesterling2024fair}; we study group clustering cost under federated ownership.

Stability-based exact unlearning distinguishes model-output from model-and-state targets \citep{ullah2021stability,ullah2023adaptive}; ours concerns the indexed model trajectory. Center-movement bounds \citep{elkan2003triangle,hamerly2010faster,ding2015yinyang} and additive count--sum--squared-sum statistics \citep{zhang1996birch} are established; our certificate results concern a stipulated cache, its runner-up refinement, and sound aggregate replay. \Cref{app:positioning} gives extended comparisons of objectives, guarantees, and retained information.

\section{Problem Setting and Training Procedure}
\label{sec:training}

\paragraph{Data, groups, and objective.}
There are $C$ physical clients (institutions) with disjoint record collections $X=\bigcup_{c=1}^C X_c$ in $\R^d$. Records have persistent identities, so coincident feature vectors remain distinct; write $n=|X|$. Each record has a group $g(x)\in\{A,B\}$. Write $X_{c,g}=\{x\in X_c:g(x)=g\}$, $X_g=\{x\in X:g(x)=g\}$, $n_g=|X_g|>0$, and $[k]=\{1,\ldots,k\}$. For an \emph{ordered} center tuple $\mathbf C=(c_1,\ldots,c_k)$, the group costs and the socially fair objective are
\begin{equation}
\Phi_g(\mathbf C)=\frac1{n_g}\sum_{x:g(x)=g}\min_{j\in[k]}\|x-c_j\|_2^2,
\qquad
\Phi(\mathbf C)=\max_{g\in\{A,B\}}\Phi_g(\mathbf C).
\label{eq:fair-objective}
\end{equation}
The analysis also uses the balanced surrogate $G(\mathbf C)=\Phi_A(\mathbf C)+\Phi_B(\mathbf C)$, so $\Phi(\mathbf C)\le G(\mathbf C)\le2\Phi(\mathbf C)$.

\paragraph{What exact unlearning must reproduce.}
Exact unlearning is defined relative to one randomized procedure $\Train(X;s)$ with seed $s$. For replay to reproduce it exactly, the details that ordinary descriptions leave open, such as orderings, tie-breaking, empty clusters, and degenerate sampling cases, must be fixed. We call the fully specified procedure the \emph{declared training map}; \cref{app:implementation} collects these conventions. Preprocessing is fitted once and retained: features are clipped and encoded once on a fixed-point grid, $x_i=m_i2^{-b}$, and these \emph{represented data} define initialization, refinement, deletion, and reported quality. Exactness is conditional on this fixed preprocessing.

Training has two phases. Phase~I (\cref{sec:splitgroup}) builds initial centers from group-separated local summaries. Phase~II (\cref{sec:fair-update}) performs guarded fair refinement from exact statistics.

\subsection{Phase I: group-separated initialization (\SG)}
\label{sec:splitgroup}

\emph{Local summaries.} We call the group-$g$ records of client $c$ an \emph{ownership slice} (or pseudo-client) $e=(c,g)$, with $X_e=X_{c,g}$. Each nonempty slice selects $k_e=\min\{k,|X_e|\}$ representatives by \emph{$D^2$-sampling}: the first is drawn uniformly from the slice's records, and each later one with probability proportional to its squared distance to the nearest representative already selected. A slice with at most $k$ records keeps all of them, at zero representation cost. Local $D^2$-sampling performs no Lloyd or fair refinement. Each slice and the server draw from their own random-number stream, seeded by $s$ and their identity, so a slice whose records are unchanged makes the same random draws; this is what lets replay reuse its summary. Every slice point is then assigned to its nearest representative $z_{e,j}$, giving the local representative cluster $P_{e,j}$ and multiplicity $h_{e,j}=|P_{e,j}|$. For grid width $\gamma\ge0$, the client sends the quantized \emph{anchor} $q_{e,j}=\operatorname{snap}_\gamma(z_{e,j})$ and $h_{e,j}$: at most $2k$ anchors per client, hence the name \SG.

\emph{Server selection.} The server weights each anchor by
\begin{equation}
\omega_{e,j}=\frac{h_{e,j}}{n_g}.
\label{eq:anchor-weight}
\end{equation}
Only this step uses global group sizes. The server then selects $k$ initial centers by \emph{weighted $D^2$-sampling}: the first anchor is drawn with probability proportional to $\omega$, and each later anchor proportional to $\omega$ times its squared distance to the nearest selected anchor. All sampling probabilities are computed exactly. Optionally, the server then runs a fixed number $T_a$ of deterministic weighted \emph{anchor-Lloyd} iterations that never increase the weighted anchor objective.

\subsection{Phase II: exact statistics and a guarded fair update}
\label{sec:fair-update}

Phase II performs $T$ rounds of nearest-center assignment followed by a fair center update. At round $t$, let $z_t(x)\in[k]$ be the deterministically tie-broken index of the nearest center. The \emph{cluster} (Voronoi cell) of center $j$ and its group-$g$ \emph{group-cell} are
\begin{equation}
P_j^{(t)}=\{x\in X:z_t(x)=j\},\qquad
P_{g,j}^{(t)}=P_j^{(t)}\cap X_g.
\label{eq:cells}
\end{equation}
Group-cells depend on the current assignment; ownership slices depend on which client stores a record. We suppress the superscript $t$ within one round. Clients apply the classical count--sum--squared-sum representation \citep{zhang1996birch} separately to each group-cell and exactly aggregate
\begin{equation}
N_{g,j}=|P_{g,j}|,\qquad
S_{g,j}=\sum_{x\in P_{g,j}}x,\qquad
SS_{g,j}=\sum_{x\in P_{g,j}}\|x\|_2^2.
\label{eq:stats}
\end{equation}
Here $SS_{g,j}$ is a scalar squared-norm sum. These statistics determine each group's fixed-partition cost without transmitting individual records. Following \citet{ghadiri2021socially}, a two-group cluster's candidate center is a weighted average of its group means, controlled by a shared parameter $\lambda\in[0,1]$. A single-group cluster uses its occupied mean; an empty cluster keeps its old indexed center. \Cref{app:fair-update} gives the algebra and all branches.

Finite bisection over $\lambda$ need not minimize the fixed-partition worst-group cost. We therefore define $\GFU_L$ (\cref{alg:guarded-update}): evaluate the endpoints and $L$ bisection midpoints of \cref{eq:lambda-center}, choose the candidate with the smallest fixed-partition objective, and accept it only if that objective does not exceed the objective of the current centers; otherwise keep the current centers. Each candidate is scored exactly at the coordinates actually returned, after rounding. This finite guarded map is what we analyze and replay.

\section{Approximation and Refinement Theory}
\label{sec:theory}

Under ideal sampling, we prove an initialization approximation guarantee independent of client group proportions and a proportion-dependent loss for the specified merged alternative; guarded refinement never increases the objective. Let
\begin{equation}
\alpha_k=5(\ln k+2),
\label{eq:mrs-alpha}
\end{equation}
the expected approximation factor for exact-$k$ $k$-means++ proved by \citet{makarychev2020improved}, improving the classical $8(\ln k+2)$ analysis \citep{arthur2007kmeanspp}. Let $G^*=\min_{|\mathbf C|=k}G(\mathbf C)$ and $\Phi^*=\min_{|\mathbf C|=k}\Phi(\mathbf C)$. The guarantee extends to the server's weighted $D^2$-sampling with rational weights, still selecting exactly $k$ centers (\cref{app:mrs-weighted}).

\Cref{thm:warmstart} combines two errors: replacing records by local representatives, and weighted $D^2$-sampling on them. Group separation bounds the first without reference to client proportions.

\begin{theorem}[Split-group warm start]
\label{thm:warmstart}
Assume $n_A,n_B>0$, ideal independent random bits for every local slice and the server, exact rational categorical sampling, deterministic degeneracy rules, and
$\|q_{e,j}-z_{e,j}\|_2^2\le d\gamma^2/4$. Then, at $\gamma=0$, \SG{} returns $\mathbf C^{(0)}$ satisfying
\begin{equation}
\E\Phi(\mathbf C^{(0)})\le2\alpha_k(\sqrt{\alpha_k}+2)^2\Phi^*.
\label{eq:warm-phi-zero}
\end{equation}
This $O(\log^2 k)$ factor is independent of $n$, $C$, group imbalance and client group proportions. The general root-cost bound and, for every $\eta>0$, its additive quantization bound are given in \cref{eq:warm-root,eq:warm-phi-additive}. Deterministic weighted anchor-Lloyd steps that do not increase the compressed anchor objective preserve these bounds.
\end{theorem}

The proof (\cref{app:warm-proof}) conditions on the local anchors and combines the two errors with $L_2$ Minkowski inequalities. The appendix's $\gamma$ constants assume the stated geometric grid bound; binary64 rounding of grid points adds a term treated in \cref{app:represented-snap} (none at $\gamma=0$).

\subsection{Why the natural merged summary loses \texorpdfstring{$\Theta(\kappa)$}{Theta(kappa)}}
\label{sec:kappa}

\citet{pan2023mufc} summarize each client with a single set of representatives before server-side clustering. A natural adaptation to our fair objective, \Merged, retains this single-summary architecture but balances demographic groups \emph{locally}: on a mixed client $c$ with local group-average costs $A_c(\mathbf C)$ and $B_c(\mathbf C)$, it runs one weighted $D^2$-sampling call for $A_c+B_c$ and sends at most $k$ merged representatives, each weighted by the sum of the global point weights $1/n_{g(x)}$ assigned to it. The objective instead weights client $c$'s groups by their \emph{global} shares
\begin{equation}
a_c=\frac{n_{c,A}}{n_A},\qquad b_c=\frac{n_{c,B}}{n_B},\qquad
\kappa=\max\left\{1,\max_{c:\,a_cb_c>0}\max\left(\frac{a_c}{b_c},\frac{b_c}{a_c}\right)\right\},
\label{eq:kappa}
\end{equation}
where $n_{c,g}=|X_{c,g}|$. The parameter $\kappa$ measures the largest local/global mismatch; in the hospital example of \cref{sec:introduction}, $a_c/b_c=10$. Pure clients contribute factor one rather than an undefined ratio.

\begin{theorem}[Tight heterogeneity dependence for \Merged]
\label{thm:kappa}
Under the assumptions of \cref{thm:warmstart}, \Merged{} satisfies the explicit upper bounds in \cref{eq:merged-root,eq:merged-zero}; at $\gamma=0$, its expected $G/G^*$ ratio is $O(\kappa\log^2 k)$.
Conversely, for every rational $\kappa\ge1$ realizable by finite group counts, there is a two-client, two-group, one-dimensional instance with $k=1$, $\gamma=0$, and $T_a=0$ on which its expected approximation ratio is at least $\kappa/2$ for $G$ and $\kappa/4$ for $\Phi$. These lower-bound constants are not asserted for $T_a>0$, for \textnormal{\scshape Merged-2k}, or for the final refined model.
\end{theorem}

The lower bound uses a nearly pure mixed client whose minority makes a tiny global contribution but receives half its local sampling probability. \SG{} avoids this mismatch because group normalization is constant within each slice. \Cref{app:kappa-proof} gives the construction and its scope.

\paragraph{Matching the local budget.}
The natural single-call locally balanced \textnormal{\scshape Merged-2k} matches \SG{} at two anchor slots per client for $k=1$. On a separate three-location construction with rational $\kappa\ge1$ and $\gamma=T_a=T=0$, under ideal categorical sampling, their expected $G/G^*$ ratios are $(\kappa+5)/3$ and $2$, and their expected $\Phi/\Phi^*$ ratios are $(\kappa+8)/6$ and $3/2$ (\cref{prop:s2-matched-budget}). The gap concerns initialization: one round of fair refinement on the full data removes it.

\paragraph{Federated refinement.} If aggregation returns exact values of \cref{eq:stats}, then federated and centralized calls to $\GFU_L$ with the same partition, statistics, candidate order, and deterministic arithmetic return the same indexed centers (\cref{prop:central-equivalence} in \cref{app:fair-update}). Refinement also never worsens $\Phi$.

\begin{proposition}[Monotonicity]
\label[proposition]{prop:monotonicity}
One nearest-assignment step followed by $\GFU_L$ cannot increase $\Phi$.
\end{proposition}

The guard controls the current-partition cost, and reassignment only reduces each group's cost; without the guard, finite bisection can worsen an exactly represented incumbent (\cref{app:line-counterexample}).

\section{Exact Batch Unlearning}
\label{sec:unlearning}

Given a checkpoint and deletion batch $R$, replay should output exactly what $\Train(X\setminus R;s)$ would.

\subsection{Exact arithmetic and checkpoint state}
\label{sec:arith}

Floating-point reduction order alone is insufficient for exact patching: a rounded aggregate minus a deleted contribution need not equal a fresh reduction of retained records (\cref{app:exactness}). We therefore accumulate counts $N$, vector sums $S$, and squared-norm sums $SS$ of the fixed-point represented data exactly, in integer arithmetic, so that subtracting a deleted record exactly undoes adding it. Fresh training and replay feed identical statistics into the same deterministic center arithmetic. Certificates, introduced next, are evaluated with outward-rounded bounds, and any point whose certificate is numerically ambiguous is reassigned by an exact comparison.

A training checkpoint stores the original computation for replay: (i) each slice's Phase-I anchors and multiplicities; (ii) global group sizes and the server's sampling state; (iii) indexed centers and exact aggregate statistics at each refinement round; and (iv) at clients, each point's cached label, nearest and second-nearest distances, and runner-up center identity.

\subsection{Assignment certificates}
\label{sec:certificates}

Replayed centers differ from cached ones, so a retained point's nearest center may change. An assignment certificate proves that a cached winner remains nearest after the centers move. We apply the movement-bound argument \citep{elkan2003triangle,hamerly2010faster} between corresponding indexed rounds of the original and replay trajectories. A failed certificate means ``not proved unchanged,'' not necessarily ``label changed.''

Let a retained point $x$ have cached unique winner $j$, distances $d_1<d_2$ to its nearest and second-nearest centers, and runner-up identity $r$. Let $\delta_m=\|c'_m-c_m\|_2$ be the exact displacement of indexed center $m$ between the cached and replay trajectories at the same round.

\begin{lemma}[All-competitor certificate]
\label[lemma]{lem:basic-cert}
The cached label remains uniquely correct if
\begin{equation}
\delta_j+\max_{m\ne j}\delta_m<d_2-d_1.
\label{eq:basic-cert}
\end{equation}
\end{lemma}

The winner is at most $d_1+\delta_j$ away, while every competitor is at least $d_2-\delta_m$ away. The maximum must range over every competitor, not only the old runner-up. We also evaluate a stronger \emph{runner-up} test using the cached identity $r$; \cref{app:certificates} gives its formula, soundness, cache-specific tightness and constant-time evaluation. Strict inequalities certify a unique winner; ties and numerical ambiguity take the exact assignment path.

\subsection{Replay algorithm and exactness}
\label{sec:exactness}

A deletion batch $R$ is \emph{valid} when both groups remain nonempty and the declared map's completion rules remain defined. Because every slice and the server draw from separate random streams, replay can reproduce fresh training's random choices exactly. This needs only reproducible streams, whereas the approximation theorem idealizes their bits as independent.

\Cref{alg:unlearn} mirrors training: untouched slices reuse their summaries, touched slices and the server initializer are recomputed, and each round reconstructs retained statistics by exact corrections. If more than half of a round's attempted certificates fail, replay abandons certification, discards that round's patches, rebuilds the statistics directly, and assigns directly thereafter, as direct replay does throughout.

\begin{algorithm}[t]
\caption{Exact batch replay: reuse untouched summaries and certify retained labels}
\label{alg:unlearn}
\begin{algorithmic}[1]
\Require Checkpoint for $X$, valid batch $R$, seed $s$
\ForAll{slices $e=(c,g)$}
  \If{$R\cap X_e=\varnothing$} reuse cached $(q_{e,j},h_{e,j})$
  \Else\ rerun $D^2$-sampling, assignment, and quantization on $X_e\setminus R$ with the slice's own random stream
  \EndIf
\EndFor
\State Recompute retained $n_g$ and weights $h/n_g$; rerun the server's weighted $D^2$-sampling and any anchor-Lloyd steps
\For{$t=0,\ldots,T-1$}
  \State Broadcast replay centers and indexed center shifts from the cached trajectory
  \State Remove each deleted record from its cached old group-cell
  \State Certify surviving cached labels by \cref{eq:runner-cert}; exactly reassign every failure
  \State Aggregate signed exact corrections and patch cached $(N,S,SS)$
  \If{abandonment criterion fires}
    \State Discard this round's patches; rebuild all retained statistics from zero
    \State Use direct assignment in all later rounds
  \EndIf
  \State Apply the same deterministic $\GFU_L$ map as fresh training
\EndFor
\State \Return complete indexed center trajectory and final centers
\end{algorithmic}
\end{algorithm}

\begin{theorem}[Seeded exact batch model-output unlearning]
\label{thm:exactness}
Assume exact additive sufficient statistics; deterministic center arithmetic, quantization, tie, empty-cluster, and zero-potential rules; separate seeded random streams per slice and server; sound certificate comparisons; and a valid batch request. For every fixed seed $s$,
\begin{equation}
\pi\!\left(\Unlearn(\Train(X;s),R;s)\right)=\pi\!\left(\Train(X\setminus R;s)\right)
\label{eq:same-seed}
\end{equation}
where $\pi$ projects to the entire indexed center trajectory and final model. For a fixed or seed-independent $R$, the projections are also equal in distribution over $s$. This does not assert fresh-independent distributional equality for an adaptive request selected from the same seed-dependent checkpoint, or equality of complete API dictionaries/caches.
\end{theorem}

The proof (\cref{app:exactness}) matches the initializer, then shows by induction over rounds that exact signed corrections turn cached statistics into fresh retained statistics; abandonment replaces certification by direct assignment, preserving the update.

\subsection{What replay costs, and when reuse fails}
\label{sec:work}

Exactness alone does not save work. The unit-cost bound for the scanning implementation (\cref{eq:work} in \cref{app:work-bound}) charges setup work $W_0$ even when $T=0$, touched-slice and server reinitialization $W_I$, and per-round survivor scans. Each failed certificate and each point in a full direct round costs an $O(kd+k\log k)$ assignment, and an abandonment rebuild reassigns failures already assigned in that round.

\paragraph{A scoped obstruction to reuse.}
Certificates help only while replayed centers stay near the cached ones. Same-seed exactness does not imply stable reuse: for the declared sampler, two categorical laws whose total variation tends to zero can disagree with probability above one half under the same random words (\cref{lem:s2-word-coupling}). Small marginal distributional change therefore need not mean similar coupled samples. \Cref{prop:s2-deletion-work} embeds this instability in a legal three-client deletion that changes every retained group-A label with constant probability. Under ideal random words, $k=2$, $\gamma=T_a=0$, and $T\ge1$, the specified sampler, indexed cache and strict-majority fallback incur $\Omega(n)$ expected round-zero certificate failures and duplicated assignments. This is a policy-specific obstruction, not a universal lower bound or a causal explanation of ACS timing (\cref{app:s2-coupling}).

\section{Experiments}
\label{sec:experiments}
Local single-machine experiments test the scoped heterogeneity effect (P1), exact replay and timing on tabular and public Census data (P2/P3), and the quality--cost tradeoff across refinement budgets and deletion units. P2 has 175 trials and P3 has 95, each comparing fresh training with three replay variants. The separate P4 quantization study is retained in \cref{app:p4-evidence}. \Cref{app:experiments} specifies protocols, evidence versions, and distinct experimental units.

\paragraph{Validation and timing.}
P2 and P3 run fresh training and three replay variants (direct, basic-certificate, and runner-up-certificate) in separate processes. All 810 replay comparisons match fresh training exactly, in every indexed center and raw statistic; small instances are also checked against an independent integer/rational oracle (\cref{app:validation}). \emph{Algorithm time} is local serial wall time of the call; \emph{end-to-end (E2E)} time adds input materialization, checkpoint loading where needed, and output writing (\cref{app:timing}). Initial checkpoint construction is timed separately.

\paragraph{P1: the scoped heterogeneity effect.}
At $k=1$ and $\gamma=T_a=T=0$, eight values of $\kappa$ and four methods each use 10,000 seed indices. At $\kappa=128$, \Merged{} has mean $G/G^*=63.5310$ (standard error $1.1028$), versus $2.2375$ for \SG{} and $1.9034$ for anchor-slot-matched \textnormal{\scshape Merged-2k}; the gap of \cref{prop:s2-matched-budget} arises on a different, three-location instance. The growing \Merged{} loss supports \cref{thm:kappa}; \cref{fig:p1-current} separates finite-PRNG observations from ideal-bit expectations.

\paragraph{P2/P3: does replay save time?}
\Cref{tab:replay-main} reports total fresh-to-replay algorithm and E2E time ratios; values above one favor replay. P2 uses 15 refinement rounds; P3 uses 5, 10, 20 or 40 rounds. The later budget studies use 0--2 rounds and are reported separately. P2 uses processed Adult, Bank, and Credit data with five deletion requests per configuration (\cref{app:p124-protocol}). P3 uses 2018 public American Community Survey (ACS) person records from the Census Bureau, 16 coded features from the Folktables 0.0.12 ACSEmployment schema, SEX groups, and state or state-qualified PUMA ownership \citep{ding2021folktables,folktablescode,census2018pums}. A fixed standardization precedes encoding (\cref{app:p3-evidence}). The 19 one-factor configurations vary $n,C,k,T,r$ with five seeds each; changing $C$ also changes composition.

\begin{table}[t]
\centering
\caption{Does replay save time? Total ratios divide summed fresh times by summed replay times over all trials; values above one favor replay. Algorithm time measures the call; E2E also includes input materialization, checkpoint loading and equal output writing. Losses count algorithm-time comparisons with fresh. Useful coverage is pooled $S/N$, the share of survivor point-rounds whose assignments certification actually saved. All 525 P2 and 285 P3 replay comparisons match exactly.}
\label{tab:replay-main}
\begin{tabular}{llrrrr}
\toprule
Suite & Method & Total algorithm & Total E2E & Algorithm losses & $S/N$ (\%)\\
\midrule
P2 & Direct & 1.061 & 1.049 & 2/175 & 0.00 \\
P2 & Basic & 1.006 & 0.996 & 68/175 & 0.76 \\
P2 & Runner-up & 1.005 & 0.995 & 67/175 & 0.76 \\
\midrule
P3 & Direct & 1.056 & 1.043 & 15/95 & 0.00 \\
P3 & Basic & 0.938 & 0.930 & 74/95 & 1.35 \\
P3 & Runner-up & 0.940 & 0.931 & 76/95 & 1.35 \\
\bottomrule
\end{tabular}
\end{table}

Direct wins most algorithm-time comparisons, while certification increases aggregate E2E cost in both suites (\cref{tab:replay-main}). P3's mean paired certificate ratios average trialwise ratios and exceed one because three successful $r=1$ trials lift the means, even though both total ratios are below one. Attempted certificate acceptance is not useful saved work: same-round fallback discards passes and repeats failed assignments. Useful saved coverage is only $0.760\%$ in P2 and $1.351\%$ in P3. \Cref{app:timing,app:p2-evidence,app:p3-evidence} retain the accounting. The prescribed failed ACS cases show geometric changes and no numerical-only certificate failures; this does not establish the sampler obstruction as their causal explanation (\cref{app:p3-mechanisms}).

\paragraph{Refinement budget and record deletion.}
Because replay is exact, it has the same quality as fresh training whenever survivors, preprocessing, configuration, seed and refinement budget match; the quality differences below come from choosing a different budget or update rule. A separate Adult/Bank/Credit study holds the original record batches and five seeds fixed at $k=10$, with $T=0,1,2$ refinement rounds. All 135 distinct replay comparisons pass. Direct's total algorithm-time ratios at those budgets are 4.12/1.74/1.42 (Adult), 1.39/1.23/1.16 (Bank), and 3.70/1.97/1.59 (Credit): refinement lowers worst-group cost while reducing relative reuse gains. Both certificate variants fall back in round zero in every positive-budget case; no certificate is checked at $T=0$. \Cref{app:round-budgets} gives the quality results and separately reports algorithm time and end-to-end time including file I/O.

\paragraph{Whole-client deletion: what survives refinement?}
We time each whole-client request from resident preparation through the public call and equal model-output writing. At matched $T=0$, total fresh/replay request-time ratios are $35.79\times$ for compact and $17.79\times$ for direct replay. On these same populations, canonical Direct retains $2.357\times$ at $T=1$ and $1.696\times$ at $T=2$, winning all 15 cases at each budget (\cref{tab:client-main}). Both certificate variants beat fresh but lose to Direct, failing every attempted certificate and rebuilding round zero in all 30 settings. All 90 distinct positive-budget replay comparisons match fresh trajectories and exact statistics.

\begin{table}[htbp]
\centering
\caption{Whole-client deletion at $k=10$: one preselected client per dataset, five seeds, on the same retained populations. Ratios divide total fresh request time by total replay request time at the same $T$ (covering resident preparation, the call and output writing); values above one favor replay. Compact $T=0$ uses a separate implementation and three repeats; canonical Direct $T=1,2$ uses four repeats. Quality is mean worst-group cost over five seeds, computed within seed before averaging.}
\label{tab:client-main}
\begin{tabular}{lrrrrr}
\toprule
 & Compact $T=0$ & Direct $T=1$ & Direct $T=2$ & $\Phi(T=0)$ & $\Phi(T=2)$ \\
\midrule
Adult & 26.96 & 2.053 & 1.544 & 117.957 & 95.998 \\
Bank & 115.53 & 2.646 & 1.828 & 45.898 & 32.634 \\
Credit & 25.03 & 2.647 & 1.856 & 20.124 & 13.009 \\
\bottomrule
\end{tabular}
\end{table}

The zero-round quality cost is material: its worst-group cost is 18.2--62.1\% above two-round training across seed cases. No application-level adequacy threshold was established. Request times exclude cold loading, setup, network, physical purge and durable state persistence. Ratios always compare the same budget; dividing them does not measure cross-budget latency (\cref{app:client-summary,app:client-refinement}).

\paragraph{Which quality improves?}
From identical fair-initialized centers on the same survivors, ordinary Lloyd and guarded fair updates isolate the effect of the update rule. At each of $T=1,2$, fair updates lower worst-group cost in all 30 dataset/seed/population cases, but raise population-average and larger-group cost in every case. At $T=2$, dataset/population mean paired worst-group reductions span 3.2--5.1\%, while population-average increases span 0.5--4.7\% (\cref{tab:lloyd-main}). Nor does fair refinement improve every group at every step: from $T=1$ to $2$, Adult's male-mask group cost rises in three record-deletion seeds and two whole-client seeds; population-average cost rises in two record-deletion seeds (\cref{app:group-quality,app:lloyd}).

\begin{table}[htbp]
\centering
\caption{What does fair refinement improve? At $T=2$, fair-minus-Lloyd percentage differences $100(c_{\rm fair}-c_{\rm Lloyd})/c_{\rm Lloyd}$ are computed within seed, then averaged over five seeds. Survivors and fair-initialized centers are matched. Negative favors fair for that cost. This compares update rules, not optimized ordinary $k$-means.}
\label{tab:lloyd-main}
\begin{tabular}{llrr}
\toprule
Deletion population & Dataset & Worst-group cost & Population-average cost\\
\midrule
Record & Adult & $-4.75$ & $+4.72$ \\
Record & Bank & $-3.15$ & $+0.55$ \\
Record & Credit & $-5.07$ & $+0.71$ \\
Client & Adult & $-3.92$ & $+4.47$ \\
Client & Bank & $-4.24$ & $+1.87$ \\
Client & Credit & $-4.93$ & $+0.53$ \\
\bottomrule
\end{tabular}
\end{table}

Bounded extensions test two clients per dataset, including previously tested client~0 (\cref{app:client-diversity}), and three sequential $T=0$ withdrawals (\cref{app:client-sequences}). Overlapping cases and correlated steps are not additional independent populations.

\FloatBarrier
\section{Limitations, Ethics, and Conclusion}
\label{sec:limitations}
Exactness is conditional on one fixed represented dataset and preprocessing state; it does not erase a record's influence on a fitted scaler or vocabulary. Refined batch outputs invalidate the point cache, so further record replay needs cache reconstruction (\cref{app:cache}). If the separate $T=0$ whole-client path's compact state (local summaries, group counts and server settings) changes only through the prescribed update, each valid withdrawal yields exactly the compact initialization state of fresh $T=0$ training on the survivors (\cref{app:client-summary}). The canonical learner and main experiments are binary; \cref{app:multigroup-prototype} gives a multi-group prototype.

Warm OS caches, nondedicated desktop hardware, and only five seeds limit timing conclusions. The three-step compact-state study measures buffered local persistence, not durable commits or recovery. Network, secure aggregation, stragglers and dropouts remain unmeasured. Federation does not imply differential privacy or cryptographic erasure: group anchors and small group-cell changes can leak information, even with an additional secure-aggregation protocol \citep{bonawitz2017secure}. Group-average squared cost does not establish individual, intersectional, representation, or downstream fairness.

Keeping each institution's group summaries independent of the global fairness weights lets untouched summaries survive a deletion while retaining, under ideal sampling, a proportion-independent initialization guarantee; exact replay returns the retrained model. Reuse savings depend on the deletion unit, refinement budget and assignment stability: in the original suites direct replay saves modestly and certification raises aggregate E2E cost, while whole-client deletion saves more at low refinement budgets, although fewer rounds raise worst-group cost in our tests. These scoped results connect reusable computation to the costs and limits of exact unlearning \citep{li2024wmdp,barez2025unlearning}.\label{main-text-end}

\section*{AI Use Statement}
We used generative AI tools throughout this work. For the theory, they helped develop models and mathematical claims, formulate and refine hypotheses, construct counterexamples, and develop, write and check proofs. For the experiments, they helped with methodology and design, implementation of the method and evaluation code, preparation and reformatting of data, construction of synthetic test instances, and interpretation of results. They also supported literature search and synthesis, drafting and editing the manuscript, formatting references, and preparing figures and tables. Reported measurements come from executed programs, and our computational checks include code-based tests and numerical comparisons. The authors take responsibility for the paper's content, claims, code and other artifacts, including those developed with AI assistance.

\section*{Reproducibility Statement}
The declared training map, its numerical conventions and the full proofs are given in \cref{sec:training,app:fair-update,app:implementation}. \Cref{app:experiments} documents the data populations, deletion requests, timing definitions, evidence versions and validation scope. The public repository (\url{https://github.com/arareddy/mufsc}) and its \href{https://github.com/arareddy/mufsc/releases/tag/v1.0.0}{versioned supplementary release (v1.0.0)} provide the source code, processed data arrays, saved reporting inputs, reconstruction recipes and the commands for the full experiment grid. Its reproduction matrix records which results were executed during packaging (test fixtures, table re-aggregation and representative isolated runs) and which come from the full experiment campaigns and the raw ACS reconstruction. Adult and Credit reproduction starts from hash-identified processed inputs. The later bounded studies have their own protocols and overlapping populations (\cref{app:round-budgets,app:client-summary,app:client-refinement,app:client-diversity,app:client-sequences,app:group-quality,app:lloyd,app:multigroup-prototype}). The separate arXiv manuscript source package rebuilds this document from the saved tables and figures without retraining.

The supplement also contains Lean~4/mathlib proofs of mathematical results: the replay and certificate results, the fair-update and sufficient-statistic results, the finite lower-bound constructions, and approximation-transfer calculations that assume the cited sampling guarantees. These proofs check mathematical models under explicit hypotheses, not the full paper, executable learner or empirical measurements. The guide \path{formal_verification/README.md} describes their scope, and the ledger \path{formal_verification/THEOREM_COVERAGE.md} lists, by paper label, the checked statements, their hypotheses and the remaining gaps.

\bibliography{references}

\begin{thebibliography}{61}
\providecommand{\natexlab}[1]{#1}
\providecommand{\url}[1]{\texttt{#1}}
\expandafter\ifx\csname urlstyle\endcsname\relax
  \providecommand{\doi}[1]{doi: #1}\else
  \providecommand{\doi}{doi: \begingroup \urlstyle{rm}\Url}\fi

\bibitem[Abbasi et~al.(2021)Abbasi, Bhaskara, and
  Venkatasubramanian]{abbasi2021equitable}
Mohsen Abbasi, Aditya Bhaskara, and Suresh Venkatasubramanian.
\newblock Fair clustering via equitable group representations.
\newblock In \emph{ACM FAccT}, 2021.
\newblock \doi{10.1145/3442188.3445913}.

\bibitem[Ahrens et~al.(2020)Ahrens, Demmel, and Nguyen]{ahrens2020summation}
Willow Ahrens, James Demmel, and Hong~Diep Nguyen.
\newblock Algorithms for efficient reproducible floating point summation.
\newblock \emph{ACM Transactions on Mathematical Software}, 46\penalty0
  (3):\penalty0 22:1--22:49, 2020.
\newblock \doi{10.1145/3389360}.

\bibitem[Aliyev et~al.(2020)Aliyev, Ahmadov, Gadirli, Mammadova, and
  Alasgarov]{aliyev2020segmenting}
Musadig Aliyev, Elvin Ahmadov, Habil Gadirli, Arzu Mammadova, and Emin
  Alasgarov.
\newblock Segmenting bank customers via {RFM} model and unsupervised machine
  learning.
\newblock \href{https://arxiv.org/abs/2008.08662}{arXiv:2008.08662}, 2020.

\bibitem[Anand et~al.(2026)Anand, Makarychev, and Shan]{anand2026socially}
Aditya Anand, Yury Makarychev, and Liren Shan.
\newblock Socially fair clustering: Parameterized approximation and local
  search.
\newblock In \emph{APPROX/RANDOM}, 2026.
\newblock \doi{10.4230/LIPIcs.APPROX/RANDOM.2026.3}.

\bibitem[Arthur and Vassilvitskii(2007)]{arthur2007kmeanspp}
David Arthur and Sergei Vassilvitskii.
\newblock {$k$}-means++: The advantages of careful seeding.
\newblock In \emph{SODA}, 2007.
\newblock URL
  \url{https://theory.stanford.edu/~sergei/papers/kMeansPP-soda.pdf}.

\bibitem[Barez et~al.(2025)Barez, Fu, Prabhu, Casper, Sanyal, Bibi, O'Gara,
  Kirk, Bucknall, Fist, Ong, Torr, Lam, Trager, Krueger, Mindermann,
  Hernandez-Orallo, Geva, and Gal]{barez2025unlearning}
Fazl Barez, Tingchen Fu, Ameya Prabhu, Stephen Casper, Amartya Sanyal, Adel
  Bibi, Aidan O'Gara, Robert Kirk, Ben Bucknall, Tim Fist, Luke Ong, Philip
  Torr, Kwok-Yan Lam, Robert Trager, David Krueger, S{\"o}ren Mindermann,
  Jos{\'e} Hernandez-Orallo, Mor Geva, and Yarin Gal.
\newblock Open problems in machine unlearning for {AI} safety.
\newblock \href{https://arxiv.org/abs/2501.04952}{arXiv:2501.04952}, 2025.

\bibitem[Becker and Kohavi(1996)]{uci1996adult}
Barry Becker and Ronny Kohavi.
\newblock {Adult}.
\newblock UCI Machine Learning Repository, dataset, 1996.
\newblock \doi{10.24432/C5XW20}.

\bibitem[Bera et~al.(2019)Bera, Chakrabarty, Flores, and
  Negahbani]{bera2019fair}
Suman~K. Bera, Deeparnab Chakrabarty, Nicolas~J. Flores, and Maryam Negahbani.
\newblock Fair algorithms for clustering.
\newblock In \emph{NeurIPS}, 2019.
\newblock URL
  \url{https://proceedings.neurips.cc/paper/2019/hash/fc192b0c0d270dbf41870a63a8c76c2f-Abstract.html}.

\bibitem[Bonawitz et~al.(2017)Bonawitz, Ivanov, Kreuter, Marcedone, McMahan,
  Patel, Ramage, Segal, and Seth]{bonawitz2017secure}
Keith Bonawitz, Vladimir Ivanov, Ben Kreuter, Antonio Marcedone, H.~Brendan
  McMahan, Sarvar Patel, Daniel Ramage, Aaron Segal, and Karn Seth.
\newblock Practical secure aggregation for privacy-preserving machine learning.
\newblock In \emph{ACM CCS}, 2017.
\newblock \doi{10.1145/3133956.3133982}.

\bibitem[Bourtoule et~al.(2021)Bourtoule, Chandrasekaran, Choquette-Choo, Jia,
  Travers, Zhang, Lie, and Papernot]{bourtoule2021sisa}
Lucas Bourtoule, Varun Chandrasekaran, Christopher~A. Choquette-Choo, Hengrui
  Jia, Adelin Travers, Baiwu Zhang, David Lie, and Nicolas Papernot.
\newblock Machine unlearning.
\newblock In \emph{IEEE Symposium on Security and Privacy}, 2021.
\newblock \doi{10.1109/SP40001.2021.00019}.

\bibitem[Chen et~al.(2025)Chen, Cheng, Amiri, Nag, Lin, Liu, Tolomei, and
  Sun]{chen2025frog}
Ziheng Chen, Jiali Cheng, Hadi Amiri, Kaushiki Nag, Lu~Lin, Sijia Liu, Gabriele
  Tolomei, and Xiangguo Sun.
\newblock {FROG}: Fair removal on graph.
\newblock In \emph{ACM CIKM}, 2025.
\newblock \doi{10.1145/3746252.3761341}.

\bibitem[Chierichetti et~al.(2017)Chierichetti, Kumar, Lattanzi, and
  Vassilvitskii]{chierichetti2017fairlets}
Flavio Chierichetti, Ravi Kumar, Silvio Lattanzi, and Sergei Vassilvitskii.
\newblock Fair clustering through fairlets.
\newblock In \emph{NeurIPS}, 2017.
\newblock URL
  \url{https://proceedings.neurips.cc/paper/2017/hash/978fce5bcc4eccc88ad48ce3914124a2-Abstract.html}.

\bibitem[Chlamt{\'a}{\v c} et~al.(2022)Chlamt{\'a}{\v c}, Makarychev, and
  Vakilian]{chlamtac2022cascaded}
Eden Chlamt{\'a}{\v c}, Yury Makarychev, and Ali Vakilian.
\newblock Approximating fair clustering with cascaded norm objectives.
\newblock In \emph{SODA}, 2022.
\newblock \doi{10.1137/1.9781611977073.104}.

\bibitem[de~Berg and Staals(2021)]{deberg2021dynamic}
Sarita de~Berg and Frank Staals.
\newblock Dynamic data structures for {$k$}-nearest neighbor queries.
\newblock In \emph{ISAAC}, 2021.
\newblock \doi{10.4230/LIPIcs.ISAAC.2021.14}.

\bibitem[Dennis et~al.(2021)Dennis, Li, and Smith]{dennis2021heterogeneity}
Don~Kurian Dennis, Tian Li, and Virginia Smith.
\newblock Heterogeneity for the win: One-shot federated clustering.
\newblock In \emph{ICML}, 2021.
\newblock URL \url{https://proceedings.mlr.press/v139/dennis21a.html}.

\bibitem[Ding et~al.(2021)Ding, Hardt, Miller, and Schmidt]{ding2021folktables}
Frances Ding, Moritz Hardt, John Miller, and Ludwig Schmidt.
\newblock Retiring {Adult}: New datasets for fair machine learning.
\newblock In \emph{NeurIPS}, 2021.
\newblock URL
  \url{https://proceedings.neurips.cc/paper/2021/hash/32e54441e6382a7fbacbbbaf3c450059-Abstract.html}.

\bibitem[Ding et~al.(2015)Ding, Zhao, Shen, Musuvathi, and
  Mytkowicz]{ding2015yinyang}
Yufei Ding, Yue Zhao, Xipeng Shen, Madanlal Musuvathi, and Todd Mytkowicz.
\newblock {Yinyang} {$k$}-means: A drop-in replacement of the classic
  {$k$}-means with consistent speedup.
\newblock In \emph{ICML}, 2015.
\newblock URL \url{https://proceedings.mlr.press/v37/ding15.html}.

\bibitem[Drake and Hamerly(2012)]{drake2012adaptive}
Jonathan Drake and Greg Hamerly.
\newblock Accelerated {$k$}-means with adaptive distance bounds.
\newblock In \emph{NIPS Workshop on Optimization for Machine Learning}, 2012.
\newblock URL \url{https://cs.baylor.edu/~hamerly/papers/opt2012_paper_13.pdf}.

\bibitem[Elhussein and G{\"u}rsoy(2023)]{elhussein2023patient}
Ahmed Elhussein and Gamze G{\"u}rsoy.
\newblock Privacy-preserving patient clustering for personalized federated
  learnings.
\newblock In \emph{Machine Learning for Healthcare}, 2023.
\newblock URL \url{https://proceedings.mlr.press/v219/elhussein23a.html}.

\bibitem[Elkan(2003)]{elkan2003triangle}
Charles Elkan.
\newblock Using the triangle inequality to accelerate {$k$}-means.
\newblock In \emph{ICML}, 2003.
\newblock URL \url{https://cseweb.ucsd.edu/~elkan/kmeansicml03.pdf}.

\bibitem[{European Parliament and Council of the European
  Union}(2016)]{gdpr2016}
{European Parliament and Council of the European Union}.
\newblock Regulation ({EU}) 2016/679 ({General Data Protection Regulation}),
  article 17.
\newblock Official Journal of the European Union, L 119, 4 May 2016, 2016.
\newblock URL \url{https://eur-lex.europa.eu/eli/reg/2016/679/oj/eng}.

\bibitem[Fraboni et~al.(2024)Fraboni, Van~Waerebeke, Scaman, Vidal, Kameni, and
  Lorenzi]{fraboni2024sifu}
Yann Fraboni, Martin Van~Waerebeke, Kevin Scaman, Richard Vidal, Laetitia
  Kameni, and Marco Lorenzi.
\newblock {SIFU}: Sequential informed federated unlearning for efficient and
  provable client unlearning in federated optimization.
\newblock In \emph{AISTATS}, 2024.
\newblock URL \url{https://proceedings.mlr.press/v238/fraboni24a.html}.

\bibitem[Ghadiri et~al.(2021)Ghadiri, Samadi, and Vempala]{ghadiri2021socially}
Mehrdad Ghadiri, Samira Samadi, and Santosh Vempala.
\newblock Socially fair {$k$}-means clustering.
\newblock In \emph{ACM FAccT}, 2021.
\newblock \doi{10.1145/3442188.3445906}.

\bibitem[Ghadiri et~al.(2022)Ghadiri, Singh, and Vempala]{ghadiri2022constant}
Mehrdad Ghadiri, Mohit Singh, and Santosh~S. Vempala.
\newblock Constant-factor approximation algorithms for socially fair
  {$k$}-clustering.
\newblock \href{https://arxiv.org/abs/2206.11210}{arXiv:2206.11210}, 2022.

\bibitem[Ginart et~al.(2019)Ginart, Guan, Valiant, and Zou]{ginart2019forget}
Antonio~A. Ginart, Melody~Y. Guan, Gregory Valiant, and James~Y. Zou.
\newblock Making {AI} forget you: Data deletion in machine learning.
\newblock In \emph{NeurIPS}, 2019.
\newblock URL
  \url{https://proceedings.neurips.cc/paper/2019/hash/cb79f8fa58b91d3af6c9c991f63962d3-Abstract.html}.

\bibitem[Gorantla et~al.(2023)Gorantla, Gowda, Deshpande, and
  Louis]{gorantla2023coresets}
Sruthi Gorantla, Kishen~N. Gowda, Amit Deshpande, and Anand Louis.
\newblock Socially fair center-based and linear subspace clustering.
\newblock In \emph{ECML PKDD}, 2023.
\newblock \doi{10.1007/978-3-031-43412-9_43}.

\bibitem[Goyal and Jaiswal(2023)]{goyal2023tight}
Dishant Goyal and Ragesh Jaiswal.
\newblock Tight {FPT} approximation for socially fair clustering.
\newblock \emph{Information Processing Letters}, 182:\penalty0 106383, 2023.
\newblock \doi{10.1016/j.ipl.2023.106383}.

\bibitem[Guo et~al.(2020)Guo, Goldstein, Hannun, and van~der
  Maaten]{guo2020certified}
Chuan Guo, Tom Goldstein, Awni Hannun, and Laurens van~der Maaten.
\newblock Certified data removal from machine learning models.
\newblock In \emph{ICML}, 2020.
\newblock URL \url{https://proceedings.mlr.press/v119/guo20c.html}.

\bibitem[Gupta et~al.(2025)Gupta, Naskar, Misra, {Hemanth K. S.}, Chakraborty,
  and Kulkarni]{gupta2025arcane}
Shrutika Gupta, Annesha Naskar, Bitan Misra, {Hemanth K. S.}, Sayan
  Chakraborty, and Aparna~Shrikant Kulkarni.
\newblock Integrating {$k$}-means++ with {ARCANE}: A scalable framework for
  exact cluster unlearning.
\newblock In \emph{International Conference on Research in Computational
  Intelligence and Communication Networks}, 2025.
\newblock \doi{10.1109/ICRCICN68210.2025.11364795}.

\bibitem[Hamerly(2010)]{hamerly2010faster}
Greg Hamerly.
\newblock Making {$k$}-means even faster.
\newblock In \emph{SIAM SDM}, 2010.
\newblock \doi{10.1137/1.9781611972801.12}.

\bibitem[Heilmann et~al.(2026)Heilmann, Corbucci, Cerrato, and
  Monreale]{heilmann2026feda}
Xenia Heilmann, Luca Corbucci, Mattia Cerrato, and Anna Monreale.
\newblock {FeDa4Fair}: Client-level federated datasets for fairness evaluation.
\newblock In \emph{ACM FAccT}, 2026.
\newblock \doi{10.1145/3805689.3812291}.
\newblock URL \url{https://arxiv.org/abs/2506.21095v5}.
\newblock Author manuscript: arXiv:2506.21095v5.

\bibitem[Huang et~al.(2019)Huang, Shea, Qian, Masurkar, Deng, and
  Liu]{huang2019patient}
Li~Huang, Andrew~L. Shea, Huining Qian, Aditya Masurkar, Hao Deng, and Dianbo
  Liu.
\newblock Patient clustering improves efficiency of federated machine learning
  to predict mortality and hospital stay time using distributed electronic
  medical records.
\newblock \emph{Journal of Biomedical Informatics}, 99:\penalty0 103291, 2019.
\newblock \doi{10.1016/j.jbi.2019.103291}.

\bibitem[Huang et~al.(2025)Huang, Wu, Bai, Zhang, Cao, Zhang, and
  Hu]{huang2025wild}
ZiHeng Huang, Di~Wu, Jun Bai, Jiale Zhang, Sicong Cao, Ji~Zhang, and Yingjie
  Hu.
\newblock Federated unlearning in the wild: Rethinking fairness and data
  discrepancy.
\newblock \href{https://arxiv.org/abs/2510.07022}{arXiv:2510.07022}, 2025.

\bibitem[Li et~al.(2024)Li, Pan, Gopal, Yue, Berrios, Gatti, Li, Dombrowski,
  Goel, Mukobi, Helm-Burger, Lababidi, Justen, Liu, Chen, Barrass, Zhang, Zhu,
  Tamirisa, Bharathi, Herbert-Voss, Breuer, Zou, Mazeika, Wang, Oswal, Lin,
  Hunt, Tienken-Harder, Shih, Talley, Guan, Steneker, Campbell, Jokubaitis,
  Basart, Fitz, Kumaraguru, Karmakar, Tupakula, Varadharajan, Shoshitaishvili,
  Ba, Esvelt, Wang, and Hendrycks]{li2024wmdp}
Nathaniel Li, Alexander Pan, Anjali Gopal, Summer Yue, Daniel Berrios, Alice
  Gatti, Justin~D. Li, Ann-Kathrin Dombrowski, Shashwat Goel, Gabriel Mukobi,
  Nathan Helm-Burger, Rassin Lababidi, Lennart Justen, Andrew~Bo Liu, Michael
  Chen, Isabelle Barrass, Oliver Zhang, Xiaoyuan Zhu, Rishub Tamirisa, Bhrugu
  Bharathi, Ariel Herbert-Voss, Cort~B. Breuer, Andy Zou, Mantas Mazeika, Zifan
  Wang, Palash Oswal, Weiran Lin, Adam~Alfred Hunt, Justin Tienken-Harder,
  Kevin~Y. Shih, Kemper Talley, John Guan, Ian Steneker, David Campbell, Brad
  Jokubaitis, Steven Basart, Stephen Fitz, Ponnurangam Kumaraguru,
  Kallol~Krishna Karmakar, Uday Tupakula, Vijay Varadharajan, Yan
  Shoshitaishvili, Jimmy Ba, Kevin~M. Esvelt, Alexandr Wang, and Dan Hendrycks.
\newblock The {WMDP} benchmark: Measuring and reducing malicious use with
  unlearning.
\newblock In \emph{ICML}, 2024.
\newblock URL \url{https://proceedings.mlr.press/v235/li24bc.html}.

\bibitem[Makarychev et~al.(2020)Makarychev, Reddy, and
  Shan]{makarychev2020improved}
Konstantin Makarychev, Aravind Reddy, and Liren Shan.
\newblock Improved guarantees for {$k$}-means++ and {$k$}-means++ parallel.
\newblock In \emph{NeurIPS}, 2020.
\newblock URL
  \url{https://proceedings.neurips.cc/paper/2020/hash/ba304f3809ed31d0ad97b5a2b5df2a39-Abstract.html}.

\bibitem[Makarychev and Vakilian(2021)]{makarychev2021socially}
Yury Makarychev and Ali Vakilian.
\newblock Approximation algorithms for socially fair clustering.
\newblock In \emph{COLT}, 2021.
\newblock URL \url{https://proceedings.mlr.press/v134/makarychev21a.html}.

\bibitem[Miller et~al.(2023)Miller, Ding, Schmidt, and Hardt]{folktablescode}
John Miller, Frances Ding, Ludwig Schmidt, and Moritz Hardt.
\newblock {Folktables}.
\newblock Software, version 0.0.12, 2023.
\newblock URL \url{https://pypi.org/project/folktables/0.0.12/}.
\newblock ACSEmployment definition in folktables/acs.py; accessed September 26,
  2026.

\bibitem[Moro et~al.(2014)Moro, Cortez, and Rita]{moro2014bank}
S\'ergio Moro, Paulo Cortez, and Paulo Rita.
\newblock A data-driven approach to predict the success of bank telemarketing.
\newblock \emph{Decision Support Systems}, 62:\penalty0 22--31, 2014.
\newblock \doi{10.1016/j.dss.2014.03.001}.

\bibitem[Muresanu et~al.(2025)Muresanu, Thudi, Zhang, and
  Papernot]{muresanu2025context}
Andrei~Ioan Muresanu, Anvith Thudi, Michael~R. Zhang, and Nicolas Papernot.
\newblock Fast exact unlearning for in-context learning data for {LLM}s.
\newblock In \emph{ICML}, 2025.
\newblock URL \url{https://proceedings.mlr.press/v267/muresanu25a.html}.

\bibitem[Neal(2015)]{neal2015exact}
Radford~M. Neal.
\newblock Fast exact summation using small and large superaccumulators.
\newblock \href{https://arxiv.org/abs/1505.05571}{arXiv:1505.05571}, 2015.

\bibitem[Newling and Fleuret(2016)]{newling2016bounds}
James Newling and Fran{\c c}ois Fleuret.
\newblock Fast {$k$}-means with accurate bounds.
\newblock In \emph{ICML}, 2016.
\newblock URL \url{https://proceedings.mlr.press/v48/newling16.html}.

\bibitem[Oesterling et~al.(2024)Oesterling, Ma, Calmon, and
  Lakkaraju]{oesterling2024fair}
Alex Oesterling, Jiaqi Ma, Flavio~P. Calmon, and Himabindu Lakkaraju.
\newblock Fair machine unlearning: Data removal while mitigating disparities.
\newblock In \emph{AISTATS}, 2024.
\newblock URL \url{https://proceedings.mlr.press/v238/oesterling24a.html}.

\bibitem[Pan et~al.(2023)Pan, Sima, Prakash, Rana, and Milenkovic]{pan2023mufc}
Chao Pan, Jin Sima, Saurav Prakash, Vishal Rana, and Olgica Milenkovic.
\newblock Machine unlearning of federated clusters.
\newblock In \emph{ICLR}, 2023.
\newblock URL \url{https://openreview.net/forum?id=VzwfoFyYDga}.

\bibitem[Qiu et~al.(2023)Qiu, Wang, Xu, Cui, and Shen]{qiu2023fedcio}
Hongyu Qiu, Yongwei Wang, Yonghui Xu, Lizhen Cui, and Zhiqi Shen.
\newblock {FedCIO}: Efficient exact federated unlearning with clustering,
  isolation, and one-shot aggregation.
\newblock In \emph{IEEE Big Data}, 2023.
\newblock \doi{10.1109/BigData59044.2023.10386788}.

\bibitem[Quan et~al.(2026)Quan, Wu, and Montana]{quan2026ridge}
Yijun Quan, Wentai Wu, and Giovanni Montana.
\newblock Exact federated continual unlearning for ridge heads on frozen
  foundation models.
\newblock \href{https://arxiv.org/abs/2603.12977v3}{arXiv:2603.12977v3}, 2026.

\bibitem[Raff(2021)]{raff2021exact}
Edward Raff.
\newblock Exact acceleration of {$k$}-means++ and {$k$}-means$\parallel$.
\newblock In \emph{IJCAI}, 2021.
\newblock \doi{10.24963/ijcai.2021/403}.

\bibitem[Rodr{\'i}guez-Corominas et~al.(2026)Rodr{\'i}guez-Corominas, Blesa,
  and Blum]{rodriguez2026fastkpp}
Guillem Rodr{\'i}guez-Corominas, Maria~Josep Blesa, and Christian Blum.
\newblock {FastKPP}: An exact accelerated {$k$}-means++ algorithm.
\newblock \emph{Software Impacts}, 28:\penalty0 100830, 2026.
\newblock \doi{10.1016/j.simpa.2026.100830}.

\bibitem[Ry{\v s}av{\'y} and Hamerly(2016)]{rysavy2016geometric}
Petr Ry{\v s}av{\'y} and Greg Hamerly.
\newblock Geometric methods to accelerate {$k$}-means algorithms.
\newblock In \emph{SIAM SDM}, 2016.
\newblock \doi{10.1137/1.9781611974348.37}.

\bibitem[Schmidt et~al.(2020)Schmidt, Schwiegelshohn, and
  Sohler]{schmidt2020faircoresets}
Melanie Schmidt, Chris Schwiegelshohn, and Christian Sohler.
\newblock Fair coresets and streaming algorithms for fair {$k$}-means.
\newblock In \emph{Workshop on Approximation and Online Algorithms (WAOA
  2019)}, 2020.
\newblock \doi{10.1007/978-3-030-39479-0_16}.

\bibitem[Tao et~al.(2024)Tao, Wang, Pan, Yu, Cheng, and Wang]{tao2024fats}
Youming Tao, Cheng-Long Wang, Miao Pan, Dongxiao Yu, Xiuzhen Cheng, and
  Di~Wang.
\newblock Communication efficient and provable federated unlearning.
\newblock \emph{Proceedings of the VLDB Endowment}, 17\penalty0 (5):\penalty0
  1119--1131, 2024.
\newblock \doi{10.14778/3641204.3641220}.

\bibitem[Ullah and Arora(2023)]{ullah2023adaptive}
Enayat Ullah and Raman Arora.
\newblock From adaptive query release to machine unlearning.
\newblock In \emph{ICML}, 2023.
\newblock URL \url{https://proceedings.mlr.press/v202/ullah23a.html}.

\bibitem[Ullah et~al.(2021)Ullah, Mai, Rao, Rossi, and
  Arora]{ullah2021stability}
Enayat Ullah, Tung Mai, Anup Rao, Ryan~A. Rossi, and Raman Arora.
\newblock Machine unlearning via algorithmic stability.
\newblock In \emph{COLT}, 2021.
\newblock URL \url{https://proceedings.mlr.press/v134/ullah21a.html}.

\bibitem[{U.S. Census Bureau}(2018)]{census2018pums}
{U.S. Census Bureau}.
\newblock 2018 {ACS} 1-year {PUMS} documentation, 2018.
\newblock URL
  \url{https://www.census.gov/programs-surveys/acs/microdata/documentation.2018.html}.
\newblock Data-year documentation; accessed September 6, 2026.

\bibitem[{U.S. Census Bureau}(2019)]{censuspuma}
{U.S. Census Bureau}.
\newblock 2018 {ACS} 1-year {PUMS} data dictionary, 2019.
\newblock URL
  \url{https://www2.census.gov/programs-surveys/acs/tech_docs/pums/data_dict/PUMS_Data_Dictionary_2018.txt}.
\newblock Dated November 14, 2019; 2018 data. PUMA, ST, SEX, and person-record
  codes; accessed September 6, 2026.

\bibitem[Wei(2016)]{wei2016constant}
Dennis Wei.
\newblock A constant-factor bi-criteria approximation guarantee for
  {$k$}-means++.
\newblock In \emph{NeurIPS}, 2016.
\newblock URL
  \url{https://proceedings.neurips.cc/paper/2016/hash/357a6fdf7642bf815a88822c447d9dc4-Abstract.html}.

\bibitem[Wen et~al.(2026)Wen, Zhang, Yang, and Ding]{wen2026fedshard}
Siyuan Wen, Meng Zhang, Yang Yang, and Ningning Ding.
\newblock {FedShard}: Federated unlearning with efficiency fairness and
  performance fairness.
\newblock In \emph{AAAI}, 2026.
\newblock \doi{10.1609/aaai.v40i32.39895}.

\bibitem[Xiong et~al.(2023)Xiong, Li, Li, and Cai]{xiong2023exactfun}
Zuobin Xiong, Wei Li, Yingshu Li, and Zhipeng Cai.
\newblock {Exact-Fun}: An exact and efficient federated unlearning approach.
\newblock In \emph{IEEE ICDM}, 2023.
\newblock \doi{10.1109/ICDM58522.2023.00188}.

\bibitem[Yan et~al.(2025)Yan, Liu, and Zhang]{yan2025omnifc}
Jie Yan, Jing Liu, and Zhong-Yuan Zhang.
\newblock {OmniFC}: Rethinking federated clustering via lossless and secure
  distance reconstruction.
\newblock In \emph{NeurIPS}, 2025.
\newblock \doi{10.52202/085713-0702}.

\bibitem[Yeh and Lien(2009)]{yeh2009credit}
I-Cheng Yeh and Che-hui Lien.
\newblock The comparisons of data mining techniques for the predictive accuracy
  of probability of default of credit card clients.
\newblock \emph{Expert Systems with Applications}, 36\penalty0 (2):\penalty0
  2473--2480, 2009.
\newblock \doi{10.1016/j.eswa.2007.12.020}.

\bibitem[Zhang et~al.(1996)Zhang, Ramakrishnan, and Livny]{zhang1996birch}
Tian Zhang, Raghu Ramakrishnan, and Miron Livny.
\newblock {BIRCH}: An efficient data clustering method for very large
  databases.
\newblock In \emph{ACM SIGMOD}, 1996.
\newblock \doi{10.1145/233269.233324}.

\bibitem[Zhu et~al.(2023)Zhu, Xu, Zeng, Wang, Sun, Yang, Yang, and
  Peng]{zhu2023f3km}
Shengkun Zhu, Quanqing Xu, Jinshan Zeng, Sheng Wang, Yuan Sun, Zhifeng Yang,
  Chuanhui Yang, and Zhiyong Peng.
\newblock {F3KM}: Federated, fair, and fast {$k$}-means.
\newblock \emph{Proceedings of the ACM on Management of Data}, 1\penalty0
  (4):\penalty0 241:1--241:25, 2023.
\newblock \doi{10.1145/3626728}.

\end{thebibliography}
\bibliographystyle{plainnat}
\appendix
\setcounter{topnumber}{4}
\setcounter{bottomnumber}{3}
\setcounter{totalnumber}{6}
\renewcommand{\topfraction}{0.95}
\renewcommand{\bottomfraction}{0.9}
\renewcommand{\textfraction}{0.05}
\renewcommand{\floatpagefraction}{0.85}
\crefalias{section}{appendix}
\crefalias{subsection}{appendix}
\crefalias{subsubsection}{appendix}
\crefname{appendix}{Appendix}{Appendices}
\Crefname{appendix}{Appendix}{Appendices}
\section{Guarded Fixed-Partition Fair Update}
\label{app:fair-update}

This appendix supports \cref{sec:fair-update}. It shows that the exact statistics determine the fixed-partition objective, specifies the guarded search $\GFU_L$, proves that federated execution matches centralized execution and that refinement is monotone, and shows why the guard is needed.

\subsection{Sufficient statistics and cluster cases}
Here $SS_{g,j}$ is a scalar \emph{squared-norm sum}, not the square of the vector sum $S_{g,j}\in\R^d$. Expanding squared distances gives
\begin{equation}
\sum_{x\in P_{g,j}}\|x-c_j\|_2^2
=SS_{g,j}-2\langle S_{g,j},c_j\rangle+N_{g,j}\|c_j\|_2^2,
\label{eq:stats-identity}
\end{equation}
so the statistics recover the fixed-assignment objective without transmitting individual records. If both $N_{A,j}$ and $N_{B,j}$ are positive, $P_j$ is a \emph{two-group cluster}. If exactly one is positive, it is a \emph{single-group cluster}; the other group-cell is empty. If both are zero, $P_j$ is an \emph{empty cluster}.

For nonempty group-cells, the statistics determine group means $\mu_{g,j}=S_{g,j}/N_{g,j}$, normalized masses $\alpha_j=N_{A,j}/n_A$ and $\beta_j=N_{B,j}/n_B$, and constants $\Delta_g=n_g^{-1}\sum_{j:N_{g,j}>0}(SS_{g,j}-\|S_{g,j}\|_2^2/N_{g,j})$. Omitting empty group-cells, the group objectives for the fixed partition are
\begin{equation}
f_A(\mathbf C)=\Delta_A+\sum_j\alpha_j\|c_j-\mu_{A,j}\|^2,
\qquad
f_B(\mathbf C)=\Delta_B+\sum_j\beta_j\|c_j-\mu_{B,j}\|^2.
\label{eq:fixed-objective}
\end{equation}
Following the two-group fair update of \citet{ghadiri2021socially}, the Lagrangian candidate for a two-group cluster $P_j$ and $\lambda\in[0,1]$ is a weighted average of the two group means,
\begin{equation}
c_j(\lambda)=
\frac{\lambda\alpha_j\mu_{A,j}+(1-\lambda)\beta_j\mu_{B,j}}
{\lambda\alpha_j+(1-\lambda)\beta_j}.
\label{eq:lambda-center}
\end{equation}
For a single-group cluster, the candidate is the occupied group's mean. An empty cluster has no data-driven update and retains its old indexed center.

Fix the indexed partition of \cref{eq:cells}. For each nonempty group-cell, put $\mu_{g,j}=S_{g,j}/N_{g,j}$. Expanding around this mean cancels the cross term because $\sum_{x\in P_{g,j}}(x-\mu_{g,j})=0$, giving
\[
\sum_{x\in P_{g,j}}\|x-c_j\|^2
=SS_{g,j}-\frac{\|S_{g,j}\|^2}{N_{g,j}}+N_{g,j}\|\mu_{g,j}-c_j\|^2.
\]
Sum over nonempty cells and divide by $n_g$ to obtain \cref{eq:fixed-objective}. Empty cells are omitted before division. Thus the statistics determine the full fixed-partition objective, including its constant term.

For a two-group cluster, projection onto the segment joining the two means weakly decreases both group terms; minimizing $\lambda f_A+(1-\lambda)f_B$ gives \cref{eq:lambda-center}. A single-group cluster moves to its occupied mean, weakly reducing that group's term without changing the other. An empty cluster contributes nothing and retains its indexed incumbent. Coincident group means need no special case: every candidate of \cref{eq:lambda-center} then equals that common mean.

On represented data, $S=S_{\rm int}/2^b$ and $SS=SS_{\rm int}/2^{2b}$. Candidate binary64 coordinates are scored as exact rationals: a guard evaluated on an unrounded candidate would not justify returning different rounded coordinates.

\subsection{Deterministic guarded search}

\begin{algorithm}[ht]
\caption{$\GFU_L$: deterministic guarded two-group update}
\label{alg:guarded-update}
\begin{algorithmic}[1]
\Require Current indexed centers $\mathbf C$, exact $(N,S,SS)$, integer $L$
\State Construct the deterministic endpoints $\mathbf C(0),\mathbf C(1)$ using \cref{eq:lambda-center} and all single-group/empty-cluster branches
\State $\mathcal V\gets\{0,1\}$; $\ell\gets0$; $u\gets1$
\For{$q=1,\ldots,L$}
  \State $\lambda\gets(\ell+u)/2$; add $\lambda$ to $\mathcal V$
  \State Represent $\mathbf C(\lambda)$ canonically; evaluate both group costs exactly on those coordinates
  \If{$f_A(\mathbf C(\lambda))>f_B(\mathbf C(\lambda))$}
    \State $\ell\gets\lambda$
  \ElsIf{$f_A(\mathbf C(\lambda))<f_B(\mathbf C(\lambda))$}
    \State $u\gets\lambda$
  \Else
    \State $u\gets\lambda$ (the declared equality branch)
  \EndIf
\EndFor
\State Choose $\widehat\lambda\in\mathcal V$ minimizing $\max\{f_A(\mathbf C(\lambda)),f_B(\mathbf C(\lambda))\}$; break ties by the smallest encoded $\lambda$
\If{$\max_g f_g(\mathbf C(\widehat\lambda))\le\max_g f_g(\mathbf C)$}
  \State \Return $\mathbf C(\widehat\lambda)$
\Else
  \State \Return $\mathbf C$
\EndIf
\end{algorithmic}
\end{algorithm}

Along the ideal rational curve, $f_A$ is nonincreasing and $f_B$ is nondecreasing in $\lambda$. Rounding may change this search geometry; \cref{alg:guarded-update} nevertheless guards the actual represented candidate, so monotonicity does not require an exact convex search.

\subsection{Centralized equivalence and monotonicity}

The main text (\cref{sec:theory}) summarizes the following statement and states \cref{prop:monotonicity}.

\begin{proposition}[Algorithm-relative centralized equivalence]
\label[proposition]{prop:central-equivalence}
If federated aggregation returns exact values of \cref{eq:stats}, then a federated call to $\GFU_L$ returns exactly the same indexed centers as a centralized implementation using the same partition, exact statistics, candidate order, and deterministic arithmetic. It is equivalent to the declared finite guarded solver, not necessarily to the exact convex optimum.
\end{proposition}

\begin{proof}[Proof of \cref{prop:central-equivalence}]
Exact addition over the disjoint clients gives the centralized $N,S,SS$. These determine every branch and candidate cost in \cref{alg:guarded-update}; canonical candidate order and ties therefore give the same indexed output. Finite $L$ need not give the convex optimum.
\end{proof}

\begin{proof}[Proof of \cref{prop:monotonicity}]
Recall that $\Phi$ is the larger group-average cost when every point uses its nearest center. The update is guarded on a fixed partition, so we first compare costs on the current partition and only then allow reassignment. Let $\sigma_t$ be the nearest-center partition for $\mathbf C^{(t)}$, so that $F_{\sigma_t}(\mathbf C^{(t)})=\Phi(\mathbf C^{(t)})$. By the guard,
\begin{equation}
F_{\sigma_t}(\mathbf C^{(t+1)})
\le F_{\sigma_t}(\mathbf C^{(t)}),
\qquad
F_{\sigma}(\mathbf C)=\max_g\frac1{n_g}\sum_{x:g(x)=g}\|x-c_{\sigma(x)}\|^2.
\end{equation}
The guard compares exactly these costs: on partition $\sigma_t$, the fixed-partition objective $\max\{f_A,f_B\}$ of \cref{eq:fixed-objective} equals $F_{\sigma_t}$. Reassigning each point to its nearest center under $\mathbf C^{(t+1)}$ weakly decreases every individual squared distance, hence each group-average cost. Therefore
$\Phi(\mathbf C^{(t+1)})\le F_{\sigma_t}(\mathbf C^{(t+1)})\le\Phi(\mathbf C^{(t)})$.
\end{proof}

\subsection{Finite line search can increase the objective}
\label{app:line-counterexample}

Take $d=k=1$, $A=\{-1,+1\}$ and $B=\{3\}$. The unique partition has $f_A(c)=1+c^2$, $f_B(c)=(3-c)^2$ and $F(c)=\max\{f_A(c),f_B(c)\}$. Their intersection is the unique minimizer $c^*=4/3$, with $F(c^*)=25/9$. On the Lagrangian curve $c(\lambda)=3(1-\lambda)$, the optimum $\lambda^*=5/9$ is not dyadic. One bisection step returns $c=3/2$, with cost $13/4>25/9$.

This can also worsen an exactly represented incumbent: $c_0=5/4$ has cost $49/16$, whereas the endpoints and midpoint $0,3,3/2$ cost $9,10,52/16$. The guard rejects the best of these. Data, incumbent and candidates are all binary-representable. Thus finite search can miss the optimum and, separately, violate unguarded monotonicity on represented inputs.

\section{Proof of the \SG{} Warm-Start Guarantee}
\label{app:warm-proof}

Under the assumptions of \cref{thm:warmstart}, the general root-cost and quantization bounds are
\begin{equation}
\boxed{
\sqrt{\E\,G(\mathbf C^{(0)})}
\le
\sqrt{\alpha_k}(\sqrt{\alpha_k}+2)\sqrt{G^*}
+(1+\sqrt{\alpha_k})\sqrt{\frac{d\gamma^2}{2}}.}
\label{eq:warm-root}
\end{equation}
At $\gamma=0$, the surrogate also satisfies
\begin{equation}
\E G(\mathbf C^{(0)})\le\alpha_k(\sqrt{\alpha_k}+2)^2G^*.
\label{eq:warm-g-zero}
\end{equation}
For every $\eta>0$, the corresponding fair-objective bound is
\begin{equation}
\E\Phi(\mathbf C^{(0)})
\le
2(1+\eta)\alpha_k(\sqrt{\alpha_k}+2)^2\Phi^*
+\frac{(1+1/\eta)(1+\sqrt{\alpha_k})^2}{2}\,d\gamma^2.
\label{eq:warm-phi-additive}
\end{equation}
These bounds are independent of $n$, $C$, group imbalance and client group proportions. For fixed $\eta$, the additive quantization term is $O(d\gamma^2\log k)$. The geometric grid assumption and represented rounding are distinguished below.

\paragraph{Objective, obstacle, and strategy.}
\Cref{thm:warmstart} compares the expected initial cost $\E G(\mathbf C^{(0)})$ with the optimum $G^*$. The server never sees the records: it runs weighted $D^2$-sampling on anchors, so the final error mixes two sources (replacing records by anchors, and sampling on the anchors), and the randomness comes in two stages (local tapes, then the server tape). Combining the errors with the squared triangle inequality $\|a+b\|^2\le2\|a\|^2+2\|b\|^2$ at every step would multiply constants. Instead we (i) bound the expected local replacement error slice by slice, where group separation makes each record's weight constant within its slice; (ii) condition on the local tapes, so that the anchors are fixed and the independent server call satisfies the $D^2$-sampling guarantee; and (iii) combine all errors on the square-root scale with Minkowski inequalities, first over the server tape and then over the local tapes.

\paragraph{Notation.}
Put $w_x=1/n_{g(x)}$, let $\mathbf C_G^*$ minimize $G$, and write $z(x)$ and $q(x)$ for each point's local representative and quantized anchor. Define
\[
D_z=\sum_xw_x\|x-z(x)\|^2,\quad
Q=\sum_xw_x\|z(x)-q(x)\|^2,\quad
D_q=\sum_xw_x\|x-q(x)\|^2.
\]
Throughout, ``weighted Minkowski'' means that for nonnegative weights $w_x$ and vectors $a_x,b_x$, $\bigl(\sum_xw_x\|a_x+b_x\|^2\bigr)^{1/2}\le\bigl(\sum_xw_x\|a_x\|^2\bigr)^{1/2}+\bigl(\sum_xw_x\|b_x\|^2\bigr)^{1/2}$; ``$L_2$ Minkowski'' means $\sqrt{\E(X+Y)^2}\le\sqrt{\E X^2}+\sqrt{\E Y^2}$ for nonnegative random variables.

\subsection{Weighted sampling and local error}
\label{app:mrs-weighted}
Under ideal independent bits, the sampler in \cref{app:pseudocode} realizes the exact categorical law. The finite seeded PCG64 execution is a different probability model. To apply \citet{makarychev2020improved} to rational anchor weights $\omega_i=h_i/n_{g(i)}$, choose an integer $M$ with $M\omega_i$ integral and replace anchor $i$ by that many coincident copies. Summing probabilities over copies gives first-draw probability $\omega_i/\sum_j\omega_j$ and subsequent probability proportional to $\omega_iD_i^2$. Every objective is multiplied by $M$, so the ordinary exact-$k$ guarantee $\alpha_k=5(\ln k+2)$ transfers unchanged. If fewer than $k$ positive-weight locations exist, their selection makes the cost zero and deterministic completion preserves it. This also applies to the weighted merged sampler.

Within a split slice, $w_x$ is constant and does not change the sampling law. A slice with at most $k$ records is saturated and has zero replacement error. Otherwise $k_e=k$, and because $z(x)$ is the nearest selected representative, the slice's part of $D_z$ is $w_x$ times the cost of the representatives that $D^2$-sampling selects from $X_e$. The $D^2$-sampling guarantee bounds its expectation by $\alpha_k$ times the optimal $k$-center cost of $X_e$, which is at most the cost of the restriction of $\mathbf C_G^*$ to $X_e$, namely $\sum_{x\in X_e}\dist(x,\mathbf C_G^*)^2$. Multiplying by the constant weight and summing slices, including empty slices with zero cost, yields $\E D_z\le\alpha_k\sum_xw_x\dist(x,\mathbf C_G^*)^2=\alpha_kG(\mathbf C_G^*)$, that is,
\begin{equation}
\E D_z\le\alpha_kG^*.
\label{eq:dz-mrs}
\end{equation}
No client/group-proportion comparison is needed. The grid error and weighted Minkowski inequality give
\begin{equation}
Q\le\sum_xw_x\frac{d\gamma^2}{4}=\frac{d\gamma^2}{2},\qquad
\sqrt{D_q}\le\sqrt{D_z}+\sqrt Q,
\label{eq:q-bound}
\end{equation}
because each group's normalized mass is one. Taking $L_2$ norms over the local tapes gives
\begin{equation}
\sqrt{\E D_q}\le\sqrt{\E D_z}+\sqrt{\E Q}
\le\sqrt{\alpha_kG^*}+\sqrt{d\gamma^2/2}.
\label{eq:edq-root}
\end{equation}

\subsection{Server composition}
Let $H(\mathbf C)=\sum_{e,j}(h_{e,j}/n_{g(e)})\dist(q_{e,j},\mathbf C)^2=\sum_xw_x\dist(q(x),\mathbf C)^2$ be the server's weighted anchor objective and $H^*=\min_{|\mathbf C|=k}H(\mathbf C)$. Assigning each point to a nearest center for its anchor gives the first inequality below; assigning each anchor using a nearest center for its point in $\mathbf C_G^*$ gives the second:
\begin{equation}
\sqrt{G(\mathbf C)}\le\sqrt{D_q}+\sqrt{H(\mathbf C)},\qquad
\sqrt{H^*}\le\sqrt{D_q}+\sqrt{G^*}.
\label{eq:g-h-minkowski}
\end{equation}
Both follow from the pointwise triangle inequality and weighted Minkowski.

\emph{Conditioning on the local tapes.} Fix the local tapes. Then $q,h,D_q$ are fixed, and the server tape is independent of them, so the weighted $D^2$-sampling guarantee of \cref{app:mrs-weighted} applies to the anchor instance: $\E_{\rm server}[H(\mathbf C^{(0)})\mid q,h]\le\alpha_kH^*$. Using the first inequality of \cref{eq:g-h-minkowski}, $L_2$ Minkowski over the server tape (with $D_q$ constant), this guarantee, and then the second inequality of \cref{eq:g-h-minkowski},
\begin{align}
\sqrt{\E_{\rm server}[G(\mathbf C^{(0)})\mid q,h]}
&\le\sqrt{D_q}+\sqrt{\E_{\rm server}[H(\mathbf C^{(0)})\mid q,h]}
\le\sqrt{D_q}+\sqrt{\alpha_k}\bigl(\sqrt{D_q}+\sqrt{G^*}\bigr)\nonumber\\
&=(1+\sqrt{\alpha_k})\sqrt{D_q}+\sqrt{\alpha_kG^*}.\label{eq:conditional-root}
\end{align}

\emph{Removing the conditioning.} By the tower property, $\E G(\mathbf C^{(0)})$ is the expectation over local tapes of the conditional expectation in \cref{eq:conditional-root}. Squaring \cref{eq:conditional-root}, taking this expectation, and applying $L_2$ Minkowski over the local tapes gives
\begin{align*}
\sqrt{\E G(\mathbf C^{(0)})}
&\le(1+\sqrt{\alpha_k})\sqrt{\E D_q}+\sqrt{\alpha_kG^*}\\
&\le\sqrt{\alpha_k}(\sqrt{\alpha_k}+2)\sqrt{G^*}
 +(1+\sqrt{\alpha_k})\sqrt{d\gamma^2/2}.
\end{align*}
The last step substitutes \cref{eq:edq-root}: the coefficient of $\sqrt{G^*}$ is $(1+\sqrt{\alpha_k})\sqrt{\alpha_k}+\sqrt{\alpha_k}=\sqrt{\alpha_k}(\sqrt{\alpha_k}+2)$. This proves \cref{eq:warm-root}.

\emph{Corollaries.} At $\gamma=0$, square and use $\Phi\le G$ and $G^*\le2\Phi^*$ for \cref{eq:warm-g-zero,eq:warm-phi-zero}. For $\gamma>0$, apply $(a+b)^2\le(1+\eta)a^2+(1+1/\eta)b^2$ for every $\eta>0$, then $\Phi\le G$ and $G^*\le2\Phi^*$, to obtain \cref{eq:warm-phi-additive}. Deterministic weighted anchor-Lloyd steps that do not increase $H$ preserve this argument, because they can only lower the conditional expectation of $H$ in \cref{eq:conditional-root}.

\subsection{Represented snap rounding}
\label{app:represented-snap}
The displayed $\gamma$ constants of \cref{thm:warmstart} assume the stated geometric grid bound $\|q_{e,j}-z_{e,j}\|_2^2\le d\gamma^2/4$. Rounding a grid point to binary64 adds a represented-anchor residual. For the implemented snap, replace $\sqrt{d\gamma^2/2}$ by $\sqrt{d\gamma^2/2}+\sqrt2\rho$, where $\rho$ bounds this final anchor rounding in Euclidean norm; \cref{eq:represented-quantization} states the resulting bound on $Q$. Alternatively, the proof uses $\sqrt{\E Q}$ for the actual represented-anchor distortion, since $Q$ enters only through \cref{eq:edq-root}. At $\gamma=0$ there is no extra snap rounding.

\subsection{Extension to \texorpdfstring{$m$}{m} groups}
\label{app:m-warm}

Let $G_m=\sum_{g=1}^m\Phi_g$ and $\Phi_m=\max_g\Phi_g$, assuming all $m$ global groups are nonempty. The only change is
\begin{equation}
Q\le\frac{m}{4}d\gamma^2.
\end{equation}
Thus
\begin{equation}
\sqrt{\E G_m(\mathbf C^{(0)})}
\le
\sqrt{\alpha_k}(\sqrt{\alpha_k}+2)\sqrt{G_m^*}
+(1+\sqrt{\alpha_k})\sqrt{\frac{m d\gamma^2}{4}}.
\label{eq:m-root}
\end{equation}
In particular, for $\gamma=0$,
\begin{equation}
\E\Phi_m(\mathbf C^{(0)})
\le
m\alpha_k(\sqrt{\alpha_k}+2)^2\Phi_m^*.
\end{equation}
For nonzero $\gamma$, Young's inequality gives an additive term proportional to $m d\gamma^2(1+\sqrt{\alpha_k})^2=O(m d\gamma^2\log k)$ for any fixed tradeoff parameter. This federated/unlearning-compatible extension does not improve the best known centralized multi-group dependence \citep{makarychev2021socially}.

\section{Tight Heterogeneity Loss of the Natural Merged Summary}
\label{app:kappa-proof}

This appendix proves \cref{thm:kappa} and the matched-slot comparison. The upper bound reuses the proof of \cref{thm:warmstart}; only the local replacement error changes, because \Merged{} runs weighted $D^2$-sampling on each mixed client with locally balanced weights that can differ from the global weights by a factor $\kappa_c$. The lower bound exhibits one client on which that mismatch is realized.

\subsection{Upper bound}
Under the assumptions of \cref{thm:kappa}, the explicit general bound is
\begin{equation}
\sqrt{\E G(\mathbf C^{(0)}_{\rm merged})}
\le
\sqrt{\alpha_k}\!\left((1+\sqrt{\alpha_k})\sqrt{\kappa}+1\right)\sqrt{G^*}
+(1+\sqrt{\alpha_k})\sqrt{\frac{d\gamma^2}{2}}.
\label{eq:merged-root}
\end{equation}
At $\gamma=0$, it becomes
\begin{equation}
\E G(\mathbf C^{(0)}_{\rm merged})
\le
\alpha_k\left((1+\sqrt{\alpha_k})\sqrt{\kappa}+1\right)^2G^*.
\label{eq:merged-zero}
\end{equation}
For mixed client $c$, let $A_c(\mathbf C)=n_{c,A}^{-1}\sum_{x\in X_{c,A}}\dist(x,\mathbf C)^2$ and define $B_c$ analogously. Its local objective is $L_c=A_c+B_c$, while its global contribution is $J_c=a_cA_c+b_cB_c$. Since
\[
\min(a_c,b_c)L_c\le J_c\le\max(a_c,b_c)L_c,
\]
weighted $D^2$-sampling with $\kappa_c=\max(a_c/b_c,b_c/a_c)$ gives
\[
\E J_c(Z_c)\le\max(a_c,b_c)\alpha_k L_c(\mathbf C_G^*)
\le\alpha_k\kappa_c J_c(\mathbf C_G^*).
\]
For a pure client its sole global weight cancels from both sides, giving factor one. Summing yields $\E D_z^{\rm merged}\le\alpha_k\kappa G^*$. The total normalized mass remains two, so \cref{eq:edq-root} becomes
$\sqrt{\E D_q}\le\sqrt{\alpha_k\kappa G^*}+\sqrt{d\gamma^2/2}$.
Substitute this in the unchanged server composition following \cref{eq:conditional-root}. The coefficient of $\sqrt{G^*}$ is
$\sqrt{\alpha_k}((1+\sqrt{\alpha_k})\sqrt\kappa+1)$, proving \cref{eq:merged-root}; setting $\gamma=0$ and squaring proves \cref{eq:merged-zero}.

\subsection{Matching lower bound}

The idea is to make one client nearly pure: it holds most of group A but only a small fraction $\varepsilon$ of group B, placed far away. Local balancing gives that small B fraction half of the client's sampling probability, although its global weight is small, so the server often receives, and sometimes selects, a far anchor.

Take $d=k=1$, $\gamma=T_a=0$, equal global group sizes, rational $\varepsilon\in(0,1/2]$ and exactly represented $L>0$. Client 1 holds fractions $1-\varepsilon$ of A at $0$ and $\varepsilon$ of B at $L$; client 2 holds the remaining fractions, both at $0$. Thus $\kappa=(1-\varepsilon)/\varepsilon$. Counts divisible by the denominator of $\varepsilon$ make this a finite instance for every rational $\kappa\ge1$.

Client 1's locally balanced sampler gives total mass one to each group, choosing $0$ or $L$ with probability $1/2$; client 2 always sends $0$. Each representative has global weight $a_c+b_c=1$, so the server outputs $L$ with probability $1/4$ and $0$ with probability $3/4$. Hence
\begin{align*}
G(c)&=2c^2-2\varepsilon Lc+\varepsilon L^2,&
G^*&=\varepsilon(1-\varepsilon/2)L^2,\quad c_G^*=\varepsilon L/2,\\
\E G&=(1+\varepsilon)L^2/2,&
\frac{\E G}{G^*}&=\frac{1+\varepsilon}{2\varepsilon(1-\varepsilon/2)}\ge\kappa/2.
\end{align*}
Similarly, $\Phi^*=\varepsilon(1-\varepsilon)L^2$ and $\E\Phi=(1+3\varepsilon)L^2/4$, giving $\E\Phi/\Phi^*\ge\kappa/4$.

The restriction $T_a=0$ matters. One exact weighted one-center anchor-Lloyd step instead returns $0$ or $L/2$, each with probability $1/2$; at $\kappa=8$ the $G$ and $\Phi$ ratios become $99/34$ and $117/64$, below the asserted $\kappa/2$ and $\kappa/4$. No different refined witness is proved.

\subsection{A matched-slot comparison for a specific merged sampler}
\label{app:s2-matched}
\begin{proposition}[A pre-refinement gap with matched anchor slots]
\label[proposition]{prop:s2-matched-budget}
For each rational $\kappa\ge1$, a finite two-client, two-group instance in
one dimension has $k=1$ and heterogeneity parameter $\kappa$ on which the
natural \textnormal{\scshape Merged-2k} and \SG{} each send exactly two
anchor slots per client. With $\gamma=T_a=T=0$ and ideal independent
categorical bits, their warm-start centers satisfy
\begin{align}
 \frac{\mathbb E G(C_{\rm M2})}{G^*}&=\frac{\kappa+5}{3},&
 \frac{\mathbb E\Phi(C_{\rm M2})}{\Phi^*}&=\frac{\kappa+8}{6},
 \label{eq:s2-matched-ratios}\\
 \frac{\mathbb E G(C_{\rm SG})}{G^*}&=2,&
 \frac{\mathbb E\Phi(C_{\rm SG})}{\Phi^*}&=\frac32.
 \nonumber
\end{align}
Here Merged-2k uses one locally group-balanced weighted $D^2$-sampling call per
client with budget $\min(2k,n_c)$, followed by integer group-count transfer
and globally normalized server weights. Matching slots does not match
positive-weight locations, metadata, serialized bytes, or runtime. On this
particular family, one full-data guarded fair-refinement round maps every
valid one-center incumbent to the optimum, for every declared bisection
length $L$. Thus the displayed separation concerns initialization only, before fair refinement.
\end{proposition}

\begin{proof}[Proof of Proposition~\ref{prop:s2-matched-budget}]
We compute the probability of selecting a nonzero endpoint under each sampler, then express both objective ratios in terms of that probability. Write $\kappa=u/v$ with integers $u\ge v\ge1$, put
$N=2(u+v)$ and $\varepsilon=v/(u+v)=1/(\kappa+1)$, and give distinct
persistent identities to all records. Client 0 contains $2u$ group-A
records at $0$ and $v$ group-B records at each of $-1,+1$. Client 1
contains $2v$ group-A and $2u$ group-B records, all at $0$.
Both clients and both global groups have size $N$. The group ownership
fractions are $(1-\varepsilon,\varepsilon)$ on client 0 and
$(\varepsilon,1-\varepsilon)$ on client 1, so their largest ratio is
$\kappa$. Every local group slice has at least two records; neither
method's sampling call is saturated. Coordinates $-1,0,1$ are represented
exactly, for example under either of the evaluated fine-grid precisions
with clipping radius 8. No transform is refitted.

For Merged-2k on client 0, the total local sampling masses at
$(0,-1,+1)$ are $(1,1/2,1/2)$. Its first-location probabilities are
$(1/2,1/4,1/4)$. After choosing $0$, the two endpoints are equally
likely. After choosing an endpoint, the masses at $0$ and the opposite
endpoint are $1$ and $2$, giving probabilities $1/3$ and $2/3$.
Consequently the six ordered representative-location pairs have laws
\[
\begin{array}{c|rrrrrr}
\text{pair}&(0,-1)&(0,+1)&(-1,0)&(+1,0)&(-1,+1)&(+1,-1)\\
\hline
\Pr&1/4&1/4&1/12&1/12&1/6&1/6.
\end{array}
\]
These are geometric probabilities: summing the exact rational
probabilities of the exchangeable record identities gives each entry.

If the pair contains $0$ and one endpoint, the $0$ representative receives
all $2u$ group-A records and the $v$ group-B records at the unselected endpoint;
the selected endpoint receives its own $v$ group-B records. Thus its server
weight is $\varepsilon/2$ and the other weight is
$1-\varepsilon/2$. If both endpoints are selected, their own group-B records
stay with them, and every group-A record at $0$ ties and goes to the first
indexed representative. Their group counts are therefore
$(2u,v)$ and $(0,v)$ in selection order; both representatives are
endpoints and their total server weight is $1$.

Client 1 selects an arbitrary first record at $0$, then encounters zero
$D^2$ potential and completes with the smallest unused record index.
Every record ties and is assigned to the first indexed representative,
so its two server weights are $1,0$. The zero-weight row remains in the
ordered table but is never selected by the positive-mass categorical
call. The server's total weight is always $2$. The probability that its
single center is an endpoint is hence
\[
p_{\rm M2}=\frac13\frac12+\frac23\frac{\varepsilon}{4}
          =\frac{1+\varepsilon}{6}.
\]
The declared indexed tie rule gives equal probabilities to the two signs;
the combined endpoint probability above also does not depend on how
the zero-location ties between endpoints are assigned.

For SplitGroup-2k, each of the four nonempty slices selects one
representative and transfers all its own records to that slot.
Only client 0's group-B representative is an endpoint. Its server weight is
$\varepsilon$, while the other three weights are
$1-\varepsilon,\varepsilon,1-\varepsilon$ at $0$.
Therefore $p_{\rm SG}=\varepsilon/2$. Both methods send four slots in
total. In particular, coincident positive-weight SG slots and the
zero-weight Merged-2k slot are not removed to manufacture a different budget.

For any real center $c$, the two group costs are
$\Phi_A(c)=c^2$ and $\Phi_B(c)=c^2+\varepsilon$. Hence
$G(c)=2c^2+\varepsilon$, $\Phi(c)=c^2+\varepsilon$, and
$G^*=\Phi^*=\varepsilon$, attained at $c=0$. A law supported on
$\{-1,0,+1\}$ with endpoint probability $p$ has expected ratios
$1+2p/\varepsilon$ and $1+p/\varepsilon$. Substituting the endpoint probabilities $p_{\rm M2}$ and $p_{\rm SG}$ proves
\eqref{eq:s2-matched-ratios}. Moreover $p_{\rm M2}\ge1/6$ and,
on that endpoint event, the realized ratios are $2\kappa+3$ and
$\kappa+2$.

Finally, for $k=1$ every point lies in the sole cluster and both group
means equal exactly $0$. Every endpoint or bisection candidate of the
declared fair update is therefore the represented center $0$, whatever
the incumbent and whatever $L\in\{0,\ldots,256\}$.
The objective guard accepts this optimum. This proves the stated
refinement limitation; it does not prove a separation for $T_a>0$ or for
the final refined model. The ideal categorical calculation is not an
identity for frequencies over finitely many PCG64 seeds.
\end{proof}

\subsection{Scope of the lower bound}

The lower bound is specific to the declared merged sampler. A universal claim about ``$k$ unlabeled real representatives'' needs a finite-information or geometric-coreset model: real coordinates can encode group labels in binary digits.

\section{Assignment Certificates}
\label{app:certificates}

This appendix proves \cref{lem:basic-cert} and gives the runner-up certificate and cache-specific tightness result. Recall that a retained point $x$ has cached winner $j$, nearest and second-nearest distances $d_1<d_2$, runner-up identity $r$, and that $\delta_m$ is the displacement of indexed center $m$ between the cached and replay trajectories.

\begin{proposition}[Tightness and a stronger constant-time cache]
\label{prop:cert-tight}
If the $\delta_m$ are displacement upper radii, \cref{eq:basic-cert} is the strongest universally sound test for a unique winner obtainable from only $(j,d_1,d_2)$ and those radii. With exact displacement norms it is tight for $k\ge3$; for $k=2$ the strongest scalar-only condition is $d_1+\delta_j<|d_2-\delta_r|$.

With cached runner-up identity $r$, define $D_{\neg j,r}=\max_{m\notin\{j,r\}}\delta_m$, with the associated term treated as $+\infty$ when no such competitor exists. The strictly stronger certificate
\begin{equation}
d_1+\delta_j<
\min\left\{|d_2-\delta_r|,\ \max(0,d_2-D_{\neg j,r})\right\},
\label{eq:runner-cert}
\end{equation}
remains $O(1)$ per point after precomputing the three largest center shifts and their identities.
\end{proposition}

The proof occupies the following subsections: soundness of the basic test, its tightness under radii and under exact norms, and soundness and cost of the runner-up test.

\begin{proof}[Proof of \cref{lem:basic-cert}]
By the triangle inequality,
$\|x-c'_j\|\le d_1+\delta_j$. For every $m\ne j$, the reverse triangle inequality gives
$\|x-c'_m\|\ge\|x-c_m\|-\delta_m\ge d_2-\delta_m$. If \cref{eq:basic-cert} holds, the winner upper bound is strictly smaller than every competitor lower bound, so $j$ remains uniquely nearest.
\end{proof}

\subsection{Tightness with displacement radii}

To show tightness, it suffices to realize a tie or a changed winner whenever the test fails, using only configurations consistent with the cache. Suppose only upper radii are known. Place the old winner and one competitor on a common line through $x$, at distances $d_1$ and $d_2$. Move the winner directly away from $x$ by its allowed radius and the competitor directly toward $x$ by its allowed radius. Whenever
$d_1+\delta_j\ge d_2-D$, where $D=\max_{m\ne j}\delta_m$, the competitor can tie or win. Thus no universally sound scalar-only certificate for a unique winner can accept a configuration rejected by \cref{eq:basic-cert}.

\subsection{Exact displacement norms}

For $k\ge3$, if $D<d_2$ the preceding radial construction works. If $D\ge d_2$, let one competitor provide the cached runner-up distance $d_2$ and place a different maximum-shift competitor initially at distance exactly $D$; it can move to $x$. Hence the same scalar certificate remains tight.

For $k=2$, the sole competitor's old distance is exactly $d_2$, and a displacement of exact norm $\delta_r$ leaves it at distance at least $|d_2-\delta_r|$. Therefore the strongest universal condition is
$d_1+\delta_j<|d_2-\delta_r|$. The example $d_1=0,d_2=1,\delta_j=0,\delta_r=100$ shows strict improvement over \cref{eq:basic-cert}: the competitor must end at least $99$ away although the basic test fails.

\subsection{Runner-up identity certificate}

For the cached runner-up $r$, reverse triangle gives the exact-norm lower bound $|d_2-\delta_r|$. Every other competitor starts at an unknown distance at least $d_2$. Given only its shift norm $\delta_m$, the smallest geometrically consistent post-shift distance is
$\max(0,d_2-\delta_m)$. The largest shift among these other competitors gives the non-runner lower bound in \cref{eq:runner-cert}. Precompute the three largest $\delta_m$ values and identities each round. Removing at most the two identities $j$ and $r$ leaves the required maximum among those three whenever another competitor exists, so each query takes constant time.

A runner-up-only certificate without the non-runner maximum is unsound. Let $x=0$; place the winner at $0$, runner-up at $10$, and a third center at $11$. Move the winner to $0.5$, leave the runner-up fixed, and move the third center to $0.1$. The old runner-up remains far, but the third center becomes the new winner.

\subsection{Numerically sound implementation}

The exact-distance test is implemented with outward bounds: an upper winner bound $\overline d_1$, lower second-distance bound $\underline d_2$, upper shift bounds $\overline\delta_m$, and a certified interval for $|d_2-\delta_r|$ when using overshoot. Accept only strict separation of the winner upper and competitor lower bounds. Ambiguity triggers exact reassignment, preserving soundness at a possible coverage cost.

\section{Exact Arithmetic and Same-Seed Replay}
\label{app:exactness}

This appendix explains why exact arithmetic is needed for patching and proves \cref{thm:exactness}.

\subsection{Why fixed floating reduction order is not enough}

In IEEE-754 double precision, left-to-right summation of $10^{16},1,-10^{16}$ gives zero because the unit contribution is lost in the first addition. Deleting the record $1$ by aggregate subtraction gives $-1$, while a fresh retained reduction of $10^{16},-10^{16}$ gives zero. The two arithmetic expression trees differ; deterministic order does not make subtraction invert a rounded sum.

Exact and reproducible summation are established topics \citep{neal2015exact,ahrens2020summation}. Here subtraction must invert addition, so the fixed-point representation and validated integer arithmetic of \cref{sec:unlearning,app:implementation} provide exact signed patches, not merely repeatable rounded output. Completed sums/patches use arbitrary-precision integers; no cryptographic modulus is implemented. Every branch and guard scores actual represented candidates exactly: deterministic floating arithmetic alone does not establish monotonicity.

\subsection{Proof of \cref{thm:exactness}}

\paragraph{Objective and strategy.}
We must show that replay and fresh training on $X\setminus R$ produce the same indexed center tuple at every round. Replay reuses some quantities and recomputes others, so the obstacle is to show that every reused or patched quantity equals its fresh counterpart exactly. The argument first matches the initializer and then proceeds by induction over rounds: if the centers agree at round $t$, the patched statistics equal the fresh statistics, and the same deterministic update produces equal centers at round $t+1$.

Fix seed $s$ and view each random entity as a pure function of its input and keyed stream. Write $\widetilde{\mathbf C}^{(t)}$ for replay and $\mathbf C_R^{(t)}$ for fresh training on $X\setminus R$.

\paragraph{Initialization.}
An untouched slice has the same input and tape, so its reused summary equals fresh training. Touched slices are rerun on retained inputs with the fresh tape. The server receives the same canonical anchors and retained group sizes, rebuilds weights $h/n_g$, and repeats the same weighted $D^2$-sampling and optional anchor-Lloyd refinement. Thus $\widetilde{\mathbf C}^{(0)}=\mathbf C_R^{(0)}$ as indexed tuples.

\paragraph{Induction.}
Suppose the tuples agree at round $t$, so fresh training and replay assign survivors against the same centers. The cached statistics of group-cell $(g,j)$ sum the contributions $(1,x,\|x\|^2)$ of all records of $X_g$ with cached label $j$; the fresh retained statistics sum the same contributions over survivors whose fresh nearest-center label is $j$. Replay removes deleted records from their cached group-cells; their hypothetical revised labels are irrelevant. A sound certificate guarantees that a certified survivor's cached label is its unique fresh label, so it needs no correction. Every other survivor is assigned directly under the shared center tuple with the same exact comparison and smallest-index tie rule as fresh training; if its label changed, a signed correction moves its contribution from the old group-cell to the new one. Because these corrections are exact integer additions and subtractions, the patched $N,S,SS$ equal the fresh retained statistics exactly. The update $\GFU_L$ is a deterministic function of the current centers and these statistics, and by \cref{prop:central-equivalence} the federated call equals the centralized one, so the tuples agree at $t+1$. If abandonment fires, discarding all partial corrections and rebuilding from zero gives the same statistics, so the induction still applies. This proves equality of the full indexed trajectory and final model.

For fixed or seed-independent $R$, the two projections are equal as functions of $s$, so randomizing $s$ yields equality in distribution. Pointwise equality needs reproducibility, not ideal randomness; the argument supplies no fresh-independent guarantee for a request chosen adaptively from the same seed-dependent checkpoint.

\subsection{Validity domain}
Deleting a whole global group makes $h/n_g$ and $\Phi_g$ undefined and is excluded. Optimizing only over retained active groups would define a different map. Deterministic center completion is required when fewer than $k$ distinct positive-weight anchors remain.

\section{Work, Communication, and Conditional Speed}
\label{app:complexity}

This appendix states the replay work bound summarized in \cref{sec:work}, accounts for each term, gives a conditional expected-work bound, and proves the scoped obstruction of \cref{lem:s2-word-coupling,prop:s2-deletion-work}.

\subsection{The work bound}
\label{app:work-bound}
In the work bounds, $r=|R|$ denotes deletion size, not the runner-up index of \cref{sec:certificates}. Let $W_0$ charge public-call setup, and $W_I$ touched-slice and server recomputation; let $F_t$ be the number of failed certificates in round $t$ before any fallback. As in \cref{alg:unlearn}, refinement rounds are indexed $t=0,\ldots,T-1$, and $\mathcal T_{\rm dir}\subseteq\{0,\ldots,T-1\}$. At unit arithmetic cost, the scanning implementation has the upper bound
\begin{equation}
\begin{split}
O\!\Bigl(W_0+W_I+nT+CkdT+Ck\log k\,T+(L+2)kdT+rdT\\
+\bigl[\textstyle\sum_{t=0}^{T-1}F_t+(n-r)|\mathcal T_{\rm dir}|\bigr](kd+k\log k)\Bigr),
\end{split}
\label{eq:work}
\end{equation}
Here $W_0=O(nd+nr+(n+r)\log(n+r+2)+C)$ conservatively covers once-per-call masks, group bookkeeping, validation/indexing, and coordinate materialization, independently of $T$; \cref{app:startup} states the cost model. It excludes the $D^2$-sampling work in $W_I$. The set $\mathcal T_{\rm dir}$ indexes all full direct rounds, including the triggering-round rebuild and all rounds after abandonment; $F_t=0$ in rounds that start direct. The $nT$ term describes this scan, not a universal necessity claim. Sorting, certified comparisons, arbitrary-precision bit costs, and duplicate fallback work are detailed below. Dense logical Phase-II packets contain $2k(d+2)$ numeric values per client per round; this is not a measured byte or network cost. Point-cache state is linear in $nT$ for a fixed layout.

Margin/shift assumptions bound ideal failures; numerical-only failures, full fallback and sunk work are charged separately (\cref{app:margin,prop:stable-replay}). The conditional Phase-II bound does not imply unconditional end-to-end speed.

\subsection{Term-by-term accounting}

\label{app:startup}
The public call performs setup even at $T=0$: validation/sorting of row IDs, uniqueness and membership checks, ownership and group-count masks, touched-slice extraction, retained/removed coordinate copying, and retained-coordinate dequantization. At unit arithmetic cost and expected constant-time dictionary lookup,
\[
W_0=O\!\left(nd+nr+(n+r)\log(n+r+2)+C\right).
\]
The $nd$ term covers coordinate materialization and linear masks; request sorting costs $O(r\log(r+2))$. Membership permits comparison/sort work and $O(nr)$ pairwise work, without assuming every NumPy branch is linear. Integer row indices are bounded by client populations, also bounding index-table initialization. Local $D^2$-sampling, multiplicity assignment, snapping and server work belong to $W_I$; $O(CT)$ timer allocation is absorbed by round costs. This is neither a bit-complexity nor an adversarial hash-table bound.

To see why $W_0$ cannot be omitted, hold $C,k,d$ and touched-slice size fixed, grow an untouched slice, and set $T=0$. Then $W_I$ may remain bounded, but masks require $\Omega(n)$ work and copying requires $\Omega((n-r)d)$.

The algorithm timer encloses public-call validation and materialization. Separately recorded worker preparation enters E2E; slice bookkeeping partly shares the Phase-I timer and group-count scans are not isolated. Thus $W_0,W_I$ and primitive timers cannot be substituted for one another or used to decompose saved seconds into unmeasured costs.

Let $M_s$ be the total number of records in touched slices and $Q_s$ the number of anchors quantized there. At unit arithmetic cost, touched-slice $D^2$-sampling and multiplicity assignment contribute $O(M_s(kd+k\log k))$ when all center distances are sorted. Quantizing $Q_s$ anchors costs $O(Q_sd)$, rebuilding weights costs $O(Ck)$, and server weighted $D^2$-sampling with optional anchor-Lloyd refinement on at most $2Ck$ anchors contributes $O(Ck^2(d+\log k)(1+T_a))$. These terms define $W_I$ in \cref{eq:work}.

Each round computes indexed shifts, visits survivors during active certification, patches dense $2k(d+2)$ packets, and evaluates $L+2$ fair candidates. Assignment sorting costs $O(k\log k)$ comparisons per point after $O(kd)$ distance arithmetic. If top-three shifts are found by sorting per client, charge $O(Ck\log k)$ rather than claiming a once-per-round linear selection. Exact/certified arithmetic additionally depends on integer/rational bit widths; the unit-cost formula is not a bit-complexity claim. Interval ambiguity can add exact fallback cost without changing the map.

Abandonment is checked after the current round's client work and before its center update. If the failed fraction strictly exceeds $1/2$, discard all patched statistics and rebuild the complete current round, then remain direct. Failed points already assigned in the triggering round are assigned again. Passes discarded by this transaction save no final assignment work. Direct replay starts in direct mode and has no measured abandonment event.

Dense logical packets contain $2k(d+2)$ values per client and round. Actual serialized bytes depend on integer encodings, identifiers, framing, and any cryptographic protocol. The local simulator implements no network, secure aggregation, retry, or straggler system. Instrumented accumulation calls measure materialized call volume, not necessarily logical transmissions.

The cache layout is specified in \cref{app:implementation}: $64nT$ in-memory and $48nT$ serialized raw-array bytes, excluding object overhead, framing and other state. These are not process RSS or total checkpoint bytes; measurements use the actual jointly executed $n,T$ configuration (\cref{app:checkpoint-costs}).

\subsection{Margin-density bound}
\label{app:margin}

Suppose for every relevant $h$,
\begin{equation}
\#\{x:d_2^{(t)}(x)-d_1^{(t)}(x)\le h\}\le\mu_tnh.
\end{equation}
Let $F_t^{\rm ideal}$ count failures of the basic scalar certificate evaluated with exact distances and shifts on the active survivor population. If every center displacement is at most $\varepsilon_t$, failure of that ideal test implies
$d_2-d_1\le2\varepsilon_t$. Hence
\begin{equation}
F_t^{\rm ideal}\le\min\{n,2\mu_t\varepsilon_tn\}.
\end{equation}
The ideal runner-up certificate can only reduce this failure count. Outward numerical bounds may reject even when the corresponding exact scalar test passes. Let $B_t$ count these additional numerical-only failures for the implemented certificate. Its observed failed-assignment count satisfies
\begin{equation}
F_t\le F_t^{\rm ideal}+B_t.
\end{equation}
Thus margin density alone does not bound observed numerical failures: charge $B_t(kd+k\log k)$ assignment work, exact-number bit costs, and any triggered full fallback separately. The density assumption is an empirical CDF condition with units; $\mu_t$ is not a dimensionless universal constant.

\subsection{Expected touched-slice mass}

A slice contributes to $M_s$ exactly when at least one of its records is requested. For a uniformly random batch of size $r$ and pseudo-slice $e$ of size $n_e$,
\begin{equation}
\Pr(e\text{ touched})
=1-\frac{\binom{n-n_e}{r}}{\binom nr}
\le\frac{rn_e}{n}.
\end{equation}
Multiplying each touch probability by the slice size and summing gives
\begin{equation}
\E M_s\le\frac r n\sum_en_e^2\le r s_{\max},
\end{equation}
where $s_{\max}=\max_en_e$. Concentrated slices or adversarial requests can make $M_s$ much larger.

\subsection{A conditional expected replay theorem}

\begin{proposition}[Locally stable replay]
\label{prop:stable-replay}
Assume the maximum indexed center perturbation obeys
\begin{equation}
\delta_0\le a_0\frac rn,
\qquad
\delta_{t+1}\le q\delta_t+a\frac rn,
\qquad 0\le q<1,
\end{equation}
with $a_0,a\ge0$, nonnegative perturbations $\delta_t$, and the margin-density condition holding uniformly with $\mu\ge0$. Then
\begin{equation}
\delta_t\le L_s\frac rn,
\qquad L_s=a_0+\frac{a}{1-q},
\end{equation}
so $F_t^{\rm ideal}\le2\mu L_sr$. Including numerical-only certificate failures, Phase-II replay work is
\begin{equation}
O\!\left(nT+rdT+Ck\log k\,T
+\left[\mu L_s rT+\sum_{t=0}^{T-1}B_t+(n-r)|\mathcal T_{\rm dir}|\right](kd+k\log k)\right)
\end{equation}
up to dense packet, update, exact-number bit, and communication costs. Here $\mathcal T_{\rm dir}$ includes the triggering reconstruction and all later direct rounds. The failed assignments already performed in the triggering round remain charged in the certificate term, so duplication is included. No bound on $B_t$ is asserted from margin density alone.
\end{proposition}
\begin{proof}
Put $h=r/n\ge0$. Nonnegativity of $q$ permits induction through the recurrence:
\[
\delta_t\le q^t\delta_0+ah\sum_{i=0}^{t-1}q^i
\le\left(a_0+\frac{a}{1-q}\right)h.
\]
The empty sum covers $t=0$. Apply the margin bound in each active certificate round, add the numerical-only failures, and charge each full direct round separately as in \cref{eq:work}. Scanning, deleted-point corrections, and per-client shift sorting give the other displayed terms.
\end{proof}

The assumptions are plausible only in a stable local basin: relevant group-cells have nonvanishing mass, group means are well separated from assignment boundaries, the fair dual solution is nondegenerate, and the guarded update map is locally contractive. They are not unconditional. In particular, keyed categorical $D^2$-sampling is discontinuous: one deletion can change masses, rejection consumption, or selected anchors, and cascade through later draws. We therefore do not claim an unconditional end-to-end speed theorem controlled only by $r/n$.

\subsection{Shared-word sampling and linear discarded work}
\label{app:s2-coupling}

The main text (\cref{sec:work}) asks whether one deletion can defeat certification. The answer below is specific to the declared sampler, cache, and fallback policy. \Cref{lem:s2-word-coupling} isolates the sampling mechanism: gcd reduction lets two nearly identical categorical laws read the same random word through different thresholds, so they disagree far more often than their total variation distance. \Cref{prop:s2-deletion-work} embeds this mechanism in a legal deletion whose replayed initializer flips the second center with constant probability, which changes every retained group-A label and triggers fallback.
\begin{lemma}[Shared words need not give a stable categorical coupling]
\label[lemma]{lem:s2-word-coupling}
Consider the FTF-1 categorical sampler: clear rational denominators,
divide the integer masses by their gcd, and use low-bit rejection
sampling with fresh independent 64-bit words, discarding unused bits.
Start two calls at the same word offset, with ordered masses
$(m,m)$ and $(m,m-1)$, where $m=2^h$, $h\ge1$.
The probability that they select different categories is exactly
\[
 \frac{m}{2m-1}=\frac12+\frac{1}{2(2m-1)},
\]
although the total variation distance of their two-category marginal
laws is only $1/[2(2m-1)]$. The disagreement contribution from an
accepted first proposal is already $1/2$; rejection is not necessary.
\end{lemma}

\begin{proposition}[One legal deletion can force linear discarded point work]
\label[proposition]{prop:s2-deletion-work}
For every $m=2^h$, $h\ge2$, there is a three-client, two-group,
one-dimensional FTF-1 dataset with $n=3m$, $k=2$, $\gamma=T_a=0$
and a fixed legal single-record deletion such that the original and
retained laws of the ordered initialization coordinate tuple have
total variation distance $1/[2(2m-1)]$. Under the declared shared-word
coupling, an event of probability at least
\[
 \beta_m=\frac{(m-1)^2}{4m(2m-1)}>\frac1{16}
\]
changes the round-zero indexed labels of all $2m-1$ retained group-A
records. On this event Basic has $3m-1$ ideal certificate failures and
Runner-up has $2m-1$. With the declared strict majority abandonment
rule and $T\ge1$, both abandon at round zero. Their useful saved
coverage in that round is zero; their duplicated point assignments
are respectively $3m-1$ and $2m-1$.
Thus for either variant
\[
 \mathbb E F_0\ \ge\ \frac{(m-1)^2}{4m}=\Omega(n),
 \qquad
 \mathbb E J_0\ \ge\ \frac{(m-1)^2}{4m}=\Omega(n).
\]
The expectations concern ideal independent words and this particular
sampler, indexed cache, and fallback policy. They are not wall-clock
lower bounds, a claim about every exact unlearning method, or a causal
explanation of a measured ACS trial.
\end{proposition}

\begin{proof}[Proof of Lemma~\ref{lem:s2-word-coupling}]
We count disagreement on accepted first proposals, then add the contribution from rejection. Multiplying either mass vector by any positive rational does not alter
its primitive integer vector after denominator clearing and gcd
reduction. The old vector is $(1,1)$ and uses one low bit. The new
vector is $(m,m-1)$, of total $2m-1$, and uses $h+1$ low bits.
For widths above 64, FTF-1 assembles words with the first word in the
least significant position before masking. Write the new first
proposal as $U$, uniform on $\{0,\ldots,2m-1\}$. Both calls read the
same first word, and the old call uses only its lowest bit, so the old
selection is the parity of $U$ (even selects the first category). The
new call accepts $U\le2m-2$ and selects the first category exactly when
$U<m$. Among accepted proposals, disagreement therefore occurs for
the $m/2$ odd values below $m$ and the $m/2$ even values from $m$
through $2m-2$. These $m$ of the $2m$ equally likely proposals
contribute $1/2$.
The only rejected value is $2m-1$, which makes the old category the
second one. The new call then uses untouched independent proposal
blocks and eventually selects the first category with probability
$m/(2m-1)$. The remaining disagreement contribution is
$(1/(2m))m/(2m-1)=1/[2(2m-1)]$, and the total is
$1/2+1/[2(2m-1)]=m/(2m-1)$. The marginal laws are
$(1/2,1/2)$ and $(m/(2m-1),(m-1)/(2m-1))$, giving the stated TV.
Calls with width zero draw no words; positive widths discard unused
high bits after every accepted or rejected proposal. This argument
starts both calls at the same offset and does not assert alignment
of later calls after a differing width or rejection history.
\end{proof}

\begin{proof}[Proof of Proposition~\ref{prop:s2-deletion-work}]
The construction keeps local representative coordinates fixed. A shared first center at zero can then be followed by opposite endpoints, changing retained assignments and triggering fallback. We verify this event before computing the marginal total variation. Give client 0 $m$ distinct group-A records at $-1$, client 1 $m$
distinct group-A records at $+1$, and client 2 $m$ distinct group-B
records at $0$. Within each client, order records by their persistent
identity. Delete the last record of client 1, independently of the
random words. The retained global groups have sizes $2m-1$ and $m$,
so the request is legal. The coordinates are exactly represented on
the declared fine grids, and every nonempty slice still has at least
three records, exceeding $k=2$.

Each slice chooses one record, encounters zero remaining potential,
and completes with the smallest unused index. Since both local
representatives coincide, every record is assigned to the first
indexed slot. The local geometric outputs are therefore fixed,
regardless of any change in selected record identities. In canonical
client and representative order the server coordinates are
$(-1,-1,+1,+1,0,0)$, with old weights
\[
 (1/2,0,1/2,0,1,0)
\]
and retained weights
\[
 (m/(2m-1),0,(m-1)/(2m-1),0,1,0).
\]
The first-call primitive vectors are $(1,0,1,0,2,0)$ and
$(m,0,m-1,0,2m-1,0)$. Zero-weight slots are retained and skipped by
the strict cumulative comparison, not collapsed into another category.

Let $U$ be the retained first proposal, uniform on
$\{0,\ldots,4m-1\}$, and consider
\[
 E=\{2m-1\le U\le4m-3,\quad U\bmod4\in\{2,3\}\}.
\]
It contains $m-1$ proposal values and has probability $(m-1)/(4m)$.
Both calls accept their first proposal and choose coordinate $0$ on
$E$: the old primitive vector has total $4$, so the old call reads the
two low bits $U\bmod4$ and maps $\{2,3\}$ to coordinate $0$, while the
retained vector has total $4m-2$ and maps $2m-1\le U\le4m-3$ to
coordinate $0$. Conditional on this common prefix, the old second primitive
masses at $(-1,+1)$ are $(1,1)$ and the retained ones are $(m,m-1)$.

If the retained first width $h+2\le64$, both calls have consumed
exactly one word. The next words are independent of $E$, and
Lemma~\ref{lem:s2-word-coupling} gives second-location disagreement
probability $m/(2m-1)$. If $h+2>64$, the retained first proposal
consumes $\ell=\lceil(h+2)/64\rceil\ge2$ words. The old second
selection uses the low bit of word 1 inside that first block (words
are numbered from 0). Conditional on the block it is fixed, whereas
the retained second call starts at word $\ell$ and uses independent
words. Its probability of choosing the opposite endpoint is at
least $(m-1)/(2m-1)$. Thus in both cases the probability of
$E$ and opposite second endpoints is at least $\beta_m$.
The inequality $\beta_m>1/16$ follows from
$2m^2-7m+4>0$ for $m\ge4$.

On this event the indexed initializer changes between $(0,-1)$ and
$(0,+1)$. The two indexed shifts are $(0,2)$, and every old winner
has margin $d_2-d_1=1$. Every retained group-A point changes its indexed
label, while all group-B points remain in slot 0. Basic fails for every
point. Runner-up fails for every group-A point; for a group-B point its left
side is $0$ and its reverse-triangle lower bound is $|1-2|=1$,
so it passes. For $k=2$ there is no further competitor: its
associated lower-bound term is $+\infty$. These strict inequalities
also have the stated decisions under the implemented outward bounds:
zero displacements/distances are exactly enclosed, the successful
comparison has unit slack, and failed ideal tests cannot pass the
sound implementation. The small controlled-word fixtures check both
evaluated $(b,L)$ settings.

Since $(2m-1)/(3m-1)>1/2$, either variant triggers the strict majority
rule. Every attempted failure has already been directly assigned;
the transaction is discarded and all $3m-1$ survivors are assigned
again. Basic discards zero passes and duplicates $3m-1$ assignments;
Runner-up discards $m$ passes and duplicates $2m-1$. Both have
$S_0=0$, and later rounds are direct with $A=P=S=J=0$ and
$N=3m-1$ per round. Multiplying the event probability by $2m-1$
proves the expectation bounds. Certificate scans, shift sorting,
arithmetic, and other overhead are additional costs, not included in
these assignment counts.

It remains to compute the total variation of the marginal initialization laws, separately from their shared-word coupling. Take the six ordered coordinate tuples
\[
(0,-1),(0,+1),(-1,0),(-1,+1),(+1,0),(+1,-1)
\]
as the sample space for TV. Put $a=m/(2m-1)$ and $b=1-a$.
Their retained probabilities are
\[
 \left(\frac a2,\frac b2,
 \frac{a}{2(1+4b)},\frac{2ab}{1+4b},
 \frac{b}{2(1+4a)},\frac{2ab}{1+4a}\right).
\]
The old vector is $(1/4,1/4,1/12,1/6,1/12,1/6)$.
For $1/2<a\le4/7$, precisely entries 1, 3 and 4 increase.
For example, the sign of entry 4's difference is the sign of
$(2a-1)(5-6a)$, which is positive on this interval. Entries 3 and 4
sum to $a/2$; their old sum is $1/4$. The total positive change is
therefore $a-1/2=1/[2(2m-1)]$, proving the TV claim.
\end{proof}

\section{Deterministic Multi-Group Refinement}
\label{app:multi-group}

This section gives the mathematical extension of the fixed-partition update to $m>2$ groups, following \citet[Section~3]{ghadiri2021socially}. Assume $\alpha_{g,j}\ge0$ and omit zero-mass group-cells before forming their means; a mean of an empty cell is never evaluated. The canonical map remains binary; \cref{app:multigroup-prototype} supplies a separately named implementation and bounded pilot. For $m>2$ and a fixed partition, write
\begin{equation}
f_g(\mathbf C)=\Delta_g+\sum_{j=1}^k\alpha_{g,j}\|c_j-\mu_{g,j}\|^2.
\end{equation}
The fixed-partition problem is the convex epigraph program
\begin{equation}
\min_{\mathbf C,\theta}\ \theta
\quad\text{s.t.}\quad f_g(\mathbf C)\le\theta\quad(g=1,\ldots,m).
\end{equation}
Its Lagrange dual is over $\lambda\in\Delta_m$, the simplex of nonnegative group weights summing to one. Let
\begin{equation}
A_j(\lambda)=\sum_g\lambda_g\alpha_{g,j},
\qquad
b_j(\lambda)=\sum_g\lambda_g\alpha_{g,j}\mu_{g,j}.
\end{equation}
For $A_j(\lambda)>0$, a minimizer of the Lagrangian at fixed $\lambda$ is $c_j(\lambda)=b_j(\lambda)/A_j(\lambda)$. This need not minimize the primal minimax objective. If $A_j(\lambda)=0$, nonnegativity implies $b_j(\lambda)=0$, and that weighted quadratic contributes zero. Use the sole occupied group mean if exactly one group occupies the cluster; otherwise retain the incumbent center (including a fully empty cluster). These choices minimize the zero weighted quadratic, without asserting primal optimality. The dual problem maximizes the concave objective over $\Delta_m$:
\begin{equation}
D(\lambda)=\sum_g\lambda_g\Delta_g+
\sum_j\left[
\sum_g\lambda_g\alpha_{g,j}\|\mu_{g,j}\|^2-
\begin{cases}\|b_j(\lambda)\|^2/A_j(\lambda),&A_j(\lambda)>0,\\0,&A_j(\lambda)=0.\end{cases}
\right].
\end{equation}
A separately specified extension can use deterministic dual supergradient ascent, minimization of $-D$, or a primal-dual method; the prototype uses bounded Frank--Wolfe search on this dual. None is added to FTF-1 here. Same-seed exactness requires only that replay reproduce the same deterministic solver; optimization accuracy is logically separate. An exact objective guard preserves monotonicity. For any feasible center tuple, $\max_g f_g(\mathbf C)-D(\lambda)\le\eta$ certifies an $\eta$-optimal fixed-partition update by weak duality, but the result should not be called the exact convex optimum unless an exact solver is actually used.

\section{Sequential Deletion and Cache Purging}
\label{app:sequential}
\label{app:cache}

This section concerns the original record-replay cache; \cref{app:client-summary} gives the separate compact-state invariant for whole-client deletion at $T=0$.

\Cref{thm:exactness} concerns the released model, not the per-point cache needed for a later deletion. After one batch, the center trajectory can match fresh retraining while certified survivors retain old nearest/second-nearest distances. Those values are not the fresh cache for the revised trajectory.

\begin{proposition}[Direct cache reconstruction upper bound]
\label{prop:cache-lower}
\label{prop:cache-upper}
A direct rebuild for $n-r$ retained records and $T$ rounds uses $(n-r)kT$ center--point distance queries, $O((n-r)kdT)$ distance arithmetic, and additional indexed nearest/second selection and storage. This is an upper bound for a specified construction, not a universal Euclidean lower bound.
\end{proposition}

\begin{proof}
Evaluate all $k$ squared distances for each retained record and round, select the first two entries in distance-then-index order, and save their verified bounds and IDs. Each coordinate-based squared distance takes $O(d)$ arithmetic. A full sort adds $O(k\log k)$ comparisons per record; first-two selection can instead be linear. The direct construction proves the upper bound without any assumption of independent geometric distance-table entries.
\end{proof}

Dimension multiplies arithmetic per distance; it is not the number of distance queries. Geometric search data structures can exploit information absent from an arbitrary distance table \citep{deberg2021dynamic}. A universal $\Omega(nk)$ Euclidean query claim also needs a precise oracle model and a simultaneous realizability proof. For example, one-dimensional coordinate access allows sorted-center search. We make no scan-necessity or universal cache-rebuild lower-bound claim. The batch API outputs the model projection; rebuilding or purging a sequential checkpoint is a separate, explicitly charged operation. Stale caches and deleted-record tombstones are not fresh-equivalent state.

\section{Experimental Protocols and Supporting Evidence}
\label{app:experiments}
\label{app:experimental-evidence}
This appendix contains the protocols, evidence, and diagnostic cases needed to assess the claims of \cref{sec:experiments}. It first records evidence versions and validation scope (\cref{app:validation}) and the timing and work definitions shared by all panels (\cref{app:timing}), then gives cohort provenance (\cref{app:p124-protocol}) and each panel in turn. The public supplementary release (v1.0.0) provides implementation source, processed inputs, saved reporting data and reproduction commands. Its inventory distinguishes included evidence from separately retained historical archives; references to ``the artifact'' below follow that inventory. These records support detailed lookup and reproduction, while this appendix gives the methods, denominators, uncertainty, adverse outcomes and scope.

\subsection{Evidence versions and validation scope}
\label{app:validation}
P1/P4 retain the source and evidence version \texttt{corrected-v2-full-2bff36fecd98fbb3}. P2/P3 are newly measured with implementation \texttt{FTF-1D-descriptor-v1}, in runs \texttt{ftf1d-p2-full-20260919} and \texttt{ftf1d-p3-full-20260919}. The numerical FTF-1 map, tabular inputs and \texttt{acs2018-ftf1-p3-v1} population, ordered identities, requests and seeds are unchanged.

\emph{Implementation repair behind the P2/P3 rerun.} Pickle-restored arrays can carry equivalent but different descriptor objects. The shared statistic accumulator in \texttt{FTF-1D-descriptor-v1} canonicalizes valid native-int64 array descriptors after the original validation and overflow checks. Six predefined same-buffer checks on this runtime reproduce a roughly five- to six-fold change in the feature-sum reduction, reversible by restoring the alternate descriptor, without changing values, layout or precision. This is a local reduction measurement, not a full-algorithm speed factor or attribution of historical timings. All four methods use the same corrected path within existing timers; sampling, updates, certificates and fallback are unchanged. Historical timings remain a separate source/environment record.

Old/new process and full-grid checkpoint/method comparisons preserve exact numerical outputs, including P3 initializers and accounting. P1/P4 were not rerun; current numerical checks remain distinct from preserved historical stress and diagnostic evidence.

The current small-instance validation records 213 passing tests, including 17 descriptor regressions covering pickle restoration, strided/empty inputs and rejection semantics. Fourteen cases use six processes each: checkpoint creation, an integer/rational oracle that does not reuse production arithmetic, fresh training, and three replays (84 executions). Cases cover certificate underflow/cancellation, nearest-center ties, represented-candidate objectives, one represented-$X$ path, rational sampling/zero potential, empty slices/cells, saturation, invalid requests, safe integer bounds, and $(b,L)=(12,6),(16,20)$. Preserved earlier reports cover 33 exact-Fraction distance-filter stress tests through dimension 4096 and compare the compressed P1 evaluator with canonical $D^2$-sampling on 96 cases. These are bounded checks, not an all-input proof.

Each P2/P3 trial creates a clean checkpoint and runs fresh training and three replays in separate processes. All indexed center coordinates at every round and raw integer $N,S,SS$ match without tolerance or permutation: 525 P2 and 285 P3 comparisons. They share production kernels, so they are not independent arithmetic-oracle validations. P3's separate observer/worker validation uses 35 process executions, checks counters against independent small fixtures, and detects corrupted counters. The ideal categorical-law proof in \cref{app:implementation} is separate from empirical PRNG frequencies.

Panel counts use method/seed outcomes for P1, trial/method outputs for P2/P3, and settings for P4; shared-seed settings are not independent. The evaluated map is binary-group and simulated locally.

\subsection{Common output, timing, and work contracts}
\label{app:timing}
Fresh and replay return comparable model/witness schemas without a reusable point cache; checkpoint construction is charged separately. Direct replay reuses untouched summaries but assigns every survivor each round. Certificate replay follows \cref{alg:unlearn}, abandoning at a strictly greater than $1/2$ failed fraction. Direct makes no certificate attempts or abandonment transition.

The paired elapsed ratio is $r_{i,m}=t_{i,\rm fresh}/t_{i,m}$ (above one favors replay). Its arithmetic mean differs from the total-time ratio $\sum_i t_{i,\rm fresh}/\sum_i t_{i,m}$ and pooled relative reduction
\begin{equation}
100\frac{\sum_i(t_{i,\rm fresh}-t_{i,m})}{\sum_i t_{i,\rm fresh}}.
\label{eq:pooled-reduction}
\end{equation}
The latter is not $100(\overline r_m-1)$. Summaries and uncertainty describe this trial grid, not a deployment population.

Let $N=(n-r)T$ count eligible survivor point-rounds, $A$ attempted tests, $P$ passes, $S$ passes retained after same-round fallback, and $J$ failed assignments repeated in a full rebuild. Work is $N-S+J$; discarded passes are $P-S$. Acceptance $P/A$ and useful coverage $S/N$ have different denominators, with zero denominators reported as N/A. Per-trial means and pooled ratios, point-units and distance pairs, exact evaluations and seconds/bytes are distinct.

Algorithm elapsed time is local serial wall time. Modeled compute combines server time with per-round maximum client time, excluding network latency. Comparable end-to-end (E2E) adds recorded input/materialization, checkpoint loading where needed, and equal-schema witness output. Nested primitive timers cannot be summed. Offline quality, diagnostics, verification, and initial checkpoint construction are excluded from deletion-call algorithm/E2E timings. Warm OS caches and nondedicated desktop hardware limit latency interpretations. The new P2/P3 runs are sequential, with no concurrent experiment jobs, profiling, heavy validation or rendering; light document preparation and ordinary desktop activity remain possible. No network, secure aggregation, retries, synchronization, stragglers, or purge service is timed.

\subsection{P2/P4 cohort provenance and execution definitions}
\label{app:p124-protocol}
The executed \texttt{p124-unique-source-v2} arrays, ordered memberships, and requests are frozen by hash. Bank reconstruction matches every processed coordinate, mask, and all 20 clients. For Adult/Credit, supplied code documents intended recipes, but original raw-file hashes, fitted transforms, expanded Adult vocabulary, Credit worksheet order, and preprocessing dependency versions are unavailable. These are evaluations on archived processed inputs; recipe names do not establish raw-population or semantic-feature identity. The gaps do not invalidate observed replay comparisons or justify dropping trials, and processed-input reproduction needs none of the missing raw preprocessing state.

The supplied recipes specify the following geometry; only Bank's raw reconstruction is verified. Full recipes remain in the artifact. In Adult, the stored literal male mask maps to B and its complement to A; rows with missing fields and sex/fnlwgt columns are removed, income remains, remaining fields are dummy-encoded and standardized on the eligible population with population variance, without PCA. Bank uses married versus its complement, removes target y but retains marital status, dummy-encodes with \texttt{drop\_first=True}, and uses sample-standard-deviation scaling. Credit uses EDUCATION in $\{1,2\}$ versus other codes, removes ID/target, and retains numeric codes including EDUCATION and SEX, with the same sample scaling. Bank/Credit apply no row exclusion or imputation. The frozen populations/dimensions are in \cref{tab:corrected-cohorts}; the artifact preserves the full source recipes and unresolved ordered-feature identities. Categorical codes are not learned semantic distances.

\paragraph{Clients and representation.}
The supplied split initializes \texttt{default\_rng(42)}, permutes true-mask then false-mask indices, gives each of the first $C-1$ clients each group's floor share, and gives the last client both remainders. The continuing generator permutes each concatenated client list. Corrected loading verifies this order against stored coordinates/masks and selects distinct source keys once: 100 Adult, 20 Bank, 100 Credit. Seed 0 selects all clients, sorted before dense reindexing. Identity is pickle hash plus full-array row, not a native identifier; coincident vectors remain distinct. Fixed transforms precede encoding with $b=12$ and clipping bounds 175, 121, 74, respectively; the same represented $X$ serves initialization, refinement, and quality.

\paragraph{P2 requests and seeds.}
For dataset, $k$, and trial $i=0,\ldots,4$, training seed is $s=1000k+i$ and request seed $s+500000$. A PCG64 generator samples uniformly without replacement from client-sorted local rows, with no group/client quota, using $r=\max(20,\operatorname{round}(0.001n))$: 30 Adult, 45 Bank, 30 Credit. Each independent request starts from the original population, is sorted within client, and is shared by all methods. Since $r$ is below both group sizes, the first draw is valid. Request validation and role/entity-keyed tapes follow \cref{app:implementation}; empty local clients/slices are permitted.

\paragraph{P4 and baselines.}
P4 uses the full cohorts without deletion. MERGED-2K makes one sequential locally weighted call per client with budget $\min(2k,n_c)$ and weights $1/n_{c,g}$, weighted first and $D^2$ draws, then assigns unsnapped points to nearest indexed representatives and sends integer group counts. Server weights are $h_{c,j,A}/n_A+h_{c,j,B}/n_B$. Multiplicity precedes snapping; slot order is client then selection index. Zero potential completes unused indices in order; saturation returns ordered input without draws. The centralized reference uses pooled weights $1/n_g$. Every initializer receives the same 15 guarded rounds, with warm/final quality recorded and no local or anchor refinement ($T_a=0$).

\paragraph{Execution and analysis.}
The new P2 run uses one local worker on Apple M3 Max with 48\,GiB memory, arm64 macOS 27.0, Python 3.12.14 and NumPy 2.3.5; five numerical-library thread settings are fixed to one. P4 retains its earlier macOS 26.6.2 execution; its exact CPU model was not recorded. P2 rotates fresh/direct/basic/runner-up by $s\bmod4$ in separate processes after checkpoint creation; garbage collection precedes each timed call. P4 follows the frozen setting order; timing boundaries are in \cref{app:timing}.
For \texttt{summarize\_corrected\_p2.py}, PCG64 seed 978531 first draws 5,000 IID row replicates, then the reported 5,000 fixed-$k$ seed-block replicates from the continuing stream. Blocks retain shared seeds across datasets and are independently resampled at each $k$; intervals use linearly interpolated 2.5th/97.5th percentiles. The artifact index identifies both implementation versions, reporting code and raw-source gaps.

\subsection{P1: architecture-specific heterogeneity}
\label{app:p1-evidence}
The construction in \cref{app:kappa-proof} uses $\kappa\in\{1,2,4,8,16,32,64,128\}$, $\epsilon=1/(\kappa+1)$, integral group counts, scale 10, $k=1$, and $\gamma=T_a=T=0$, with 10,000 common seed indices per method/value. Merged-2k matches anchor slots, not bytes; this two-location experiment does not evaluate the analytical three-location result of \cref{prop:s2-matched-budget}. Centralized weighted $D^2$-sampling is a quality reference. Empirical PRNG moments and ideal-bit analytical expectations are distinct; a finite mean below an expectation bound is no contradiction. All 32 summary rows remain in \texttt{p1\_summary.csv}, with prior receipts for the seed-level archive.

\begin{table}[!htbp]
\centering\footnotesize
\caption{P1 at $\kappa=128$, 10,000 seed indices per method. SE describes across-seed variation under the declared deterministic PRNG protocol, with the ideal-bit expectation a separate theorem.}
\label{tab:corrected-p1}
\begin{tabular}{lrrr}
\toprule
Method & Mean $G/G^*$ & SE & Mean $\Phi/\Phi^*$ \\
\midrule
Merged-1k & 63.5310 & 1.1028 & 32.3935 \\
Merged-2k & 1.9034 & 0.1518 & 1.4593 \\
SplitGroup & 2.2375 & 0.1776 & 1.6270 \\
Central & 1.7749 & 0.1406 & 1.3948 \\
\bottomrule
\end{tabular}
\end{table}

\begin{figure}[!htbp]
\centering\includegraphics[width=\linewidth]{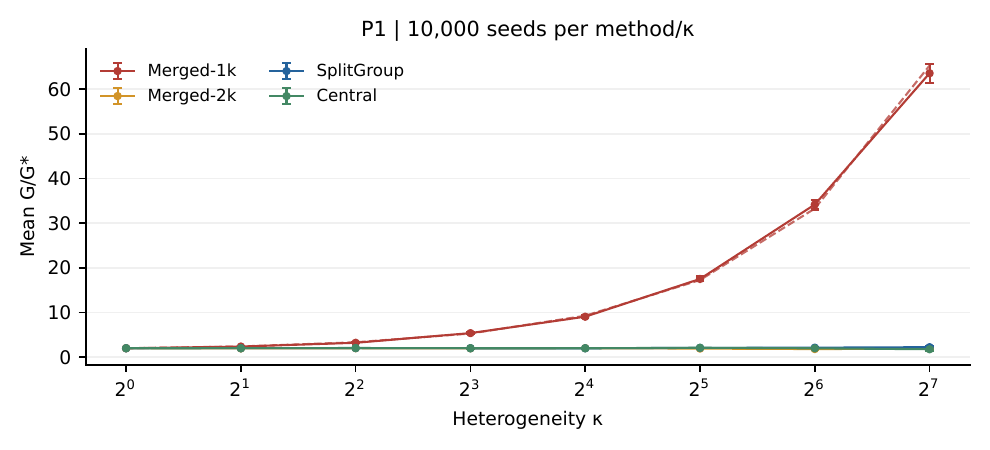}
\caption{P1: does locally balanced merging lose quality as the local/global mismatch $\kappa$ grows? Initialization quality on the lower-bound construction of \cref{app:kappa-proof}, with 10,000 seed indices per method and heterogeneity value. Dashed curves are ideal analytical expectations; error bars are descriptive normal intervals from across-seed standard errors. The increasing \Merged{} cost concerns the specified two-location construction, not every merged summary.}
\label{fig:p1-current}
\end{figure}

\subsection{P2: unique-source tabular cohorts}
\label{app:p2-evidence}
Adult, Bank Marketing, and Credit Default \citep{uci1996adult,moro2014bank,yeh2009credit} use fixed processed arrays with dimensions 103, 42, and 23. The unique-source manifest selects each physical source client once and preserves membership. \Cref{tab:corrected-cohorts} gives group counts and count-derived $\kappa$; no client is pure or at least 95\% single-group. Provenance, group definitions, encoding, and requests are specified once in \cref{app:p124-protocol}.

\begin{table}[!htbp]
\centering
\small
\caption{Source-cohort manifest, p124-unique-source-v2. All listed source clients are unique; source hashes match the audit and the full-data partition is byte-verified. No client is pure or at least 95\% single-group.}
\label{tab:corrected-cohorts}
\begin{tabular}{lrrrrrr}
\toprule
Dataset & Records & Clients & $d$ & Group 0 & Group 1 & $\kappa$ \\
\midrule
Adult & 30162 & 100 & 103 & 9782 & 20380 & 1.317781 \\
Bank & 45211 & 20 & 42 & 17997 & 27214 & 1.008094 \\
Credit & 30000 & 100 & 23 & 5385 & 24615 & 2.416867 \\
\bottomrule
\end{tabular}
\end{table}

Adult/Credit use $k\in\{4,5,\ldots,16\}$; Bank uses $k\in\{2,3,4,5,10,20,30,40,50\}$. Five independent requests per setting give 65, 65, and 45 trials, all with $T=15,b=12,L=6,T_a=\gamma=0$ and abandonment threshold 0.5. Shared seed indices across datasets are preserved in the uncertainty calculation.

\Cref{tab:corrected-p2,tab:corrected-p2-counters} report performance and assignment work; \cref{tab:corrected-p2-detail} gives timing-boundary detail. The 5,000 fixed-$k$ seed-block bootstrap replicates preserve shared seeds across datasets, rather than treating all rows as IID (\cref{app:p124-protocol}).

\begin{table}[!htbp]
\centering\footnotesize
\caption{P2, all 175 trial units. Ratios use output-matched serial elapsed time. Intervals are 5,000 fixed-$k$ seed-block bootstrap replicates, preserving shared seeds across datasets. Every replay trial saves exact indexed trajectories and raw statistics.}
\label{tab:corrected-p2}
\begin{tabular}{lrrrrr}
\toprule
Method & Mean & Mean 95\% CI & Median & Wins/fresh & Abandons \\
\midrule
Direct & 1.066 & [1.065, 1.068] & 1.067 & 173/175 & N/A \\
Basic & 1.025 & [1.004, 1.062] & 1.008 & 107/175 & 174/175 \\
Runner-up & 1.024 & [1.004, 1.061] & 1.007 & 108/175 & 174/175 \\
\bottomrule
\end{tabular}
\end{table}

\begin{table}[!htbp]
\centering\footnotesize
\caption{P2 cost detail, ratios of fresh/method time. End-to-end includes recorded input/checkpoint/materialization/witness-output cost; modeled time is compute only.}
\label{tab:corrected-p2-detail}
\begin{tabular}{lrrr}
\toprule
Method & Elapsed range & Model mean & E2E mean \\
\midrule
Direct & 0.960--1.138 & 1.002 & 1.052 \\
Basic & 0.938--4.222 & 0.985 & 1.010 \\
Runner-up & 0.934--4.187 & 0.986 & 1.009 \\
\bottomrule
\end{tabular}
\end{table}

\begin{table}[!htbp]
\centering\footnotesize
\caption{P2 point-round counters pooled over 175 trials. $N,A,P,S,J$ are defined in \cref{app:timing}; $P/A$ includes discarded passes, while useful coverage is $S/N$. Direct replay attempts no certificates.}
\label{tab:corrected-p2-counters}
\begin{tabular}{lrrrrrr}
\toprule
Method & $N$ & $A$ & $P$ & $S$ & $J$ & $P/A$ \\
\midrule
Direct & 89086500 & 0 & 0 & 0 & 0 & N/A \\
Basic & 89086500 & 6571424 & 677014 & 677014 & 5893934 & 10.3024\% \\
Runner-up & 89086500 & 6571424 & 677018 & 677014 & 5893930 & 10.3025\% \\
\bottomrule
\end{tabular}
\end{table}

Pooled algorithm-time reductions for Direct/Basic/Runner-up are $5.730\%,0.567\%,0.492\%$. Their E2E total-time ratios are $1.049,0.996,0.995$, with wins in $170,69,70$ of 175 trials, respectively; there are no ties. Thus the certificates' small algorithm-time gains do not extend to total E2E time. Both certificates abandon at round zero in 174 trials and never abandon in one. Runner-up's four extra attempted passes are discarded, leaving the same useful $S$. Basic has $J=5{,}893{,}934$ and $N-S+J=94{,}303{,}420$; runner-up has $J=5{,}893{,}930$ and $N-S+J=94{,}303{,}416$. These four work units do not determine elapsed-time ordering.

\subsection{P3: public ACS geographic federation}
\label{app:p3-evidence}
\paragraph{Population and preprocessing.}
We use the 2018 Census 1-Year person files for the 50 states, excluding DC and Puerto Rico \citep{census2018pums}. The eligible pool has 3,207,990 records: state code must match the file, SEX must be 1 (group A) or 2 (group B), and PUMA must be finite, positive, and integral. There is no age, employment, or income threshold. We use the Folktables framework \citep{ding2021folktables}; release 0.0.12 supplies the ACSEmployment schema and identity prefilter \citep{folktablescode}; no prediction outcome is used.

Feature order is \texttt{AGEP}, \texttt{SCHL}, \texttt{MAR}, \texttt{RELP}, \texttt{DIS}, \texttt{ESP}, \texttt{CIT}, \texttt{MIG}, \texttt{MIL}, \texttt{ANC}, \texttt{NATIVITY}, \texttt{DEAR}, \texttt{DEYE}, \texttt{DREM}, \texttt{SEX}, \texttt{RAC1P}. Age and 15 categorical codes remain numeric, without one-hot encoding or embeddings. Missing values become $-1$; remaining infinities and invalid native identifiers are rejected. NumPy 2.3.5 computes the full-pool mean and population standard deviation once from contiguous binary64 arrays in FIPS/state-file row order (zero standard deviation becomes one). Subsampling/deletion retain this transform. Encoding uses clip 8 and $b=12$; the full-pool maximum absolute standardized coordinate is 6.102270. This reproducible geometry does not establish a social interpretation for distances between categorical codes.

\paragraph{Ownership, sampling, and execution.}
Clients are states for $C=10,25,50$ and state-qualified PUMAs for $C=100,250$; PUMAs are Census geographic areas \citep{censuspuma}. A keyed PCG64 permutation of sorted geographic keys is sampled by prefix. Largest-remainder quotas allocate records proportionally to eligible client size, breaking ties by prefix position; sampling is without replacement and records retain source order. Thus varying $C$ also changes geography/composition. Native identity is (2018, state FIPS, SERIALNO, SPORDER), bound to source hash and row. All 95 requests are independent, noncumulative uniform subsets with fixed derived seeds. Each $r$ is below both group sizes, so the first draw is valid without conditioning. These cohorts do not use FeDa4Fair's controlled-bias partitions \citep{heilmann2026feda}.

\begin{table}[!htbp]
\centering
\footnotesize
\caption{ACS populations. PUMA identities are state-qualified. Group counts, client-size ranges, and exact count-derived heterogeneity are shown. P/NP counts pure and at least 95\% single-group clients; these are population descriptors, not performance explanations.}
\label{tab:p3-populations}
\begin{tabular}{rrlrrrrr}
\toprule
$n$ & $C$ & Client & $n_A$ & $n_B$ & $\kappa$ & Client size & P/NP \\
\midrule
100,000 & 10 & State & 48,728 & 51,272 & 1.105216 & 1446--41266 & 0/0 \\
100,000 & 25 & State & 49,022 & 50,978 & 1.127072 & 392--23054 & 0/0 \\
100,000 & 50 & State & 49,033 & 50,967 & 1.229392 & 179--11808 & 0/0 \\
100,000 & 100 & PUMA & 48,925 & 51,075 & 1.374730 & 383--2254 & 0/0 \\
100,000 & 250 & PUMA & 49,293 & 50,707 & 1.523855 & 150--884 & 0/0 \\
250,000 & 50 & State & 122,606 & 127,394 & 1.183868 & 447--29521 & 0/0 \\
500,000 & 50 & State & 244,723 & 255,277 & 1.194489 & 895--59043 & 0/0 \\
1,000,000 & 50 & State & 490,130 & 509,870 & 1.151322 & 1789--118085 & 0/0 \\
\bottomrule
\end{tabular}
\end{table}

\begin{table}[!htbp]
\centering
\footnotesize
\caption{Frozen P3 one-factor design. Every row uses seed indices 0--4 and all four methods; $d=16,b=12,L=6,T_a=0,\gamma=0$, abandonment threshold $0.5$. Designs 15 and 16 are the prespecified pilots. Repeated defaults across one-factor sweeps form a union of 19 configurations, not a factorial experiment.}
\label{tab:p3-designs}
\begin{tabular}{rrrrrrr}
\toprule
ID & $n$ & $C$ & $k$ & $T$ & $r$ & $r/n$ \\
\midrule
0 & 100,000 & 50 & 4 & 5 & 100 & 0.001 \\
1 & 250,000 & 50 & 4 & 5 & 250 & 0.001 \\
2 & 500,000 & 50 & 4 & 5 & 500 & 0.001 \\
3 & 1,000,000 & 50 & 4 & 5 & 1,000 & 0.001 \\
4 & 100,000 & 10 & 4 & 5 & 100 & 0.001 \\
5 & 100,000 & 25 & 4 & 5 & 100 & 0.001 \\
6 & 100,000 & 100 & 4 & 5 & 100 & 0.001 \\
7 & 100,000 & 250 & 4 & 5 & 100 & 0.001 \\
8 & 100,000 & 50 & 8 & 5 & 100 & 0.001 \\
9 & 100,000 & 50 & 16 & 5 & 100 & 0.001 \\
10 & 100,000 & 50 & 32 & 5 & 100 & 0.001 \\
11 & 100,000 & 50 & 64 & 5 & 100 & 0.001 \\
12 & 100,000 & 50 & 4 & 10 & 100 & 0.001 \\
13 & 100,000 & 50 & 4 & 20 & 100 & 0.001 \\
14 & 100,000 & 50 & 4 & 40 & 100 & 0.001 \\
15 & 100,000 & 50 & 4 & 5 & 1 & 0.00001 \\
16 & 100,000 & 50 & 4 & 5 & 10 & 0.0001 \\
17 & 100,000 & 50 & 4 & 5 & 1,000 & 0.01 \\
18 & 100,000 & 50 & 4 & 5 & 5,000 & 0.05 \\
\bottomrule
\end{tabular}
\end{table}

Methods run in separate processes with numerical-library thread limits of one and four-method order rotated by seed modulo four. The environment is Python 3.12.14/NumPy 2.3.5, macOS 27.0 arm64, Apple M3 Max with 48\,GiB unified memory. This is not a networked deployment. Designs 15/16 are the ten pilot trials, included once among the 95 trials.

\paragraph{Checkpoint and timing costs.}
\label{app:checkpoint-costs}
Across 95 checkpoints, mean construction time is 2.4395\,s and mean serialization is 0.0271\,s; mean/maximum checkpoint size is 103,342,382/504,266,791 bytes. These costs are separate from deletion-call timing. Mean insertion time, 0.0002\,s, measures dictionary insertion of already computed arrays; bound construction is charged inside checkpoint assignments. Model-only output with zero insertion does not avoid bound computation.

For fresh/direct/basic/runner-up, mean algorithm seconds are 2.4261/2.2982/2.5861/2.5823 and mean comparable E2E seconds are 2.4624/2.3598/2.6480/2.6440. Primitive timers nest and include observation wrappers; they do not isolate a causal explanation of runtime. Process RSS and evidence-output bytes are not cache-only storage or network traffic. Complete component and resource records remain in the indexed artifact tables.

\begin{table}[!htbp]
\centering
\footnotesize
\caption{Additional descriptive P3 suite summaries. End-to-end (E2E) means and total-time ratios are distinct. Losses are paired elapsed E2E losses against fresh. Acceptance percentages exclude undefined trial ratios from the mean; Def. gives the count. No IID interval is attached to this varied suite.}
\label{tab:p3-suite-extra}
\begin{tabular}{lrrrrrr}
\toprule
Method & Mean E2E & Total E2E & Losses & Mean $P/A$ & Pool $P/A$ & Def. \\
\midrule
Fresh & 1.000 & 1.000 & 0/95 & N/A & N/A & 0/95 \\
Direct & 1.078 & 1.043 & 23/95 & N/A & N/A & 0/95 \\
Basic & 1.252 & 0.930 & 78/95 & 3.16 & 8.38 & 95/95 \\
Runner-up & 1.228 & 0.931 & 79/95 & 3.16 & 8.38 & 95/95 \\
\bottomrule
\end{tabular}
\end{table}

\FloatBarrier
\paragraph{Variation and uncertainty.}
\Cref{tab:p3-algorithm-timing} reports all 19 designs' new algorithm ratios and intervals; \cref{tab:pilot-compact} retains all ten small-deletion seeds' algorithm and E2E ratios, including every loss. The artifact contains per-design E2E intervals and duplicate seed plots. Within each fixed design, 5,000 paired-seed percentile bootstrap draws use seed 20260906 and linear interpolation. No IID interval is attached to the varied suite; five-seed intervals do not establish reliability, causal parameter effects, or a universal deletion threshold.

All five acceptance ratios are defined in every design. For both certificates, designs other than 15 have zero mean/pooled acceptance and useful coverage, with all five trials abandoning. Design 15 has 60.00\% mean acceptance and mean/pooled useful coverage, but 88.23\% pooled acceptance; two trials abandon. Over the suite, each certificate abandons at round zero in 92/95 trials, later in none, and never in three. Transitions occur in 92/107 attempted and 92/750 total rounds, followed by 643 direct rounds; no passed assignments are discarded. 

\begin{table}[!htbp]
\centering
\footnotesize
\caption{P3 algorithm elapsed ratios Fresh/method for all 19 designs. Each cell's first line is the mean paired ratio [95\% percentile interval]; its second line is the ratio of total times [interval]. The paired bootstrap resamples five seed indices within each fixed design, 5,000 draws, seed 20260906. Intervals are descriptive small-sample uncertainty, not deployment reliability. Ratios above one favor replay.}
\label{tab:p3-algorithm-timing}
\begin{tabular}{rccc}
\toprule
ID & Direct & Basic & Runner-up \\
\midrule
0 & \shortstack{1.079 [1.072, 1.086]\\1.079 [1.071, 1.086]} & \shortstack{0.918 [0.913, 0.923]\\0.918 [0.913, 0.923]} & \shortstack{0.919 [0.913, 0.925]\\0.919 [0.913, 0.925]} \\
1 & \shortstack{1.026 [1.016, 1.035]\\1.026 [1.016, 1.034]} & \shortstack{0.882 [0.872, 0.890]\\0.882 [0.872, 0.890]} & \shortstack{0.876 [0.866, 0.884]\\0.876 [0.866, 0.884]} \\
2 & \shortstack{1.020 [1.007, 1.035]\\1.020 [1.007, 1.035]} & \shortstack{0.878 [0.867, 0.889]\\0.878 [0.867, 0.889]} & \shortstack{0.872 [0.861, 0.887]\\0.872 [0.861, 0.887]} \\
3 & \shortstack{1.011 [0.996, 1.029]\\1.011 [0.996, 1.029]} & \shortstack{0.871 [0.865, 0.877]\\0.871 [0.865, 0.877]} & \shortstack{0.872 [0.861, 0.882]\\0.872 [0.861, 0.882]} \\
4 & \shortstack{0.999 [0.984, 1.013]\\0.999 [0.984, 1.013]} & \shortstack{0.859 [0.844, 0.877]\\0.859 [0.844, 0.877]} & \shortstack{0.864 [0.853, 0.873]\\0.864 [0.853, 0.873]} \\
5 & \shortstack{1.036 [1.025, 1.049]\\1.036 [1.025, 1.049]} & \shortstack{0.898 [0.879, 0.915]\\0.898 [0.879, 0.915]} & \shortstack{0.874 [0.853, 0.892]\\0.874 [0.853, 0.892]} \\
6 & \shortstack{1.204 [1.197, 1.212]\\1.204 [1.197, 1.212]} & \shortstack{1.012 [1.002, 1.024]\\1.012 [1.002, 1.024]} & \shortstack{0.998 [0.995, 1.001]\\0.998 [0.995, 1.001]} \\
7 & \shortstack{1.300 [1.276, 1.326]\\1.300 [1.276, 1.326]} & \shortstack{1.049 [1.016, 1.075]\\1.048 [1.014, 1.075]} & \shortstack{1.057 [1.047, 1.067]\\1.057 [1.047, 1.067]} \\
8 & \shortstack{1.092 [1.076, 1.108]\\1.091 [1.076, 1.108]} & \shortstack{0.944 [0.933, 0.955]\\0.944 [0.933, 0.955]} & \shortstack{0.948 [0.932, 0.964]\\0.948 [0.932, 0.964]} \\
9 & \shortstack{1.125 [1.107, 1.144]\\1.125 [1.107, 1.144]} & \shortstack{0.975 [0.955, 0.991]\\0.974 [0.954, 0.991]} & \shortstack{0.977 [0.962, 0.992]\\0.977 [0.962, 0.992]} \\
10 & \shortstack{1.104 [1.075, 1.128]\\1.103 [1.074, 1.128]} & \shortstack{0.977 [0.964, 0.991]\\0.977 [0.964, 0.991]} & \shortstack{0.972 [0.941, 0.993]\\0.971 [0.940, 0.993]} \\
11 & \shortstack{1.113 [1.104, 1.121]\\1.113 [1.104, 1.121]} & \shortstack{0.985 [0.971, 0.998]\\0.984 [0.971, 0.998]} & \shortstack{0.988 [0.975, 0.999]\\0.988 [0.975, 0.999]} \\
12 & \shortstack{1.040 [1.029, 1.057]\\1.039 [1.029, 1.057]} & \shortstack{0.945 [0.940, 0.951]\\0.945 [0.940, 0.951]} & \shortstack{0.953 [0.945, 0.962]\\0.953 [0.945, 0.962]} \\
13 & \shortstack{1.023 [1.006, 1.035]\\1.022 [1.006, 1.035]} & \shortstack{0.970 [0.952, 0.991]\\0.970 [0.952, 0.991]} & \shortstack{0.982 [0.962, 1.001]\\0.982 [0.962, 1.001]} \\
14 & \shortstack{1.008 [0.990, 1.023]\\1.007 [0.989, 1.023]} & \shortstack{0.952 [0.927, 0.970]\\0.952 [0.926, 0.970]} & \shortstack{0.966 [0.951, 0.980]\\0.966 [0.951, 0.980]} \\
15 & \shortstack{1.368 [1.350, 1.384]\\1.367 [1.350, 1.384]} & \shortstack{10.205 [3.806, 16.604]\\2.581 [1.404, 16.525]} & \shortstack{9.208 [3.519, 14.899]\\2.539 [1.394, 14.834]} \\
16 & \shortstack{1.276 [1.252, 1.314]\\1.276 [1.252, 1.314]} & \shortstack{1.073 [1.055, 1.103]\\1.073 [1.055, 1.103]} & \shortstack{1.066 [1.041, 1.101]\\1.065 [1.041, 1.101]} \\
17 & \shortstack{0.991 [0.977, 1.000]\\0.990 [0.977, 1.000]} & \shortstack{0.856 [0.853, 0.860]\\0.856 [0.853, 0.860]} & \shortstack{0.858 [0.847, 0.869]\\0.857 [0.847, 0.869]} \\
18 & \shortstack{0.997 [0.984, 1.008]\\0.997 [0.984, 1.008]} & \shortstack{0.859 [0.853, 0.868]\\0.859 [0.853, 0.868]} & \shortstack{0.851 [0.847, 0.858]\\0.851 [0.847, 0.858]} \\
\bottomrule
\end{tabular}
\end{table}

\begin{table}[!htbp]
\centering\footnotesize
\caption{When does certification pay off? All small-deletion trials, as algorithm / E2E paired fresh-to-method ratios. Certification reaches large ratios only in the three trials that never fall back; in the other seven it beats fresh training here but trails direct replay. Values above one favor replay. Designs 15 and 16 have $r=1$ and $r=10$; other settings are in \cref{tab:p3-designs}. Fallback is the first zero-based round, identical for both certificates. Direct has no fallback. Both certificates' attempted acceptance and useful coverage round to 100\% in the three never-fallback trials and are zero in the other seven; exact unsaved counts for the successes appear in \cref{tab:s2-p3-diagnostics}. Full per-method seconds and unrounded records remain in the artifact.}
\label{tab:pilot-compact}
\begin{tabular}{rllll}
\toprule
Design/seed & Direct & Basic & Runner-up & Fallback \\
\midrule
15/0 & 1.374 / 1.342 & 17.331 / 11.488 & 15.544 / 10.731 & Never \\
15/1 & 1.355 / 1.324 & 16.959 / 11.377 & 15.191 / 10.591 & Never \\
15/2 & 1.386 / 1.353 & 1.145 / 1.126 & 1.144 / 1.125 & 0 \\
15/3 & 1.386 / 1.353 & 1.151 / 1.130 & 1.138 / 1.119 & 0 \\
15/4 & 1.337 / 1.308 & 14.438 / 10.241 & 13.026 / 9.504 & Never \\
16/0 & 1.258 / 1.232 & 1.070 / 1.055 & 1.065 / 1.049 & 0 \\
16/1 & 1.269 / 1.242 & 1.058 / 1.043 & 1.054 / 1.040 & 0 \\
16/2 & 1.243 / 1.217 & 1.055 / 1.041 & 1.046 / 1.032 & 0 \\
16/3 & 1.265 / 1.238 & 1.052 / 1.037 & 1.031 / 1.017 & 0 \\
16/4 & 1.347 / 1.318 & 1.130 / 1.113 & 1.131 / 1.114 & 0 \\
\bottomrule
\end{tabular}
\end{table}

\begin{table}[!htbp]
\centering
\footnotesize
\caption{P3 raw point-round totals over the 95-trial suite. $N$ counts eligible assignments, $A$ attempts, $P$ attempted passes, $S$ effective saved assignments, and $J$ duplicate assignments. Non-certifying methods have no attempts and hence undefined $P/A$, not zero acceptance.}
\label{tab:p3-suite-counts}
\begin{tabular}{lrrrrr}
\toprule
Method & $N$ & $A$ & $P$ & $S$ & $J$ \\
\midrule
Fresh & 110,998,475 & 0 & 0 & 0 & 0 \\
Direct & 110,998,475 & 0 & 0 & 0 & 0 \\
Basic & 110,998,475 & 17,905,183 & 1,499,968 & 1,499,968 & 16,405,198 \\
Runner-up & 110,998,475 & 17,905,183 & 1,499,968 & 1,499,968 & 16,405,198 \\
\bottomrule
\end{tabular}
\end{table}

Pooled useful coverage is $1.351\%$. Each certificate variant repeats 16,405,198 assignments in fallback, for $N-S+J=125,903,705$ assignment point-units, versus 110,998,475 for fresh training and Direct. These work counters do not isolate elapsed-time effects.

\FloatBarrier
\subsubsection{Seven prescribed mechanism diagnostics}
\label{app:p3-mechanisms}
The subset is designs 15/seeds 0--4 ($r=1$), 16/seed 0 ($r=10$), and 0/seed 0 ($r=100$), all with $n=100000,C=50,k=4,T=5,b=12,L=6,\gamma=T_a=0$. It retains prescribed successes and failures. These preserved FTF-1 original-checkpoint versus retained-training comparisons measure stability, separately from replay-versus-fresh equality. Their arithmetic/trace analysis was not rerun; exact old/new initializer, trajectory and accounting agreement supports retaining their numerical conclusions.

\begin{table}[htbp]
\centering
\begin{tabular}{lrrrrrr}
\toprule
Design/seed & $r$ & Changed labels/$N_0$ & $\Delta_{\rm idx}$ & $\Delta_{\rm match}$ & $F_0$ & Server draw\\
\midrule
15/0 & 1 & 0/99999 & 0.000 & 0.000 & 0 & none \\
15/1 & 1 & 0/99999 & 0.000 & 0.000 & 0 & none \\
15/2 & 1 & 43870/99999 & 4.302 & 3.867 & 99999 & 0 \\
15/3 & 1 & 31280/99999 & 6.471 & 6.264 & 99999 & 2 \\
15/4 & 1 & 0/99999 & 0.000 & 0.000 & 0 & none \\
16/0 & 10 & 60120/99990 & 8.880 & 7.292 & 99990 & 0 \\
0/0 & 100 & 67071/99900 & 9.502 & 7.032 & 99900 & 0 \\
\bottomrule
\end{tabular}

\medskip
\begin{tabular}{lrrrrrr}
\toprule
Design/seed & $N$ & $A$ & $P$ & $S/N$ & $J$ & Fallback\\
\midrule
15/0 & 499995 & 499995 & 499992 & 499992/499995 & 0 & none \\
15/1 & 499995 & 499995 & 499989 & 499989/499995 & 0 & none \\
15/2 & 499995 & 99999 & 0 & 0/499995 & 99999 & 0 \\
15/3 & 499995 & 99999 & 0 & 0/499995 & 99999 & 0 \\
15/4 & 499995 & 499995 & 499987 & 499987/499995 & 0 & none \\
16/0 & 499950 & 99990 & 0 & 0/499950 & 99990 & 0 \\
0/0 & 499500 & 99900 & 0 & 0/499500 & 99900 & 0 \\
\bottomrule
\end{tabular}
\caption{Seven prescribed ACS diagnostic cases: why did certification fail or succeed? The top panel is round zero:
$N_0=100000-r$, indexed label changes, maximum indexed movement and diagnostic
bottleneck movement (rounded to three decimals), certificate failures $F_0$,
and first server draw with changed selected coordinates. Both implemented
and independently resolved ideal Basic/Runner-up failure counts equal the
displayed $F_0$; every numerical-only and unresolved count is zero.
Matching never defines exactness. The lower panel reuses full five-round
accounting, $N=5N_0$: attempts $A$, passes $P$, effective coverage $S/N$,
duplicated assignments $J$, and first fallback round. Counts coincide for
the two certificate methods in these seven cases. Draw/round indices start
at zero; ``none'' means no such event. The three successful trials leave
$3+6+8=17$ point-rounds unsaved despite percentages that round to $100\%$.}
\label{tab:s2-p3-diagnostics}
\end{table}

\FloatBarrier

\paragraph{Geometry and numerical decisions.}
In four abandoned cases, positive maximum matched-center movement rules out a mere index permutation. Between 31,280 and 67,071 retained indexed labels change, while both ideal certificates fail on every survivor at round zero: unchanged labels need not be certifiable. Two separately written arithmetic implementations resolve all 699,885 retained-point tests; implemented and ideal decisions coincide with no numerical-only or unresolved failures. This covers these seven round-zero comparisons, not all trials, later rounds, or arbitrary inputs.

\paragraph{Distinct sampling mechanisms.}
The first server proposal width changes from 33 to 30 bits in design 15/seed 2, and 33 to 29 in design 16/seed 0. The same word is accepted without rejection at the same offset; changed selection precedes consumption divergence. In design 15/seed 3, draws 0/1 agree; draw 2 uses the same 61-bit width, accepted proposal, and offsets, but changed cumulative intervals select different coordinates. Consumption never diverges across its four server draws. Design 0/seed 0 has first-call geometric TV approximately $0.641473$, not a tiny geometric-law change. These conditional laws use observed prefixes and unions of persistent identities or coordinates, not slot numbers; they are not TVs of the full adaptive initialization or refined trajectory.

\paragraph{Successful reuse and diagnostic scope.}
The three successful $r=1$ seeds have changed local summaries but identical round-zero server initializers. Trajectories diverge after the first update, leaving 3, 6, and 8 point-rounds unsaved over five rounds. Coverage rounds to $100\%$, unlike the exact counts in \cref{tab:s2-p3-diagnostics}. The four abandoned cases have no round-zero passes to discard, unlike the group-B passes in \cref{prop:s2-deletion-work}. Full reconstruction is the policy, not a necessity of exact replay.

Preserved trace checks cover 600 calls in 75 paired streams (68 local, seven server): frozen prefixes, unchanged-kernel endpoints, rational decisions and ideal inequalities agree. These are same-agent checks, not human review. Matching diagnoses geometry, not exactness; no intervention isolates causal effects. Diagnostic times are excluded from performance tables, and nested timers cannot establish an unmeasured alternative's speed.

\subsection{P4: quality, residuals, collisions, and serialized bytes}
\label{app:p4-evidence}
The same unique-source cohorts use three federated initializers at $\gamma\in\{0,.01,.05,.1,.5,1\}$ and one centralized reference per dataset: 57 settings, not replications. All use $k=8$, seed 42, $T=15,L=6,b=12,T_a=0$. Quality uses the same represented $X$ throughout; the centralized reference is not a federated runtime baseline.

For $m$ slots, unweighted residual is $m^{-1}\sum_i\|q_i-z_i\|^2$; the proof instead uses $Q=\sum_i(h_i/n_{g_i})\|q_i-z_i\|^2$. The theorem gives an envelope, not exact quadratic scaling with $\gamma$ (see represented rounding in \cref{eq:represented-quantization}). Duplicate IDs are slots minus unique represented locations; baseline duplicates and new snapping collisions are separate, and colliding-weight and duplicate-slot fractions can differ.
\Cref{fig:p4-current} includes all settings; \cref{tab:corrected-p4-diagnostics} gives weighted distortion, duplicates, and bytes. Exact endpoint values remain in the indexed artifact table.

\begin{table}[!htbp]
\centering\footnotesize
\caption{P4 at $\gamma=1$. $Q$ uses actual represented anchors with global group-normalized weights. Base IDs are duplicates before snapping; extra IDs show additional collisions over all anchor slots. Packet bytes compare $\gamma=0/1$ and include the implemented metadata; they are not network traffic.}
\label{tab:corrected-p4-diagnostics}
\begin{tabular}{llrrrr}
\toprule
Dataset & Method & $Q$ & Base IDs & Extra IDs / slots & Packet bytes $0/1$ \\
\midrule
Adult & Merged-1k & 8.81718 & 5 & 26/800 & 697000/697000 \\
Adult & Merged-2k & 8.79206 & 14 & 66/1600 & 1369000/1369000 \\
Adult & SplitGroup & 8.77216 & 8 & 70/1600 & 1377200/1377200 \\
Bank & Merged-1k & 6.95398 & 0 & 0/160 & 61320/61320 \\
Bank & Merged-2k & 6.98036 & 0 & 2/320 & 117640/117640 \\
Bank & SplitGroup & 7.00838 & 0 & 0/320 & 119280/119280 \\
Credit & Merged-1k & 3.27047 & 0 & 38/800 & 185000/185000 \\
Credit & Merged-2k & 3.26066 & 1 & 121/1600 & 345000/345000 \\
Credit & SplitGroup & 3.29509 & 0 & 113/1600 & 353200/353200 \\
\bottomrule
\end{tabular}
\end{table}

\begin{figure}[!htbp]
\centering\includegraphics[width=\linewidth]{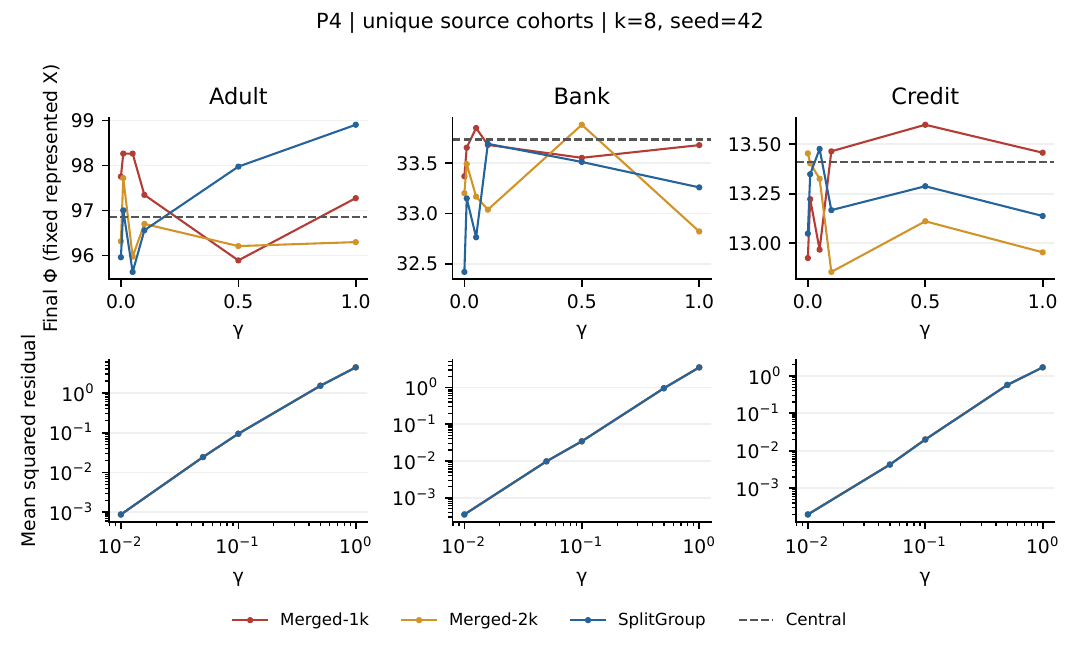}
\caption{P4: how does the anchor grid width $\gamma$ affect quality and anchor residuals? All 57 settings on the unique-source cohorts, seed 42. Top: final represented-data group-maximum cost. Bottom: unweighted mean squared residual over anchor slots, not weighted $Q$. Dashed lines are centralized quality references. No variability across training seeds is estimated.}
\label{fig:p4-current}
\end{figure}

Across the federated methods, final $\Phi$ spread does not exceed warm spread in 17/18 dataset/$\gamma$ groups; final $G$ spread grows in 1/18. These distinct summaries do not show universal erasure of initialization effects. Adult final federated $\Phi$ ranges from 95.6339 to 98.9039. Maximum duplicate-ID fraction is $7.625\%$, including baseline duplicates. Binary64 shape/dtype remain unchanged, so larger $\gamma$ alone does not reduce coordinate or packet bytes. Packets include metadata, not cryptographic or network transport. The artifact retains the separate serialization benchmark and full-run resource records. P1/P4 have no measured process RSS. P4 unit times include an unisolated output write; the later three-repeat serialization benchmark cannot recover that earlier duration and is not a rerun of clustering.

\FloatBarrier

\section{Additional Refinement-Budget and Whole-Client Studies}
\label{app:september26}
These bounded studies use the fixed FTF-1D-descriptor-v1 map, $k=10$, $b=12$, $L=6$, $\gamma=T_a=0$, and seeds 10000--10004. Preprocessing, persistent identities, row order and the inherited Adult/Credit provenance limitations remain fixed. Neither study replaces P2/P3 or changes the 810-comparison claim. Timing repetitions are not independent seeds; record and whole-client requests are not pooled. Both studies use local serial execution on an Apple M3 Max (48 GiB), macOS 27.0, Python 3.12.14 and NumPy 2.3.5, with five numerical-library thread limits set to one. An exclusive advisory lock coordinates the experiment tasks, but desktop background work remains uncontrolled; native pool sizes were not independently enumerated for the whole-client study.

\subsection{Record deletion across refinement budgets}
\label{app:round-budgets}
Adult/Bank/Credit retain their original partitions and independent record batches of 30/45/30 deletions. Each of 45 dataset/seed/budget settings has a separately constructed original-population checkpoint at $T=0,1,2$. Fresh, direct, basic and runner-up replay give 180 method keys and 135 distinct replay-versus-fresh correctness comparisons. All complete indexed trajectories and available exact integer $N,S,SS$ agree; all gate outputs also match the corresponding accepted $T=15$ prefixes. At $T=0$, only the initialized center tuple exists and refinement statistics are empty. Four separate-process timing repetitions per key give 720 observations, each matching its gate witness; cyclic order places every method in each position once. These are full measured calls at each budget, not prorated $T=15$ times or sequential deletion requests.

\begin{table}[htbp]
\centering
\caption{Record-deletion results, five seeds and four timing repetitions. Each ratio cell is algorithm / file-I/O-inclusive E2E fresh-to-replay total time. $T$ counts refinement rounds. Quality is identical across methods at a matched seed and budget.}
\label{tab:record-budgets}
\begin{tabular}{llrrrr}
\toprule
Dataset & $T$ & $\Phi$: mean $\pm$ SD & Direct & Basic & Runner-up \\
\midrule
Adult & 0 & 119.594 $\pm$ 3.048 & 4.119/3.304 & 4.101/3.284 & 4.110/3.296 \\
Adult & 1 & 97.941 $\pm$ 1.184 & 1.741/1.647 & 1.050/1.024 & 1.044/1.019 \\
Adult & 2 & 95.931 $\pm$ 1.391 & 1.421/1.374 & 1.033/1.014 & 1.031/1.012 \\
Bank & 0 & 47.959 $\pm$ 0.757 & 1.390/1.345 & 1.394/1.349 & 1.394/1.348 \\
Bank & 1 & 34.016 $\pm$ 1.106 & 1.228/1.203 & 0.889/0.881 & 0.890/0.882 \\
Bank & 2 & 33.010 $\pm$ 0.929 & 1.161/1.142 & 0.917/0.908 & 0.914/0.905 \\
Credit & 0 & 18.303 $\pm$ 0.876 & 3.696/3.200 & 3.709/3.214 & 3.687/3.187 \\
Credit & 1 & 13.461 $\pm$ 0.219 & 1.967/1.853 & 1.195/1.161 & 1.189/1.155 \\
Credit & 2 & 12.973 $\pm$ 0.146 & 1.592/1.527 & 1.144/1.115 & 1.137/1.108 \\
\bottomrule
\end{tabular}
\end{table}

\begin{figure}[htbp]
\centering
\includegraphics[width=\linewidth]{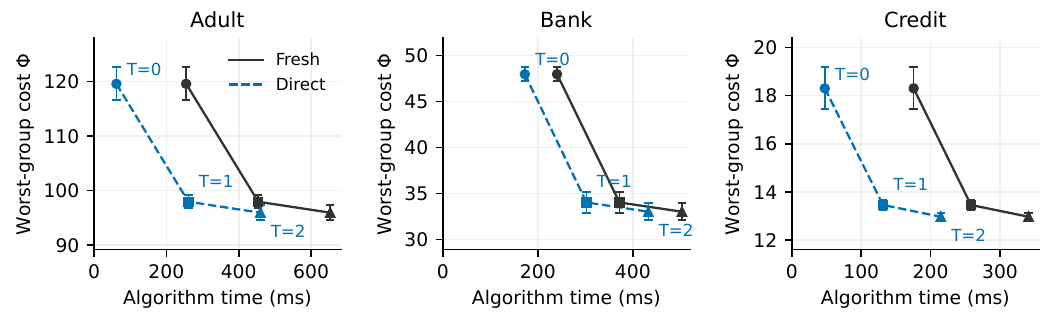}
\caption{Record deletion across $T=0,1,2$: five-seed means of four-repeat median algorithm times, with one seed SD of $\Phi$. Fresh and Direct have identical matched-budget quality. Connecting measured budgets does not establish a Pareto frontier or a cross-budget unlearning ratio.}
\label{fig:refinement-tradeoff}
\end{figure}

\paragraph{Corrected quality.}
\Cref{tab:record-budgets} and \cref{fig:refinement-tradeoff} use $\Phi=\max(\Phi_A,\Phi_B)$, not the auxiliary sum $G=\Phi_A+\Phi_B$. The initial downstream reporting interchanged these quantities; the current correction derives quality from saved exact fractions and verifies both identities on all 180 gate outputs. The correction leaves raw results and timings unchanged; no experiment was rerun. Across five seeds, the mean paired $T=0$-to-$T=2$ decreases in $\Phi$ are 19.8\% (Adult), 31.2\% (Bank), and 29.0\% (Credit). These decreases do not specify an acceptable application-level fairness loss. Error bars are descriptive seed SDs, not confidence intervals; the plotted time coordinate is the mean of seedwise four-repeat medians, whereas table ratios divide all observed total times.

\paragraph{Certificates and request boundary.}
At each positive budget, both variants reject every attempted certificate and fall back in round zero: 30 dataset/seed/budget cases per variant. Useful coverage and attempted acceptance are zero. All $T=0$ counters are zero, so these rates are undefined and timing differences do not establish a certificate advantage. For example, at $T=2$ each dataset attempts certificates only on its first round; failed assignments are repeated by the full rebuild. Basic and runner-up are slower than direct throughout, and slower than fresh on Bank (\cref{tab:record-budgets}).

Algorithm time includes the full public call, validation, encoding/materialization, allocations and instrumentation; fresh uses no replay-cache construction. The file-I/O-inclusive E2E clock runs from input-bundle opening through deletion materialization, replay checkpoint loading, pre-call garbage collection, the call, exact witness construction and equal-schema JSON output. Receipt metadata, quality scoring and verification are outside it. Process startup is separately reported, and no disk cache is flushed. Across 45 checkpoints, offline setup totals 17.08 seconds (16.26 seconds in cache-enabled calls); serialized checkpoint sizes are 12.51--58.01 MB and include original data. This boundary differs from the resident requests below and from the earlier P2/P3 E2E boundary, which does not identically enclose witness construction and intervening work. $T=0$ still communicates summaries/counts and performs server initialization. Network time is unmeasured; no deployed-speedup or communication guarantee follows.

Full-grid fresh/replay comparisons share production numerical kernels. The separately implemented reporting verifier checks exact bytes, integers, identities, counts and quality fractions; it is not an independent reimplementation of training arithmetic. The earlier independent small-case arithmetic gate is historical source-bound evidence, not a rerun. Corrected compact tables and presentation checks are included under \texttt{generated/september26}; the preserved run is \texttt{round-budget-20260926T064431Z}.

\subsection{Whole-client summary replay at initialization}
\label{app:client-summary}
Whole-client removal leaves every surviving local slice unchanged. Dropping departing summaries and retraining only the server follows the architecture of \citet[Section~5.3, Algorithm~4]{pan2023mufc}. The statement below specializes the initializer part of \cref{thm:exactness} to current demographic normalization and the declared indexed map; it is not a new general unlearning theorem. Pan's ordinary local sampler is not our locally demographic-balanced \Merged{} adaptation, and its reported $84\times$ is not directly numerically comparable to these different learners, baselines and timing boundaries.

\Needspace{12\baselineskip}
\begin{proposition}[Whole-client summary replay at initialization]
\label{prop:client-summary}
Fix a valid represented dataset $X$, persistent client and record ordering,
a supported configuration with $T=0$, and seed $s$. Let $D$ be a set of
whole clients such that both protected groups remain nonempty. Assume
local summaries depend only on their ordered represented slice, fixed
parameters, and a reproducible stream keyed by $(s,c,g)$. Remove the
summaries of clients in $D$, compute the retained group counts $n'_g$
exactly, rebuild the canonical $(c,g,j)$ anchor list with weights
$h_{c,g,j}/n'_g$, and rerun the declared server procedure from its original
seeded stream, including any fixed optional anchor-Lloyd iterations.
With the same representation, sampling, quantization, tie and completion
rules as training, this procedure returns exactly the indexed tuple
$\mathbf C^{(0)}$ of fresh training on $X\setminus\bigcup_{c\in D}X_c$.
\end{proposition}
\begin{proof}
Every surviving slice has the same ordered represented records, parameters
and keyed stream as before deletion. Hence its selected indices,
representatives, multiplicities and quantized anchors equal fresh
training's local summary. Deleted slices are omitted in both procedures.
Summing the surviving exact multiplicities gives the fresh group counts,
so the canonical server list and every rational weight agree. Starting
the same server stream therefore gives the same categorical decisions
and indexed centers, including deterministic saturation and zero-potential
completion. Any fixed deterministic guarded anchor-Lloyd iterations also
agree. Since $T=0$, this is the entire released trajectory and final model.
\end{proof}
For fixed or seed-independent $D$, pointwise equality in $s$ implies
equality in distribution under any seed law. This does not assert equality
of complete stored states, erasure of the original checkpoint, or a
fresh-independent guarantee for a request selected adaptively from that
same seed-dependent checkpoint. No runtime factor follows from the
proposition.

\paragraph{Compact-state assumptions.}
State is produced by the trusted canonical $T=0$ learner and must remain unmodified. It retains active persistent client IDs (including empty clients), per-client exact group counts, local summaries and fixed server settings. Surviving slice order, local seed keys and canonical $(c,g,j)$ anchor order are preserved; clients are never renumbered. Old weights and server centers cannot simply be kept. Requests reject unknown, duplicate, malformed or already-deleted IDs and any removal emptying a protected group. Under these assumptions, the returned active summaries/counts equal fresh compact state after each valid whole-client step by induction. This invariant does not make the original record-replay cache reusable or establish fresh-independent distributional equality for adaptively selected requests. The separate three-step timing campaign in \cref{app:client-sequences} tests this invariant with buffered state round trips.

The candidate retains $O(Ckd)$ summary coordinates and $O(C)$ count metadata, without raw arrays or point caches. Per request it filters metadata, updates exact counts, rebuilds compact anchors/weights and runs the existing server sampler. Integer/rational bit costs remain; the request is not constant time in clients or anchors. Nested dictionaries and shared summary objects are trusted, not authenticated immutable storage. Prior states, original checkpoints, raw inputs and witnesses can retain deleted information; model equality and smaller serialized state do not establish privacy or physical erasure.

\paragraph{Cases and exactness evidence.}
Before timing, the valid client closest to median size was selected with persistent ID as tie-breaker, choosing client 0 in each dataset. Adult deletes 300/30,162 records from 100 clients ($d=103$); Bank 2,259/45,211 from 20 ($d=42$); Credit 299/30,000 from 100 ($d=23$). Retained group counts are respectively (9,685,20,177), (17,098,25,854) and (5,332,24,369). Five seeds repeat this same selected request per dataset, giving 15 cases, not diverse client-deletion coverage. Each has three balanced-order observations per method: 135 observations and 90 repeated fresh-versus-replay model comparisons (30 distinct case/method pairs). All indexed $\mathbf C^{(0)}$ outputs, group counts, surviving summary fields and rational-weight anchor tables agree; the $T=0$ statistics list is empty. Full fresh caches are not claimed equal to compact state. This real-data check shares canonical arithmetic.

A separately authored AI correctness review checked the exact candidate against a scalar oracle on 196 valid subsets across 14 tiny configurations and 84 sequential steps, plus invalid requests, raw-access traps and an equal 510-byte model-output wrapper. The oracle independently implements geometry, rational arithmetic, ties, quantization and counts but shares NumPy PCG64/SeedSequence and binary64 runtime conversion. This is separate finite small-case evidence, not human validation or independent certification of real-data timing/quality aggregates. The reviewed candidate SHA-256 is \texttt{6d7b8d993ed9e9fc}\allowbreak\texttt{da896192ee6e6299}\allowbreak\texttt{a434ab7cdc25ab99}\allowbreak\texttt{07fc6b4bc1cf3dde}; it matches the experiment source and immutable v1 snapshot. The review found no blocking issue within its trusted-state, fixed-preprocessing, persistent-ID, valid whole-client, $T=0$ scope.

\begin{table}[htbp]
\centering
\caption{Whole-client case means: three timing repetitions per method, with a single preselected client per dataset. All speed ratios use fresh $T=0$; $\Phi_2$ is solely a quality reference.}
\label{tab:client-cases}
\begin{tabular}{lrrrrrrr}
\toprule
Dataset & Seed & Fresh ms & Direct ms & Fast ms & F/fast & $\Phi_0$ & $\Phi_2$ \\
\midrule
Adult & 10000 & 244.925 & 17.138 & 9.019 & 27.16 & 112.512 & 95.182 \\
Adult & 10001 & 243.347 & 17.308 & 9.070 & 26.83 & 117.132 & 95.807 \\
Adult & 10002 & 242.373 & 17.141 & 9.162 & 26.46 & 117.421 & 97.242 \\
Adult & 10003 & 250.312 & 17.690 & 9.403 & 26.62 & 119.269 & 98.305 \\
Adult & 10004 & 248.264 & 17.986 & 8.950 & 27.74 & 123.448 & 93.452 \\
Bank & 10000 & 221.181 & 6.316 & 1.973 & 112.12 & 44.867 & 32.578 \\
Bank & 10001 & 220.752 & 5.801 & 1.893 & 116.61 & 47.217 & 32.397 \\
Bank & 10002 & 219.056 & 6.295 & 1.898 & 115.39 & 44.861 & 32.770 \\
Bank & 10003 & 219.423 & 6.224 & 1.876 & 116.94 & 49.188 & 32.814 \\
Bank & 10004 & 220.003 & 6.288 & 1.884 & 116.77 & 43.358 & 32.613 \\
Credit & 10000 & 166.166 & 11.917 & 6.674 & 24.90 & 19.775 & 13.039 \\
Credit & 10001 & 165.317 & 11.772 & 6.599 & 25.05 & 20.554 & 13.106 \\
Credit & 10002 & 166.679 & 11.925 & 6.675 & 24.97 & 21.336 & 13.160 \\
Credit & 10003 & 166.718 & 12.082 & 6.602 & 25.25 & 20.177 & 12.883 \\
Credit & 10004 & 165.725 & 11.798 & 6.628 & 25.00 & 18.776 & 12.859 \\
\bottomrule
\end{tabular}
\end{table}

\paragraph{Resident-request timing and quality.}
All three methods have their required inputs/state resident. Fresh reconstructs local summaries with cache construction disabled; direct reuses them but retains raw-data bookkeeping/copies; the compact path needs no surviving raw records for the request. All rebuild and sample the compact server. Request time encloses preparation, the full public call, exact model-witness serialization and local file write, without \texttt{fsync}. It excludes cold loading, process startup, pre-request garbage collection, initial training/compaction, quality scoring, verification, network, physical purge and durable next-state persistence. The overall comparison harness retains raw data for fresh references and quality scoring, despite the compact request's narrower access needs.

Fresh/direct/fast totals over 45 observations per method are 9.480725/0.533040/0.264916 seconds for resident requests and 9.451897/0.506083/0.240401 seconds for algorithm calls. Thus pooled request ratios are 17.79 fresh/direct, 35.79 fresh/fast and 2.01 direct/fast; algorithm ratios are 18.68 and 39.32 fresh/direct and fresh/fast. Bank's 380 surviving anchor slots versus 1,980 for Adult/Credit help explain its smaller server cost. All 15 fresh/fast case-mean request ratios exceed ten (range 24.90--116.94), a property of this grid and implementation, not a universal speedup.

Matched retained-population mean $\Phi$ at $T=0,1,2$ is 117.957/97.816/95.998 (Adult), 45.898/33.544/32.634 (Bank), and 20.124/13.798/13.009 (Credit). The paired $T=0$ excess over $T=2$ ranges from 18.2--32.1\%, 32.9--49.9\%, and 46.0--62.1\%, respectively. No application-level quality threshold was established. The $T=1,2$ runs are quality references on these whole-client-deleted populations, never denominators for $T=0$ speedup and never substituted with record-deletion quality values.

\begin{table}[htbp]
\centering
\caption{Whole-client setup means across five seeds. Times are milliseconds; protocol-5 serialized sizes are measured in memory. Serialization is accounted separately, not as durable per-request state persistence.}
\label{tab:client-setup}
\begin{tabular}{lrrrrrrr}
\toprule
Dataset & Train & Compact & Load & Full ser. & Small ser. & Full MiB & Small MiB \\
\midrule
Adult & 251.805 & 1.178 & 44.902 & 5.627 & 1.202 & 52.481 & 3.196 \\
Bank & 231.292 & 0.173 & 30.118 & 2.455 & 0.184 & 29.734 & 0.267 \\
Credit & 167.754 & 0.669 & 17.164 & 2.887 & 0.869 & 11.929 & 0.754 \\
\bottomrule
\end{tabular}
\end{table}

\paragraph{Setup and storage accounting.}
\Cref{tab:client-setup} separates canonical original training, incremental summary compaction, input load/identity verification, and serialization. Compaction copies summaries and uses their existing exact counts without scanning raw records; this prototype still constructs the full canonical checkpoint first and does not eliminate that initial memory peak. Full checkpoint/compact serialized means are 55,030,623/3,351,441 bytes (Adult), 31,178,477/280,225 (Bank), and 12,508,311/791,090 (Credit). These are object sizes measured in memory, not observed disk or RAM savings. No new full checkpoints were persisted. Returned next-state dictionaries share unchanged summary arrays; prior state references are not purged.

Using observed case means, incremental compaction amortizes in one independent request against either fresh or direct; charging the full initial training plus compaction gives two against fresh. This arithmetic does not measure a durable sequential service. The full compact setup and quality tables accompany the source; the preserved run is \texttt{client-c0-20260926-v1}.

\section{Refinement, Client Coverage, Sequences, and Group Quality}
\label{app:consolidated-studies}
These studies extend the September 26 evidence without replacing the original P2/P3 panels. The binary studies retain the fixed preprocessing, original partitions, persistent identities, $k=10$, $b=12$, $L=6$, $\gamma=T_a=0$ and seeds 10000--10004. Adult/Credit's inherited raw-provenance limitations remain. The reporting inputs, protocols, source identities and manifests are included under \texttt{generated/consolidation20260926}. All results below were executed before manuscript integration; integration runs no training campaign. Five seeds, repeated timings, overlapping requests and correlated budgets or sequence steps are distinct units, never pooled as independent populations. Local serial measurements use the same Apple M3 Max/48 GiB, Python 3.12.14 and NumPy 2.3.5 environment; the shared advisory lock coordinates research workers but does not isolate background desktop load.

\subsection{Canonical whole-client replay with positive refinement}
\label{app:client-refinement}
Run \texttt{client-refinement-20260926-v1} extends exactly the 15 client-0 cases of \cref{app:client-summary} to $T=1,2$. Each of 30 settings builds its own original-population checkpoint at the matching budget. Fresh, Direct, Basic and Runner-up produce 120 distinct outputs and 90 distinct replay-versus-fresh comparisons; all indexed $C^{(0)},\ldots,C^{(T)}$, final centers and exact $N,S,SS$ match. Every new fresh witness also equals the previously completed same-population quality reference. Four method-order rotations give 480 timing observations; these are technical repetitions, not new scientific trials. The compact $T=0$ implementation is neither imported nor called at positive budgets.

\begin{table}[htbp]
\centering
\caption{Positive-budget whole-client replay. Ratios divide total resident-request times over five seeds and four repeats per dataset/budget; ``All'' divides pooled seconds. Each replay method wins all 15 seed/dataset cases per budget against fresh.}
\label{tab:client-positive}
\begin{tabular}{llrrr}
\toprule
Dataset & $T$ & Fresh/Direct & Fresh/Basic & Fresh/Runner-up \\
\midrule
Adult & 1 & 2.053 & 1.164 & 1.159 \\
Bank & 1 & 2.646 & 1.455 & 1.464 \\
Credit & 1 & 2.647 & 1.425 & 1.407 \\
All & 1 & 2.357 & 1.309 & 1.305 \\
Adult & 2 & 1.544 & 1.107 & 1.103 \\
Bank & 2 & 1.828 & 1.287 & 1.286 \\
Credit & 2 & 1.856 & 1.284 & 1.274 \\
All & 2 & 1.696 & 1.200 & 1.196 \\
\bottomrule
\end{tabular}
\end{table}

Inputs and required checkpoints are resident for every method. The algorithm timer includes the entire public call, validation, encoding/materialization, initialization and refinement. The request timer adds actual survivor/request preparation and identical full-witness projection, serialization and buffered local output write, without \texttt{fsync}. Fresh omits replay-cache construction. Initial data loading, original checkpoint training, setup/size measurement, pre-call garbage collection, reference checks and quality scoring are separate. This is not a cold-start, network, durable-state or physical-erasure measurement. Fresh/Direct pooled request totals are 20.550742/8.718981 seconds at $T=1$ and 28.607252/16.870120 seconds at $T=2$, each over 60 observations per method. Paired case ratios range 2.041--2.681 and 1.538--1.868, respectively. The earlier $T=0$ totals used three repeats and a separate run; no cross-budget ratio of these speed ratios is a latency measurement.

Direct retains Phase-I summaries but repeats all raw-survivor refinement. For Adult, mean fresh/Direct Phase-I times are 230.60/9.96 ms at $T=1$, whereas their refinement times are 195.00/195.77 ms. At $T=2$, the corresponding refinement times are 391.47/391.53 ms. This explains the falling relative saving on this grid without asserting a law for arbitrary $T$. Additive components sum every client and round once; certificate-scan, failed-assignment and fallback diagnostics overlap those components and are not added again. Residual time is not silently attributed to refinement.

Every Basic and Runner-up case rejects every attempted certificate in round zero, triggers strict-majority fallback and repeats a full survivor assignment round; the second round, when present, starts direct. Across the 15 cases per budget, each certificate method has $A=J=512{,}575$ and $P=S=0$. With $N=512{,}575$ at $T=1$ and $N=1{,}025{,}150$ at $T=2$, its assignment count $N-S+J$ is respectively $2N$ and $1.5N$. Thus beating fresh reflects initialization reuse, not successful certification. All 30 positive-budget settings are retained; no timing-selected exclusions or retries occurred. These tests share canonical numerical kernels; the earlier compact-only review is not a positive-budget harness endorsement.

\subsection{Limited client-identity coverage at zero refinement}
\label{app:client-diversity}
Six metadata-only roles were frozen before outcomes: minimum, nearest-median and maximum client size; minimum and maximum group-0 fraction; and smallest surviving smaller-group count. Ties use persistent IDs, invalid group-emptying requests are excluded, and repeated roles are deduplicated. On the original partitions they select only two clients per dataset: 0/99 for Adult and Credit, 0/19 for Bank. This includes previously tested client~0 and adds only one client identity per dataset. The resulting 30 dataset/client/seed cases include 15 cases identical to the original client-0 study and to sequence step~1 below. They are not 30 new independent populations, and the six overlapping roles are not six distinct clients.

\begin{table}[htbp]
\centering
\caption{Metadata-selected client coverage. Requests independently start from the original checkpoint. Both groups survive each removal; within each dataset both requests retain the same number of anchor slots.}
\label{tab:client-coverage}
\begin{tabular}{lrrrrr}
\toprule
Dataset & Client ID & Removed & Removed group 0/1 & Retained group 0/1 & Anchors \\
\midrule
Adult & 0 & 300 & 97/203 & 9,685/20,177 & 1,980 \\
Adult & 99 & 462 & 179/283 & 9,603/20,097 & 1,980 \\
Bank & 0 & 2,259 & 899/1,360 & 17,098/25,854 & 380 \\
Bank & 19 & 2,290 & 916/1,374 & 17,081/25,840 & 380 \\
Credit & 0 & 299 & 53/246 & 5,332/24,369 & 1,980 \\
Credit & 99 & 399 & 138/261 & 5,247/24,354 & 1,980 \\
\bottomrule
\end{tabular}
\end{table}

At $T=0$, all 60 distinct gate replay comparisons and 270 timed method witnesses match the fresh indexed model and declared surviving summary/count/anchor state. Three balanced-order repetitions per case give total fresh/compact request ratios of 26.65 (Adult), 115.54 (Bank) and 24.97 (Credit), or 35.58 pooled; paired case ratios range 24.56--118.93. The pooled Direct/compact ratio is 1.76. The equal-model-output resident boundary excludes setup, cold loading and durable next-state persistence. This limited endpoint grid does not vary compact anchor count within dataset or establish robustness to arbitrary heterogeneous partitions. No $T>0$ quality reference was measured for the additional clients. All 30 case-level quality values use $\Phi=\max_g\Phi_g$; no application-level adequacy conclusion follows.

\subsection{Three sequential compact-state withdrawals}
\label{app:client-sequences}
Run \texttt{sequences-20260926-v1} tests three $T=0$ whole-client withdrawals in original persistent-ID order, skipping only group-emptying requests. Every dataset selects IDs 0, 1, 2 without skips. The same sequence at five seeds gives 15 dataset/seed sequences and 45 correlated seed/step cases. Each of three full-sequence repeats restarts from a legitimately prepared original state. Within a sequence, each compact call consumes the validated reloaded state returned by the immediately preceding step, never an original checkpoint substituted as current state. Persistent identities and row order remain unchanged. Step~1 overlaps the original client-0 population and seed cases; later steps have different survivors.

All 135 timed compact outputs equal matched fresh $T=0$ in the indexed initialized centers, active identities, surviving summary fields, exact counts and canonical rational-weight anchor table. Returned and reloaded compact states match fresh compaction. The refinement-statistic list is empty at $T=0$, so this does not test nonempty refinement statistics. A new-process check also reloads all 15 final states and reproduces their saved step-3 models. These are finite checks of the conditional compact-state invariant, not equality of full raw caches or a fresh-independent law for model-adaptive requests.

\begin{table}[htbp]
\centering
\caption{Cumulative three-withdrawal timings at matched $T=0$, averaged over five seeds and three repeats. Operational time adds each method's necessary buffered state write/reload to its matched resident-output time; it is not durable transaction latency.}
\label{tab:client-sequence}
\begin{tabular}{lrrrrrr}
\toprule
 & \multicolumn{3}{c}{Matched resident output} & \multicolumn{3}{c}{Including buffered state round trip} \\
Dataset & Fresh ms & Compact ms & Ratio & Fresh ms & Compact ms & Ratio \\
\midrule
Adult & 713.326 & 26.608 & 26.81 & 714.219 & 44.826 & 15.93 \\
Bank & 612.280 & 5.763 & 106.25 & 613.038 & 8.346 & 73.45 \\
Credit & 483.370 & 19.331 & 25.01 & 484.338 & 28.164 & 17.20 \\
\bottomrule
\end{tabular}
\end{table}

\begin{figure}[htbp]
\centering
\includegraphics[width=\linewidth]{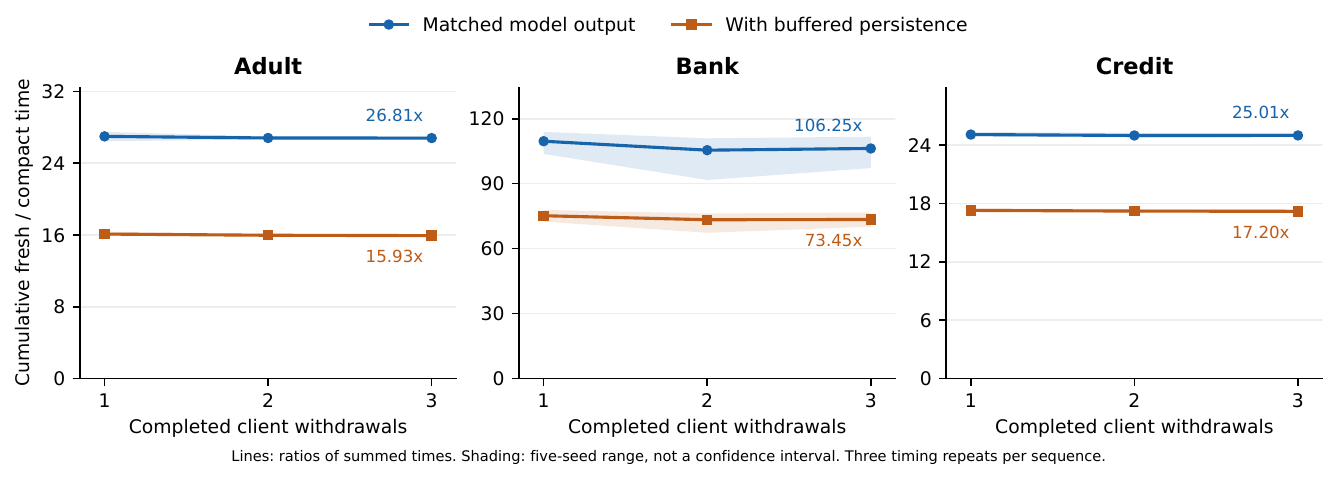}
\caption{Ratios of cumulative fresh/compact times through each withdrawal. Shading spans the five paired seed ratios, not a confidence interval. Blue uses matched resident model outputs; orange also charges method-specific buffered state persistence/reload. Steps and timing repetitions are correlated; dataset scales differ.}
\label{fig:client-sequence}
\end{figure}

The matched primary clock covers request preparation, public call, identical model-witness serialization and buffered local write. Compact next-state construction occurs within that call. The separate operational clock adds serialization, write/close, read/deserialization and validation of necessary state: compact's returned checkpoint versus fresh's active-client ledger (fixed arrays/configuration remain resident). It is the sum of charged components, not a continuous atomic transaction timer. Across all observations, fresh/compact ratios are 34.99 matched and 22.27 operational. No \texttt{fsync}, cache eviction, durable commit, network, physical purge or recovery latency is included. Initial training, incremental compaction and data loading are separate. Charging incremental compaction once to the three-request operational totals gives ratios 15.58/71.97/16.80; an explicit scenario charging common original training to both methods plus compaction to compact gives 3.31/3.56/3.35 (Adult/Bank/Credit). These setup views do not require a fresh-only service to build an original replay cache.

Initial/final compact serialized sizes are 3,351,441/3,250,905 bytes (Adult), 280,225/238,246 (Bank) and 791,090/767,372 (Credit). These are serialized state sizes, not measured disk savings or total RAM; preparation still constructs the full original checkpoint before compaction. Trusted local serialization/reload is not an adversarial-state validator, and overwritten files do not imply secure erasure. Group proportions and sizes remain similar along these early-ID sequences. Changes in $\Phi$ between steps concern different populations, so they are not fixed-population quality improvements. Longer, adaptive and positive-refinement sequences remain untested.

\subsection{Group-resolved quality and adverse transitions}
\label{app:group-quality}
For each fixed retained population and seed, report group costs $\Phi_g$, worst-group cost $\Phi=\max_g\Phi_g$, balanced sum $G=\sum_g\Phi_g$, pooled SSE $\sum_g n_g\Phi_g$, and population-average cost $\sum_g n_g\Phi_g/\sum_g n_g$. Dataset means average each case-level quantity across five seeds; in particular, the mean of maxima need not equal the maximum of group means. Sample SD uses divisor four and is descriptive. Group 0/1 means complement/male mask (Adult), not-married/married (Bank), and other stored education codes/\texttt{EDUCATION} in $\{1,2\}$ (Credit). Literal stored-mask wording preserves the Adult/Credit provenance limitation.

The analysis reuses 90 exact quality rows from the corrected record-deletion run and the original client-0 quality references, representing 30 seed/population cases at three budgets. It is saved-evidence analysis, not 90 new trials or independent raw-coordinate rescoring. The subsequent comparator in \cref{app:lloyd} supplies new scoring checks. Record requests remove 30/45/30 rows, whereas client requests remove 300/2,259/299; the two populations remain separate. All replay methods inherit matched fresh quality by exact model equality.

Mean paired $T=0$-to-$T=2$ decreases in $\Phi$ are 19.76/31.17/28.99\% for record deletion and 18.53/28.76/35.26\% for client deletion (Adult/Bank/Credit). Both groups and population-average cost decrease over that full interval in all 30 cases. Nevertheless, a fair objective decrease does not guarantee every group's or the population's improvement in every round. All five adverse fair group transitions appear in \cref{tab:adverse-quality}; two also raise population-average cost. The guard controls $\Phi$, not coordinatewise group cost or SSE.

\begin{table}[htbp]
\centering
\caption{Every adverse fair group-cost transition: Adult, $T=1$ to $T=2$. Worst-group cost decreases in every row. Population-average changes are signed absolute differences.}
\label{tab:adverse-quality}
\begin{tabular}{llrrrr}
\toprule
Population & Seed & Group 1 before & Group 1 after & Change (\%) & Population change \\
\midrule
Record & 10000 & 96.192670 & 97.326494 & +1.1787 & +0.039075 \\
Record & 10002 & 96.224892 & 96.477361 & +0.2624 & $-0.048872$ \\
Record & 10004 & 94.487752 & 96.450765 & +2.0775 & +0.593167 \\
Client & 10002 & 96.002141 & 96.957289 & +0.9949 & $-0.179302$ \\
Client & 10003 & 94.642231 & 95.233313 & +0.6245 & $-0.155996$ \\
\bottomrule
\end{tabular}
\end{table}

\begin{figure}[htbp]
\centering
\includegraphics[width=\linewidth]{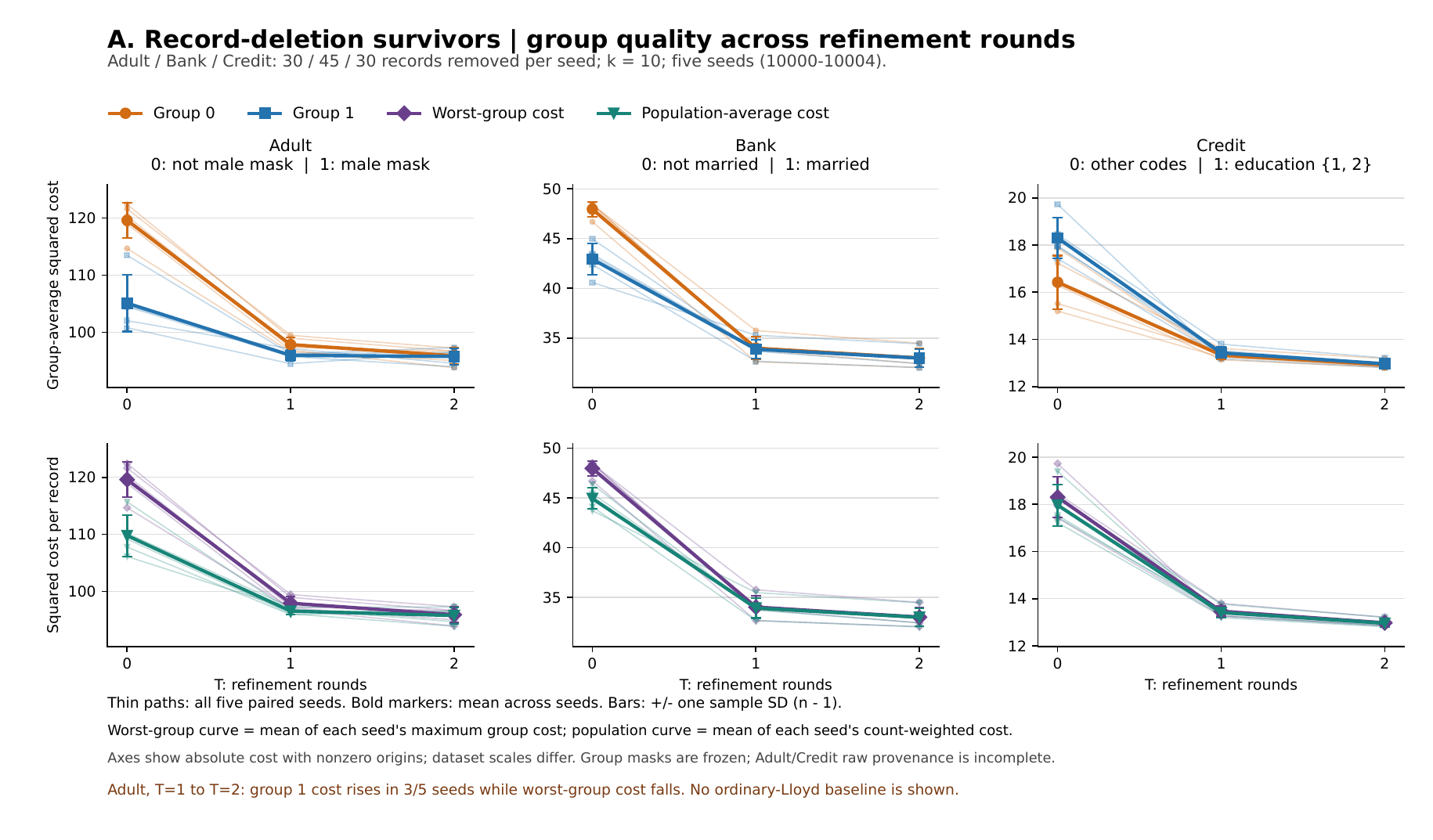}
\caption{Record-deletion quality on fixed represented survivors: separate group costs (top), worst-group and population-average costs (bottom). Thin lines retain five paired seeds; bold markers and bars show mean and one sample SD. Worst-group maxima and population count weighting are computed within seed. Adult group 1 rises in three second-round transitions. Dataset axes differ and have nonzero origins.}
\label{fig:record-groups}
\end{figure}
\begin{figure}[htbp]
\centering
\includegraphics[width=\linewidth]{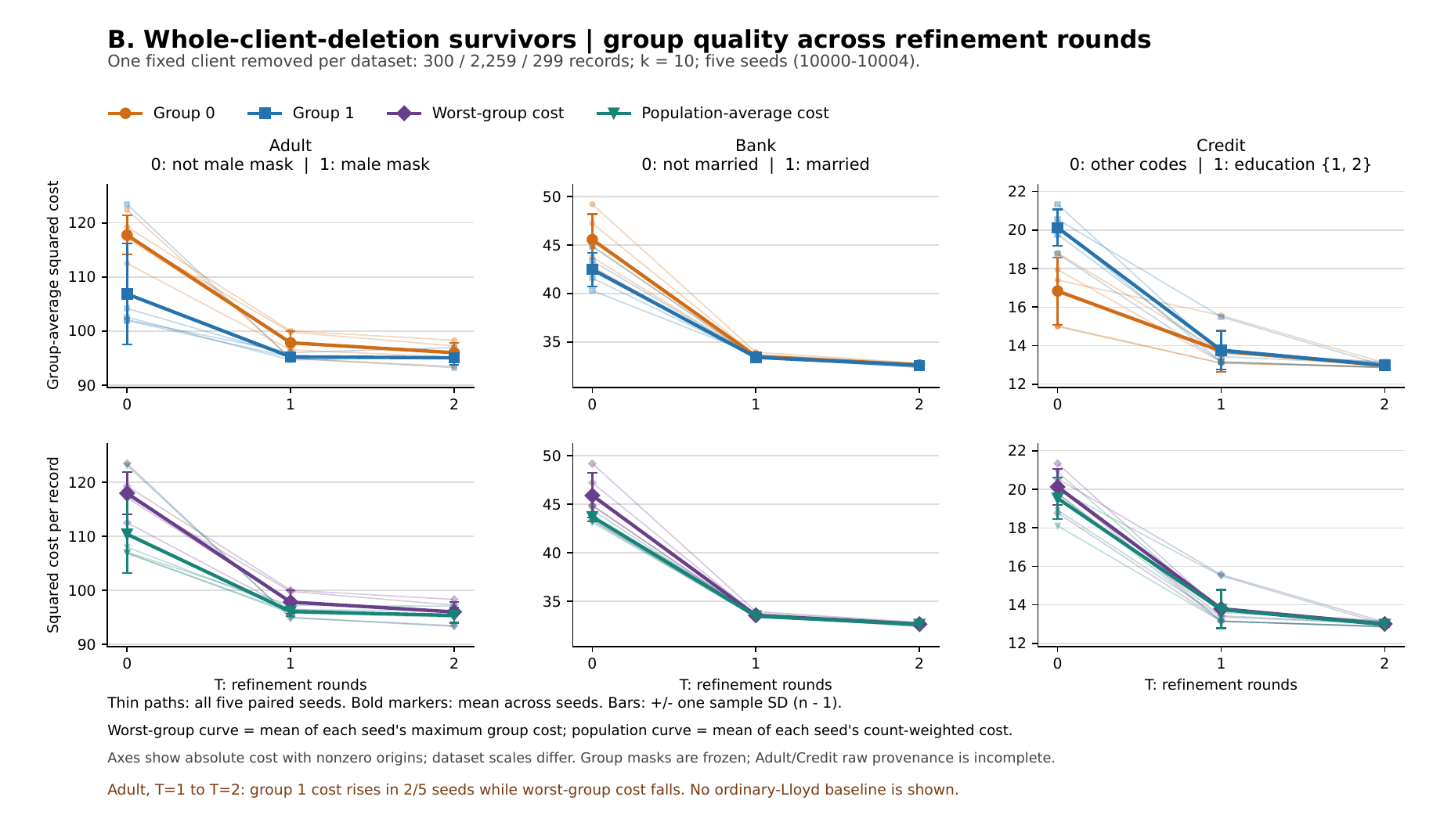}
\caption{Whole-client-deletion quality on the separate fixed client-0 survivor populations, with the same metric and seed conventions as \cref{fig:record-groups}. Adult group 1 rises in two second-round transitions. Positive-budget quality references are not denominators for zero-budget unlearning speedups.}
\label{fig:client-groups}
\end{figure}

\subsection{Matched-start ordinary Lloyd comparison}
\label{app:lloyd}
\begin{figure}[htbp]
\centering
\includegraphics[width=\linewidth]{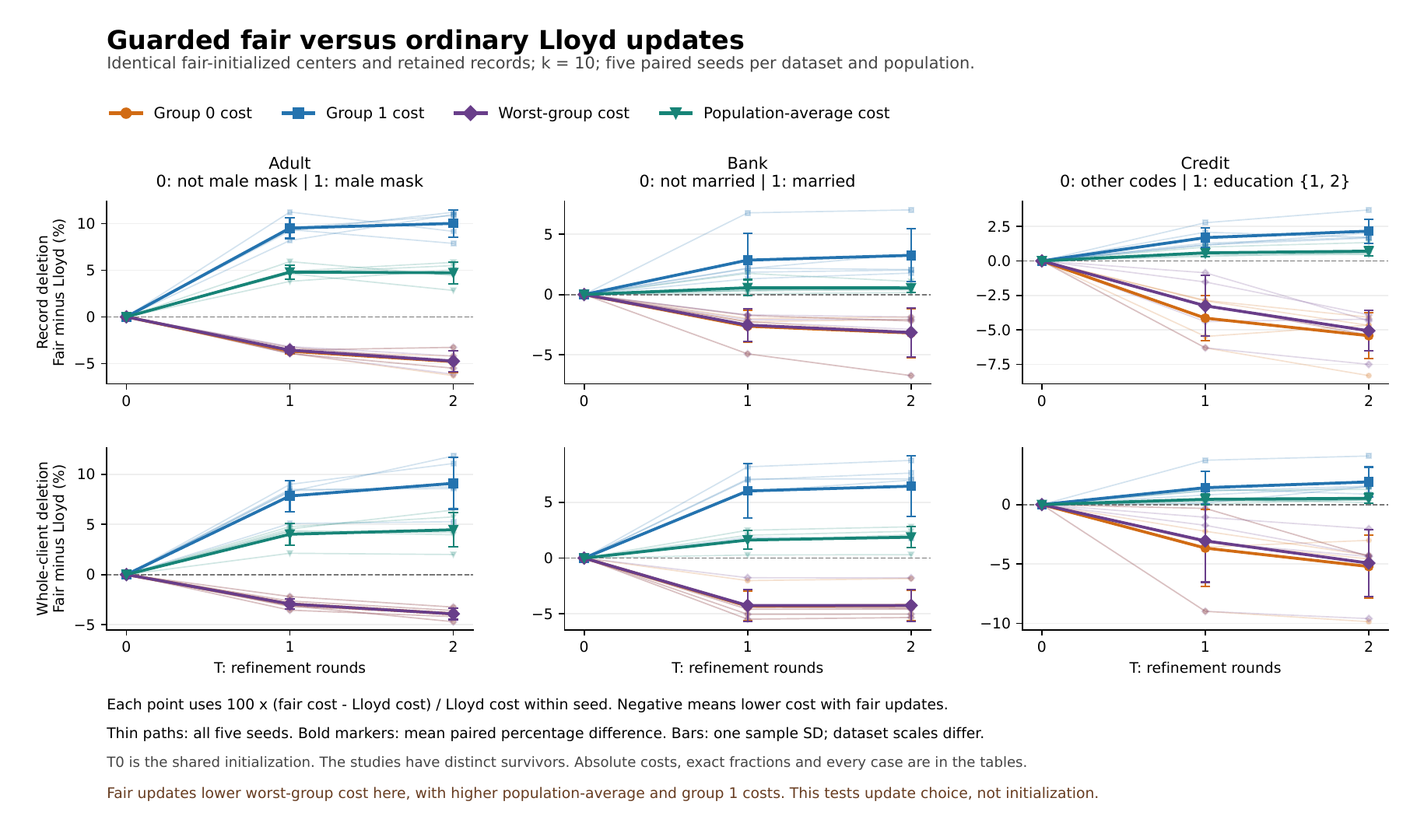}
\caption{Matched-start fair versus Lloyd refinement on separate record- and whole-client-deletion populations. Each seed's difference is $100(c_{\rm fair}-c_{\rm Lloyd})/c_{\rm Lloyd}$; negative favors fair for that cost. Thin lines show five seeds, bold markers their mean, bars one sample SD. Shared initialization gives zero differences at $T=0$; axes differ by dataset.}
\label{fig:matched-start}
\end{figure}
The comparator fixes survivors, represented coordinates, $k$, seeds and the entire indexed fair initialization. It then compares unchanged guarded fair refinement with ordinary unweighted Lloyd means for $T=1,2$. Lloyd uses exact integer sums, one binary64 rounding of each rational mean, indexed nearest-center ties and unchanged empty centers, with no acceptance guard and no resnapping to the data grid. It changes only the update rule: no independent ordinary initialization, optimization-to-convergence, runtime benchmark or new Lloyd-unlearning claim is supplied. P4's centralized comparator uses fair weighting and guarded refinement, so it cannot substitute for this test.

\begin{table}[htbp]
\centering
\caption{Fair-minus-Lloyd mean paired percentage differences: $100(c_{\rm fair}-c_{\rm Lloyd})/c_{\rm Lloyd}$ within seed, then averaged. Negative favors fair for the indicated cost. Each row has five seeds; $T=0$ is exactly tied.}
\label{tab:lloyd-quality}
\begin{tabular}{llrrrr}
\toprule
 & & \multicolumn{2}{c}{$T=1$} & \multicolumn{2}{c}{$T=2$} \\
Population & Dataset & Worst group & Population & Worst group & Population \\
\midrule
Record & Adult & $-3.57$ & +4.80 & $-4.75$ & +4.72 \\
Record & Bank & $-2.54$ & +0.57 & $-3.15$ & +0.55 \\
Record & Credit & $-3.26$ & +0.59 & $-5.07$ & +0.71 \\
Client & Adult & $-2.93$ & +4.02 & $-3.92$ & +4.47 \\
Client & Bank & $-4.25$ & +1.62 & $-4.24$ & +1.87 \\
Client & Credit & $-3.06$ & +0.45 & $-4.93$ & +0.53 \\
\bottomrule
\end{tabular}
\end{table}

At each positive budget, fair has lower $\Phi$ and group-0 cost in all 30 cases, but higher population-average and group-1 cost in all 30; group 1 is the larger group on these populations. This is a tradeoff between objectives, not simultaneous demographic improvement. Fair's five group-cost and two population-cost second-round increases remain as above. Lloyd additionally increases worst-group/group-0 cost by 0.039865 (0.2977\%) for Credit client-deletion seed 10004 from $T=1$ to $2$, even though its population SSE falls. Across all 60 adjacent steps, Lloyd's SSE falls and fair's $\Phi$ falls; these are observations on this grid. The sum $G$ favors fair in only 15/30 cases at $T=1$ and 14/30 at $T=2$, and must not be relabeled as $\Phi$.

Thirty completed case reruns reproduce all 90 saved fair center/statistic prefixes and objective sets exactly. A separately written scorer independently assigns represented raw rows using certified distance intervals with exact rational fallback, forms group moments and computes represented-center costs; it does not import production assignment/objective/accumulator routines. All 150 distinct center tuples agree with canonical labels and objectives; fixtures exercise ambiguous/tied distances and binary64 rounding. Fixtures and source/input/identity checks pass without exclusions or retries. Shared Python/NumPy and conversion remain dependencies; this is executed separately written checking, not independent human validation. The comparison does not establish optimality, cost equalization, downstream fairness or an application-level quality threshold.
\FloatBarrier

\paragraph{Separately authored experimental audit.}
A separately authored AI audit checked saved output/timing arithmetic, sequence chains, Lloyd mean-update coordinates and source/input/report bindings; it required the client-0 overlap disclosure in \cref{app:client-diversity}. A new arbitrary-precision integer distance scorer matched exact objectives and full labels for twelve $T=2$ tuples: both methods, seed 10000, on all six dataset/population combinations. This is a bounded independent rescore, not a rescore of all 150 tuples. Positive-budget and serialized/reloaded sparse-ID fixtures passed. A sub-microsecond cancellation in an auxiliary startup-inclusive record clock does not affect reported algorithm/request ratios. These are separately authored AI checks, not human validation; detailed scope and dispositions accompany the source.

\section{A Separate Categorical Multi-Group Prototype}
\label{app:multigroup-prototype}

The mathematical extensions in \cref{app:m-warm,app:multi-group} admit an isolated implementation for disjoint categorical groups. This is a separate numerical map, \texttt{FTF-MG-FW12-B6-v1}, and does not change the canonical binary implementation or its experimental panels. Multi-group socially fair clustering and its fixed-partition convex formulation are established prior work \citep[Section~3]{ghadiri2021socially}; whole-client summary removal with server recomputation is also an existing architecture \citep[Section~5.3, Algorithm~4]{pan2023mufc}.

\paragraph{Map and exactness.}
The prototype indexes local summaries, global counts and exact integer moments by $g\in\{0,\ldots,m-1\}$. Every declared global group must remain nonempty; absent local groups are omitted before any mean is formed. Whole-client deletion preserves surviving client and row identities, local stream keys and slice summaries. Updating $n_g$, rebuilding the sorted anchor weights $h/n_g$, and restarting the server stream therefore reproduces fresh initialization. Direct replay subsequently recomputes every retained assignment and exact $(N,S,SS)$ triple. Induction through the same deterministic guarded updater gives pointwise same-seed equality of indexed trajectories and moments. This is conditional on fixed preprocessing, numerical conventions and trusted state. It is not a claim of identical internal caches, physical erasure or a fresh-independent law under model-adaptive deletion.

\paragraph{Bounded optimization.}
For each fixed partition, the prototype searches the concave dual in \cref{app:multi-group}, not a presumed convex primal objective composed with simplex weights. Its frozen schedule evaluates uniform weights and simplex vertices, then at most twelve deterministic Frank--Wolfe directions with six dyadic derivative-bisection evaluations per segment. Exact arithmetic determines comparisons and ties. Weighted means are rounded once to binary64; the exact objective guard evaluates those represented returned centers against the incumbent. Zero weighted denominators use the sole occupied group mean, or the incumbent when multiple or no groups occupy the cluster. The maximum exact dual value over evaluated weights gives a valid lower bound, so the returned fixed-partition primal value minus that bound is an absolute optimality certificate for that partition. The finite solver need not attain the optimum for replay to be exact. Neither this gap nor exact replay certifies a global clustering optimum.

\paragraph{Bounded validation.}
Separately implemented scalar rational fixtures at $m=3,5$ pass 120 initialization comparisons, 48 positive-round replay comparisons and 54 valid compact-state sequence steps, covering local absences, ties, saturation, zero mass and invalid deletions. Six small primal-reference cases reach an exact optimization bracket of width at most $10^{-5}$. These checks are separately implemented self-validation for the prototype, not human validation. A subsequent separately authored AI audit found no blocking correctness defect within the frozen map and bounded pilot; its independent checks are distinguished below.

One real-data pilot uses the existing ACS 2018 population with 100,000 rows, ten state clients and the nine documented \texttt{RAC1P} categories. Exact recovery from the frozen feature schema and transform agrees with every stored row and native-identity mapping. No guessed Adult dummy-column interpretation or new population is used. The original sixteen-feature processing is preserved, including \texttt{RAC1P} as a feature, with no survey weighting. One prespecified whole-client request removes Missouri (persistent client~1, 6,799 rows), leaving all nine groups nonempty; the smallest retained group contains seven rows. The pilot uses $k=10$, seeds 10000--10004, and matched $T\in\{0,1,2\}$.

There are 20 configured replay-versus-fresh gate comparisons across 35 gate files, all equal at matched budgets including nonempty refinement statistics. The 105 timed outputs (45 fresh and 60 replay) repeat these models. This is one population/request, five seeds and 15 seed/budget cases, not 105 independent deletion cases. Mean nearest-center worst-group cost $\Phi=\max_g\Phi_g$ is 16.374, 9.994 and 9.445 at $T=0,1,2$. The separate sum $G=\sum_g\Phi_g$ is not used as the fairness objective. Across the ten distinct positive updates, the exact fixed-partition primal--dual gaps range from 0.041149 to 0.153621; all updates reach the fixed iteration bound. No utility threshold or global-optimum claim is made.

With resident data/state, three technical repetitions and equal model-output serialization, ratios of total fresh to direct request time are 138.70, 1.60 and 1.30 at matched $T=0,1,2$. The raw-free compact $T=0$ path has a ratio of 139.14 and essentially the same latency as the optimized whole-client direct path. Its serialized state is 205,773 bytes, compared with 26,608,600 bytes for the coordinate-retaining direct state, measured in memory. Initial setup, loading, state persistence, network and erasure costs are outside these request ratios; the report supplies measured setup costs and boundary details. The shared process lock coordinates experimental workers but not unrelated desktop load. These figures describe one fixed population/request and five seeds, not a general service speed factor or a comparison with the binary panels. Full exact witnesses, source/protocol identities, technical-repeat variability and attribution are in the separate prototype report.

\paragraph{Separately authored prototype audit.}
The frozen manifest and all 335 listed files, ten data dependencies and seventeen baseline modules passed an independent identity recheck. Separately written tiny checks covered 192 candidate evaluations, 600 exact supergradient inequalities, eighteen fresh/Direct cases and eighteen compact sequence steps. Every processed row's recovered membership, native identity alignment and all 1,600,000 encoded coordinates were rechecked. Saved moment-based certificates were independently rederived for 778 candidate evaluations across ten distinct seed/round partitions; 1,164 stored evaluation records include repeated first-round partitions at $T=1,2$. This validates the scoped weak-duality/guard/replay argument and saved reporting, not independent reassignment of every real row for every model. Retained group counts in RAC1P code order are 66,577; 7,489; 1,059; 7; 162; 8,240; 168; 5,969; and 3,530. The audit does not establish original raw Census provenance, hostile-state protection, an exact optimizer or a positive-round sequential checkpoint service. No human validation or publication approval is implied.

\section{Implementation and Numerical Conventions}
\label{app:pseudocode}
\label{app:implementation}
The main text describes the procedures conceptually. Exact replay also depends on details that ordinary algorithm descriptions leave open, and this appendix collects the conventions needed to implement, reproduce or audit the computation: the training pseudocode, the numerical boundary and encoding, keyed randomness and the exact rational categorical sampler, saturation, completion and tie rules, quantization, the represented guarded update, certificate arithmetic, and deletion validation. Together these define the \emph{declared training map}. Algorithm~\ref{alg:training} summarizes the training path; \cref{alg:unlearn} gives deletion. Full records and per-round point caches stay at clients. Transmitted anchors can coincide with individual records, especially at $\gamma=0$; the server processes weighted anchors directly, without dense synthetic expansion. Slices are processed in canonical $(c,g)$ order, and the server receives anchors as a canonical list ordered by $(c,g,j)$. We write $(s,\texttt{local},c,g)$ for the separate random stream (tape) of slice $(c,g)$ described in the main text, keyed by the seed and the slice, and $(s,\texttt{server})$ for the server's; their executable form is specified below.

\begin{algorithm}[ht]
\caption{Declared training map $\Train(X;s)$}
\label{alg:training}
\begin{algorithmic}[1]
\ForAll{$(c,g)$ with $X_{c,g}\ne\varnothing$}
  \State Select $k_e$ representatives by $D^2$-sampling with tape $(s,\texttt{local},c,g)$
  \State Assign all slice records to nearest selected representative; build $h_{e,j}$
  \State Quantize deterministically and emit $(q_{e,j},h_{e,j})$
\EndFor
\State Compute exact $n_A,n_B$ and canonical weights $h/n_g$
\State Run weighted $D^2$-sampling on the anchors with tape $(s,\texttt{server})$
\State Apply optional fixed $T_a$ deterministic weighted anchor-Lloyd rounds
\For{$t=0,\ldots,T-1$}
  \State Clients assign the fixed represented $X$ using exact squared-distance comparisons and indexed ties
  \State Aggregate exact $(N,S,SS)$ (in-process in the simulation)
  \State Cache labels, nearest/second distances, and runner-up identities locally
  \State Server applies $\GFU_L$
\EndFor
\State \Return indexed trajectory, final centers, and distributed checkpoint state
\end{algorithmic}
\end{algorithm}

\paragraph{Versioned numerical boundary.}
The canonical specification is \texttt{TRAINING\_MAP.md}, version FTF-1. Input is finite binary64 of magnitude at most $2^{450}$, with common feature dimension $1\le d\le4096$, stable nonnegative integer client IDs, group IDs $0,1$, and both global groups nonempty. Preprocessing is fixed outside deletion. With integer $0\le b\le30$, positive finite clip $a$, and $a2^b\le2^{52}$, encode once by nearest-even rounding of $2^b\operatorname{clip}(x,-a,a)$. Require $dM^2\le2^{59}$ for the maximum encoded coordinate bound $M$ before any narrow-integer square or reduction. The same condition is checked on direct accumulator input. Chunk totals are bounded by $2^{59}$; completed $N,S,SS$ and signed corrections use Python integers. Any exact additive representation would serve; what matters is that subtracting a record's contribution exactly inverts adding it. Invalid numeric input is rejected, never silently wrapped. Empty clients retain dimension and identity. The main experiments use $b=12,L=6$; regression validation also covers $b=16,L=20$.

\paragraph{Canonical rational categorical sampling.}
Represent each positive mass as a rational, clear denominators by their least common multiple, and divide the integer vector by its greatest common divisor. If the resulting total is $W>1$, draw $B=\lceil\log_2W\rceil$ bits to obtain $U\in\{0,\ldots,2^B-1\}$. Reject and redraw all bits when $U\ge W$; otherwise choose the unique category whose half-open cumulative integer interval contains $U$. $W=1$ consumes no bits. $W=0$ invokes deterministic completion.

Under ideal independent fair bits, any $u<W$ is returned with probability
\begin{equation}
\sum_{q=0}^{\infty}\left(1-\frac{W}{2^B}\right)^q2^{-B}=\frac1W.
\end{equation}
The method terminates almost surely, each category receives its exact rational mass fraction, and zero-mass categories cannot be sampled. This proof establishes unbiasedness; a frequency test alone does not. Local squared-distance masses and server $hD^2/n_g$ masses use the exact rational values of represented coordinates.

Executable streams use explicit NumPy PCG64 with \texttt{SeedSequence}, using entropy $[s,\mathrm{role},\mathrm{entity\ IDs}]$. Role IDs are local 1, server 2, merged-local 3, merged-server 4, and centralized 5. Each trial consumes $\lceil B/64\rceil$ raw 64-bit words, first word least significant, masks bits above $B$, and discards unused bits. Rejection consumes a new complete word block. The local tape $(s,\texttt{local},c,g)$ and server tape $(s,\texttt{server})$ of the main text are these streams with roles 1 and 2. Keys are namespace separation; finite deterministic PRNG execution is not a proof of independent ideal tapes. The NumPy version, environment, seed, role/entity keys, and consumed words/rejections must be frozen or logged for reproducibility.

\paragraph{Saturation, order, and ties.}
For $k\ge n>0$, return ascending input row indices, then cyclic row indices if more slots are requested, without draws. Local budget is $\min(k,n)$, so a saturated slice has zero representation cost. Otherwise use the weighted first draw and ordinary $D^2$ draws until zero potential, then append unselected row IDs in ascending order and cycle if necessary. The server applies the same rules to its canonical anchor list, with weights $\omega=h/n_g$ and masses $\omega D^2$ held as exact rationals. This can include zero-weight rows and repeated locations. An empty sampler input permits only $k=0$. Distinct source records and coincident center slots remain distinct indexed objects. Assignments compare exact squared distances, break ties by smallest center index, and define the runner-up as the second entry in distance-then-index order, including coincident slots.

\paragraph{Quantization representation.}
$\gamma=0$ preserves representatives. For $\gamma>0$, interpret its binary64 value as an exact rational, compute the exact nearest-even grid index of $z/\gamma$, multiply by rational $\gamma$, and round once to binary64. Write the exact grid point as $q^\star$ and represented output as $q$. If $\|q-q^\star\|\le\rho$ over the valid anchor range, then
\begin{equation}
Q=\sum_x\frac{\|z(x)-q(x)\|^2}{n_{g(x)}}
\le\left(\sqrt{d\gamma^2/2}+\sqrt2\rho\right)^2.
\label{eq:represented-quantization}
\end{equation}
Thus the ideal-grid $d\gamma^2/2$ term requires exactly represented grid outputs or an explicit rounding qualification. The proof's root-cost argument also works directly with $\sqrt{\E Q}$ for the actual represented anchors. A measured $Q$ is a realization-level diagnostic, not an expectation bound by itself. Coordinates remain float64, so changing $\gamma$ alone does not compress their array payload.

\paragraph{Represented guarded update.}
Centers are finite indexed binary64 values interpreted as exact dyadic rationals. At each rational dyadic endpoint/midpoint $\lambda$, form \cref{eq:lambda-center} from exact rational means and group-cell weights and round each coordinate once to binary64. A one-group cell uses the occupied mean for every $\lambda$, including a zero-weight endpoint; a fully empty cell retains its incumbent. Evaluate both actual represented candidate costs by \cref{eq:stats-identity}. Set the lower bisection bound to $\lambda$ iff $f_A>f_B$; equality sets the upper bound. Select the lowest exact maximum group cost, with smallest rational $\lambda$ on ties, and accept iff its cost is no greater than the incumbent. Counts $T,L,T_a$ are nonnegative integers, with $L\le256$. Optional anchor-Lloyd uses a deterministic empty-cluster rule, exact weighted means, one final coordinate rounding, and an exact represented-objective guard, so it never increases the weighted anchor objective.

\paragraph{Certificates and fallback.}
Store outward lower/upper bounds for each cached nearest and second-nearest Euclidean distance and indexed labels/runner-up IDs. Shifts have outward bounds too. Basic certification compares the outward winner upper sum against $\max(0,\underline d_2-\max_{m\ne j}\overline\delta_m)$. Runner-up certification takes the distance between the intervals for $d_2$ and $\delta_r$, intersected with the non-runner competitor lower bound. Only a strict separated comparison accepts. Underflow is enclosed, never interpreted as a proven zero shift. Inconclusive bounds use exact/certified assignment with exact tie fallback; $k=1$ has label zero. Full sorting and exact fallback cost are charged. In memory, the two bound objects also retain one float64 display array each: six float64 arrays and two int64 IDs use $64nT$ array bytes before object overhead and other state. Serialization stores only four bound arrays and two IDs, reconstructing the display arrays on load: $48nT$ raw-array bytes, plus pickle framing and other checkpoint state. Recorded checkpoint bytes and process RSS supersede these array-only counts.

\paragraph{Deletion and output.}
Deletion identity is the checkpoint's (client ID, original local row index), resolved to persistent source identities by its manifest. Reject unknown clients, nonintegral, negative or out-of-range row indices, repeated requests, invalid modes/thresholds, and deletions emptying either global group. Canonicalize valid indices, preserve client IDs and survivor order, and keep each request relative to its original checkpoint. Reuse untouched summaries, rerun touched ones with their original keys, and rebuild the entire server initializer with retained group denominators. Finite nonnegative thresholds are valid; thresholds at least one disable abandonment. The failed fraction is compared to the actual represented threshold by integer cross multiplication. A fraction strictly greater than the threshold (default $1/2$) triggers current-round full reconstruction before the update, with all later rounds direct.

Replay returns indexed trajectories, exact statistics, and final centers with \texttt{cache\_valid=false}; it cannot be reused as a fresh checkpoint. Fresh training can omit cache construction for a model-output-matched timing comparison. No \texttt{rebuild\_cache=true} fast service or free sequential deletion is claimed. Cache construction, loading, purging, encoding, input materialization, output serialization, and fallback sunk costs are distinct measured tasks.

\section{Extended Related Work}
\label{app:positioning}
\label{app:related-extended}
We expand the comparisons in \cref{sec:related} by objective, retained information, and unlearning target.

\paragraph{Fair objectives and summaries.}
Representation-constrained clustering and coresets \citep{chierichetti2017fairlets,bera2019fair,schmidt2020faircoresets} differ from group-average cost fairness \citep{abbasi2021equitable,ghadiri2021socially}. Unlike our initialization anchors, the union of per-group strong coresets preserves the socially fair objective for every center query \citep{gorantla2023coresets}. Centralized guarantees address group-count dependence \citep{makarychev2021socially}, fixed-parameter factors \citep{goyal2023tight}, and constant-factor, parameterized or local-search algorithms \citep{ghadiri2022constant,anand2026socially}; we claim no best approximation ratio. Cascaded norms generalize point/group aggregation, with square-root and squared-cost guarantees requiring translation \citep{chlamtac2022cascaded}. F3KM uses a balance-oriented federated objective \citep{zhu2023f3km}.

\paragraph{Federated clustering and deletion.}
Beyond the two-level architectures discussed in \cref{sec:related}, ARCANE-based $k$-means++ unlearning also precedes our work \citep{gupta2025arcane}. OmniFC reconstructs a global distance matrix through coded computation \citep{yan2025omnifc}; its information/privacy protocol differs from our anchors and additive packets, and we do not benchmark distributed superiority. Q-$k$-means quantizes iterative centroids, saves optimization metadata, tests deletion-induced changes and stops when loss fails to improve; DC-$k$-means recomputes affected tree paths \citep{ginart2019forget}. These precede our stability checks and reuse. Sharded retraining \citep{bourtoule2021sisa}, FedCIO \citep{qiu2023fedcio}, and sequential-client SIFU \citep{fraboni2024sifu} address other deletion procedures.

\paragraph{Exactness targets.}
Stability and coupling already support exact unlearning \citep{ullah2021stability}. \citet{ullah2023adaptive} distinguish model-and-state from model-output distributional targets and cover adaptive linear queries, including Lloyd iterations and federated averaging. FATS uses total-variation stability for sample/client deletion \citep{tao2024fats}; Exact-Fun specifies a quantized federated map \citep{xiong2023exactfun}. Their guarantees are relative to their own training procedures. Our indexed center-trajectory equality does not establish Ullah--Arora/FATS full-state guarantees or a fresh-independent adaptive-request law. Smooth-risk approximate removal certificates have another target \citep{guo2020certified}. Quantized clustering for in-context language-model adaptation concerns a different pipeline \citep{muresanu2025context}. These distinctions preclude broad first-exact-federated or first-iterative claims.

\paragraph{Assignment information and arithmetic.}
Elkan, Hamerly and Yinyang use per-center, compact or grouped distance bounds \citep{elkan2003triangle,hamerly2010faster,ding2015yinyang}; adaptive and geometric alternatives trade different information and costs \citep{drake2012adaptive,newling2016bounds,rysavy2016geometric}. Our checkpoint-to-replay movement test specializes this tradition: the contributions are tightness for the stipulated cache, an exact-displacement runner-up refinement, and sound arithmetic with aggregate replay. Richer caches need not be dominated. Yinyang already notes sensitivity to one moving center \citep[Section 2.1]{ding2015yinyang}, which does not identify a P3 trial's cause from aggregate counters. Exact $D^2$-sampling accelerations \citep{raff2021exact,rodriguez2026fastkpp} likewise preclude universalizing our straightforward distance-work bound; they are not benchmarked here. Counts, sums and squared-norm sums are classical \citep{zhang1996birch}, as is exact/reproducible summation \citep{neal2015exact,ahrens2020summation}. Frozen-feature federated ridge-head deletion has exact additive statistics and a closed-form solution \citep{quan2026ridge}; our changing partitions must first be reconstructed by sound assignments. Repeatable rounded sums alone are not algebraically removable.

\paragraph{Other fairness targets.}
FedShard's hierarchical sharding addresses client efficiency/performance fairness \citep{wen2026fedshard}; supervised fair unlearning \citep{oesterling2024fair}, FROG's graph fairness \citep{chen2025frog}, and federated fairness/heterogeneity evaluation \citep{huang2025wild} also combine deletion and fairness. Our demographic maximum average squared-distance objective establishes neither downstream fairness nor fair client computation burdens. FeDa4Fair evaluates broader client fairness \citep{heilmann2026feda}; our independently specified ACS cohort does not reproduce that benchmark.

\end{document}